%% file: arxiv_cycle.tex
\documentclass[10pt,twocolumn]{article}

\usepackage{amsmath,amssymb}
\usepackage{algorithmic}
\usepackage{algorithm}
\usepackage{amsmath}
\usepackage{amssymb}
\usepackage{amsthm}

\newcommand{\nbr}[1]{\left\|#1\right\|}
\DeclareMathOperator*{\argmax}{\mathrm{argmax}}
\DeclareMathOperator*{\argmin}{\mathrm{argmin}}
\newcommand{\eq}[1]{(\ref{#1})}
\newtheorem{theorem}{Theorem}
\newtheorem{lemma}[theorem]{Lemma}
\newtheorem{proposition}[theorem]{Proposition}
\newtheorem{definition}[theorem]{Definition}

\usepackage{cite}
\usepackage{graphicx}
\usepackage{url}

\begin{document}
\title{Graph rigidity, Cyclic Belief Propagation and Point Pattern Matching}

\author{Julian~J.~McAuley\thanks{The authors are with the Statistical Machine Learning Program, NICTA, and the Research School of Information Sciences and Engineering, Australian National University}, Tib\'erio~S.~Caetano and Marconi~S.~Barbosa}

\maketitle

\begin{abstract}

A recent paper \cite{CaeCaeSchBar06} proposed a provably optimal, polynomial time method for performing near-isometric point pattern matching by means of exact probabilistic inference in a chordal graphical model. Their fundamental result is that the chordal graph in question is shown to be \emph{globally rigid}, implying that exact inference provides the \emph{same} matching solution as exact inference in a complete graphical model. This implies that the algorithm is optimal when there is no noise in the point patterns. In this paper, we present a new graph which is also globally rigid but has an advantage over the graph proposed in \cite{CaeCaeSchBar06}: its maximal clique size is smaller, rendering inference significantly more efficient. However, our graph is not chordal and thus standard Junction Tree algorithms cannot be directly applied. Nevertheless, we show that loopy belief propagation in such a graph converges to the optimal solution. This allows us to retain the optimality guarantee in the noiseless case, while substantially reducing both memory requirements and processing time. Our experimental results show that the accuracy of the proposed solution is indistinguishable from that of \cite{CaeCaeSchBar06} when there is noise in the point patterns.



\end{abstract}

\section{Introduction}

Point pattern matching is a fundamental problem in pattern recognition, and has been modeled in several different forms, depending on the demands of the application domain in which it is required \cite{CarHan00, Kel06}. A classic formulation which is realistic in many practical scenarios is that of near-isometric point pattern matching, in which we are given both a ``template'' ($\mathcal T$) and a ``scene'' ($\mathcal S$) point patterns, and it is assumed that $\mathcal S$ contains an instance of $\mathcal T$ (say $\mathcal T'$), apart from an isometric transformation and possibly some small jitter in the point coordinates. The goal is to identify $\mathcal T'$ in $\mathcal S$ and find which points in $\mathcal T$ correspond to which points in $\mathcal T'$.


Recently, a method was introduced which solves this problem efficiently by means of exact belief propagation in a certain graphical model \cite{CaeCaeSchBar06}. The approach is appealing because it is optimal not only in that it consists of exact inference in a graph with small maximal clique size ($=4$ for matching in $\mathbb R^2$), but that the \emph{graph itself} is optimal. There it is shown that the maximum a posteriori (MAP) solution in the sparse and tractable graphical model where inference is performed is actually the \emph{same} MAP solution that would be obtained if a fully connected model (which is intractable) could be used. This is due to the so-called \emph{global rigidity} of the chordal graph in question: when the graph is embedded in the plane, the lengths of its edges uniquely determine the lengths of the absent edges (i.e.~the edges of the graph complement) \cite{Connelly05}. The computational complexity of the optimal point pattern matching algorithm is then shown to be $O(nm^4)$ (both in terms of processing time and memory requirements), where $n$ is the number of points in the template point pattern and $m$ is the number of points in the scene point pattern (usually with $m>n$ in applications). This reflects precisely the computational complexity of the Junction Tree algorithm in a chordal graph with $O(n)$ nodes, $O(m)$ states per node and maximal cliques of size 4. The authors present experiments which give evidence that the method substantially improves on well-known matching techniques, including Graduated Assignment \cite{Gold96}.

In this paper, we show how the same optimality proof can be obtained with an algorithm that runs in $O(nm^3)$ time per iteration. In addition, memory requirements are precisely decreased by a factor of $m$. We are able to achieve this by identifying a new graph which is globally rigid but has a \emph{smaller} maximal clique size: 3. The main problem we face is that our graph is \emph{not} chordal, so in order to enforce the running intersection property for applying the Junction Tree algorithm the graph should first be triangulated; this would not be interesting in our case, since the resulting triangulated graph would have larger maximal clique size. Instead, we show that belief propagation in this graph \emph{converges to the optimal solution}, although not necessarily in a single iteration. In practice, we find that convergence occurs after a small number of iterations, thus improving the running-time by an order of magnitude. We compare the performance of our model to that of \cite{CaeCaeSchBar06} with synthetic and real point sets derived from images, and show that in fact comparable accuracy is obtained while substantial speed-ups are observed.

\section{Background}
\label{sec:background}


We consider point matching problems in $\mathbb R^2$. The problem we study is that of near-isometric point pattern matching (as defined above), i.e.~one assumes that a near-isometric instance ($\mathcal T'$) of the template ($\mathcal T$) is somewhen ``hidden'' in the scene ($\mathcal S$).  By ``near-isometric'' it is meant that the relative distances of points in $\mathcal T$ are approximately preserved in $\mathcal T'$. For simplicity of exposition we assume that $\mathcal T$, $\mathcal T'$, and $\mathcal S$ are ordered sets (their elements are indexed). Our aim is to find a map $x:\mathcal T\mapsto \mathcal S$ with image $\mathcal T'$ that best preserves the relative distances of the points in $\mathcal T$ and $\mathcal T'$, i.e.
\begin{eqnarray}
\label{eq:norm}
x^* = \argmin_{x}\nbr{D(\mathcal T)-D(x(\mathcal T))}^2_2,
\end{eqnarray}
\noindent where $D(\mathcal T)$ is the matrix whose $(i,j)^{th}$ entry is the Euclidean distance between points indexed by $i$ and $j$ in $\mathcal T$. Note that finding $x^*$ is inherently a combinatorial optimization problem, since $\mathcal T'$ is itself a subset of $\mathcal S$, the scene point pattern. In \cite{CaeCaeSchBar06}, a generic point in $\mathcal T$ is modeled as a random variable ($X_i$), and a generic point in $\mathcal S$ is modeled as a possible realization of the random variable ($x_i$). As a result, a joint realization of all the random variables corresponds to a match between the template and the scene point patterns. A graphical model (see \cite{Lau96,Bis06}) is then defined on this set of random variables, whose edges are set according to the topology of a so-called 3-tree graph (any 3-tree that spans $\mathcal T$). A 3-tree is a graph obtained by starting with the complete graph on 3 vertices, $K_3$, and then adding new vertices which are connected only to those same 3 vertices.\footnote{Technically, connecting new vertices to the 3 nodes of the original $K_3$ graph is not required: it suffices to connect new vertices to any existent 3-clique.} Figure \ref{fig:3tree} shows an example of a 3-tree. The reasons claimed in \cite{CaeCaeSchBar06} for introducing 3-trees as a graph topology for the probabilistic graphical model are that (i) 3-trees are globally rigid in the plane and (ii) 3-trees are chordal\footnote{A chordal graph is one in which every cycle of length greater than 3 has a chord. A chord of a cycle is an edge not belonging to the cycle but which connects two nodes in the cycle (i.e.~ a ``shortcut'' in a cycle).} graphs. This implies (i) that the 3-tree model is a type of graph which is in some sense ``optimal'' (in a way that will be made clear in the next section in the context of the new graph we propose) and (ii) that 3-trees have a Junction Tree with fixed maximal clique size ($=4$); as a result it is possible to perform exact inference in polynomial time \cite{CaeCaeSchBar06}.

\begin{figure}
\begin{center}
\includegraphics[width=0.45\textwidth]{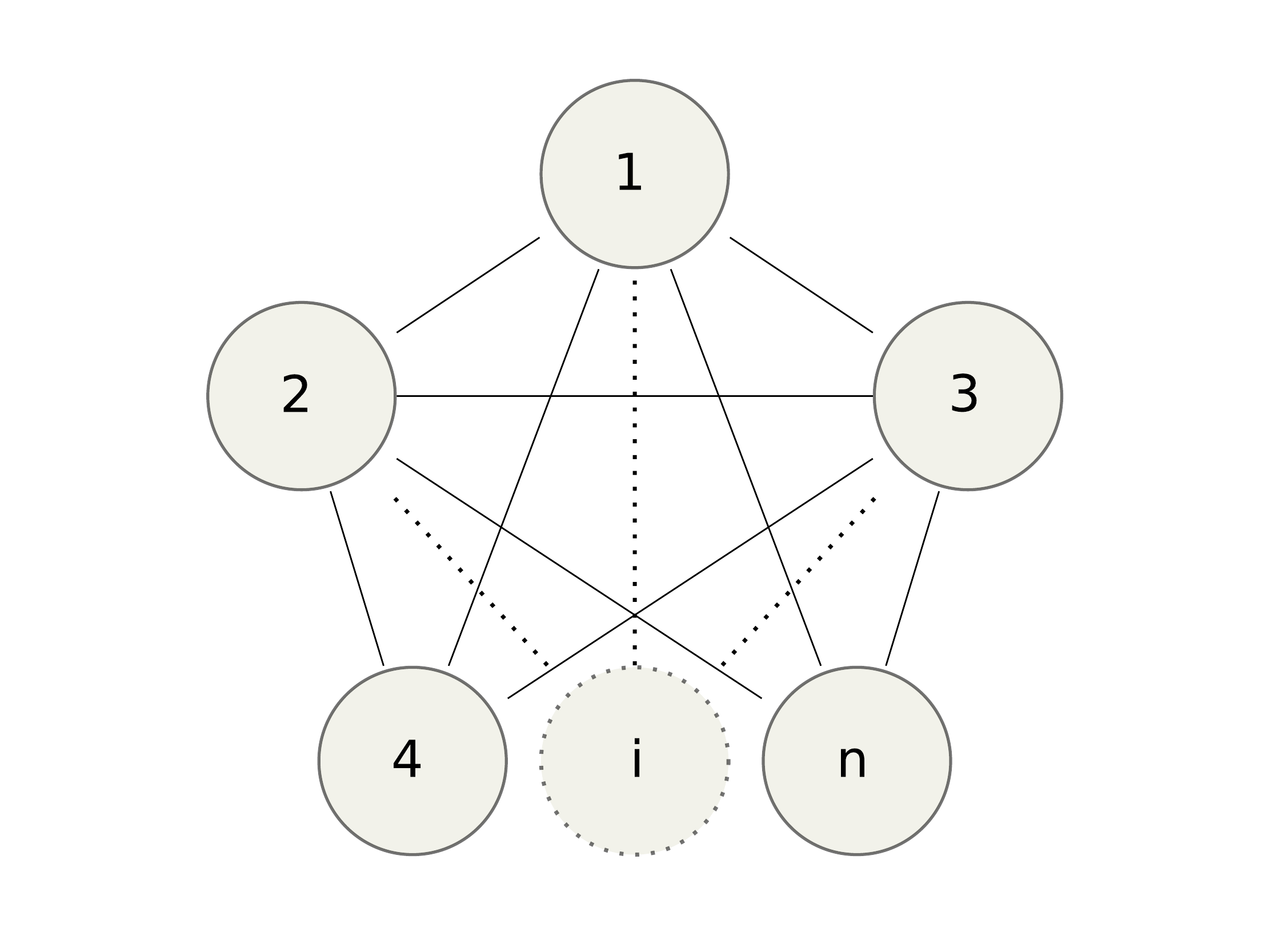}
\end{center}
\caption{An example of a 3-tree}
\label{fig:3tree}
\end{figure}

Potential functions are defined on pairs of neighboring nodes and are large if the difference between the distance of neighboring nodes in the template and the distance between the nodes they map to in the scene is small (and small if this difference is large). This favors isometric matchings. More precisely,
\begin{equation}
\label{eq:potential}
\psi_{ij}(X_i=x_i,X_j=x_j)=f(d(X_i,X_j)-d(x_i,x_j))
\end{equation}
\noindent where $f(\cdot)$ is typically some unimodal function peaked at zero (e.g.~a zero-mean Gaussian function) and $d(\cdot,\cdot)$ is the Euclidean distance between the corresponding points (for simplicity of notation we do not disambiguate between random variables and template points, or realizations and scene points). For the case of \emph{exact} matching, i.e.~when there exists an $x^*$ such that the minimal value in \eq{eq:norm} is zero, then $f(\cdot)=\delta(\cdot)$ (where $\delta(\cdot)$ is just the indicator function $1_{\lbrace 0 \rbrace}(\cdot)$). The potential function of a maximal clique ($\Psi$) is then simply defined as the product of the potential functions over its 6 ($=C^4_2$) edges (which will be maximal when every factor is maximal). It should be noted that the potential function of each edge is included in no more than \emph{one} of the cliques containing that edge.


For the case of exact matching (i.e.~no jitter), it is shown in \cite{CaeCaeSchBar06} that running the Junction Tree algorithm on the 3-tree graphical model with $f(\cdot)=\delta(\cdot)$ will actually find a MAP assignment which coincides with $x^*$, i.e.~such that $\nbr{D(\mathcal T)-D(x^*(\mathcal T))}^2_2=0$. This is due to the ``graph rigidity'' result, which tells us that equality of the lengths of the edges in the 3-tree and the edges induced by the matching in $\mathcal T'$ is sufficient to ensure the equality of the lengths of all pairs of points in $\mathcal T$ and $\mathcal T'$. This will be made technically precise in the next section, when we prove an analogous result for another graph.



\section{An Improved Graph}
\label{sec:improved}

Here we introduce another globally rigid graph which has the advantage of having a \emph{smaller maximal clique size}. Although the graph is \emph{not} chordal, we will show that exact inference \emph{is} tractable and that we will indeed benefit from the decrease in the maximal clique size. As a result we will be able to obtain optimality guarantees like those from \cite{CaeCaeSchBar06}. 

Our graph is constructed using Algorithm \ref{algorithm_ring}.


\begin{algorithm}
  \smallskip
  \caption{Graph Generation for $\mathcal G$}
  \smallskip
  \smallskip
  \begin{algorithmic}
   \STATE {\bfseries 1 } Create a cycle graph by traversing all the nodes in $\mathcal T$  (in any order)
   \STATE {\bfseries 2 } Connect all nodes whose distance in this cycle graph is two (i.e.~connect each node to its neighbor's neighbor)
  \end{algorithmic}
  \smallskip
\label{algorithm_ring}
\end{algorithm}

This algorithm will produce a graph like the one shown in Figure \ref{fig:graph}. We will denote by $\mathcal G$ the set of graphs that can be generated by Algorithm \ref{algorithm_ring}. $G=(V,E)$ will denote a generic graph in $\mathcal G$.

\begin{figure}
\begin{center}
\includegraphics[width=0.4\textwidth]{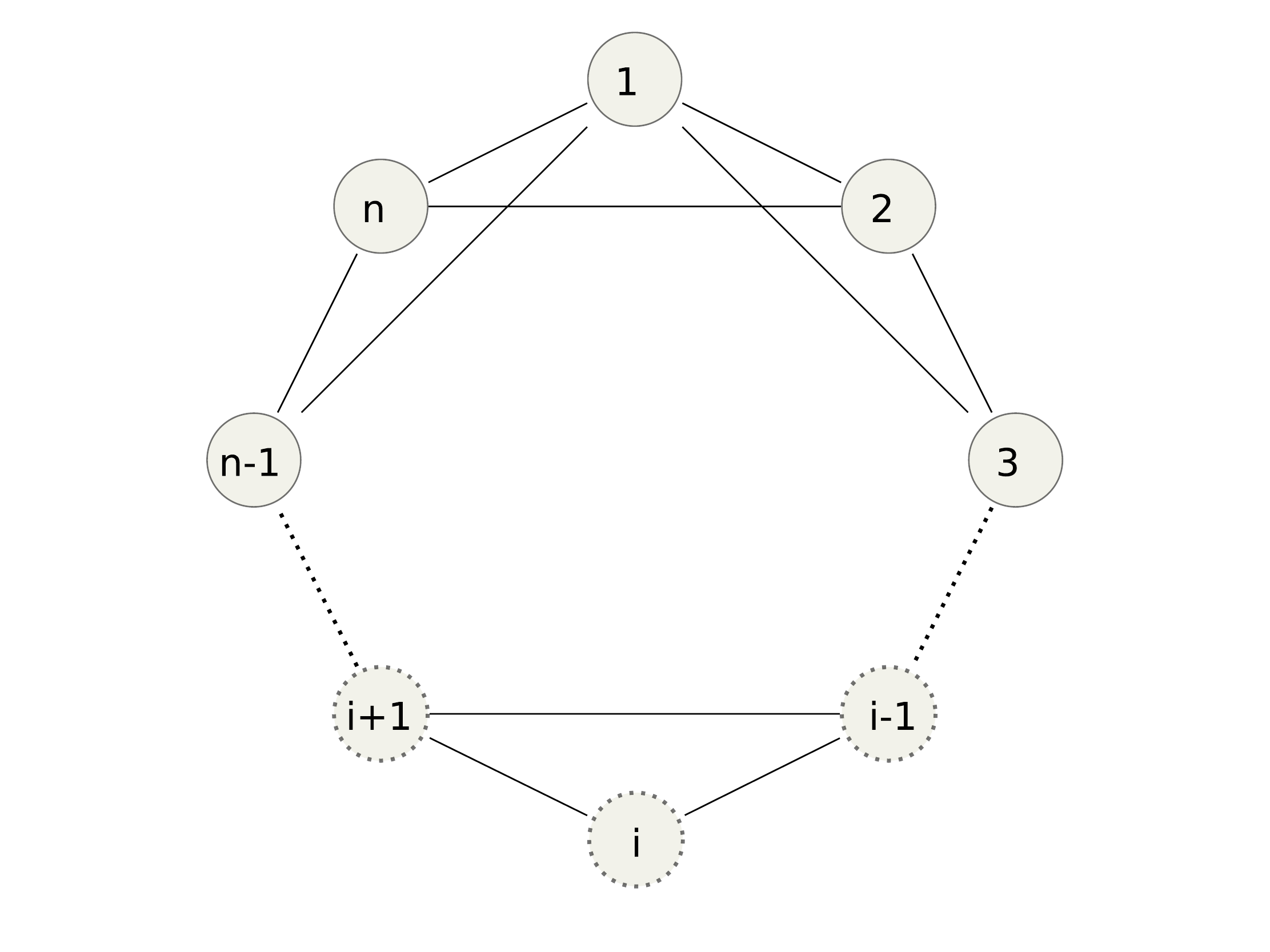}
\end{center}
\caption{The general form of the graph we consider, with $n$ nodes.}
\label{fig:graph}
\end{figure}

In order to present our results we need to start with the definition of a globally rigid graph:

\begin{definition}
A planar graph embedding $G$ is said to be globally rigid in $\mathbb R^2$ if the lengths of the edges uniquely determine the lengths of the edges of the graph complement of $G$.
\end{definition}
So our statements are really about \emph{graph embeddings} in $\mathbb R^2$, but for simplicity of presentation we will simply refer to these embeddings as ``graphs''.

This means that there are no degrees of freedom for the absent edges in the graph: they must all have specified and fixed lengths. To proceed we need a simple definition and some simple technical lemmas.

\begin{definition} A set of points is said to be in general position in $\mathbb R^{2}$ if no 3 points lie in a straight line.
\end{definition}

\begin{lemma} 
\label{lemma:2}
Given a set of points in general position in $\mathbb R^2$, if the distances from a point $P$ to two other fixed points are determined then $P$ can be in precisely two different positions.
\end{lemma}
\begin{proof}
\noindent Consider two circles, each centered at one of the two reference points with radii equal to the given distances to point $P$. These circles intersect at precisely two points (since the 3 points are not collinear). This proves the statement.
\end{proof}

The following lemma follows directly from lemma 1 in \cite{CaeCaeSchBar06}, and is stated without proof.

\begin{lemma} 
\label{lemma:3}
Given a set of points in general position in $\mathbb R^2$, if the distances from a point $P$ to three other fixed points are determined then the position of $P$ is uniquely determined. 
\end{lemma}

We can now present a proposition.

\begin{proposition} 
\label{prop}
Any graph $G \in \mathcal G$ arising from Algorithm \ref{algorithm_ring} is globally rigid in the plane if the nodes are in general position in the plane.
\end{proposition}

\begin{proof} Define a reference frame $S$ where points 1, 2 and $n$ have specific coordinates (we say that the points are ``determined''). We will show that all points then have determined positions in $S$ and therefore have determined relative distances, which by definition implies that the graph is globally rigid.

We proceed by contradition: assume there exists at least one undetermined point in the graph. Then we must have an undetermined point $i$ such that $i-1$ and $i-2$ are determined (since points 1 and 2 are determined). By virtue of lemma \ref{lemma:3}, points $i+1$ and $i+2$ must then be also undetermined (otherwise point $i$ would have determined distances from 3 determined points and as a result would be determined).

Let us now assume that \emph{only} points $i,i+1,i+2$ are undetermined. Then the only possible realizations for points $i$, $i+1$ and $i+2$ are their reflections with respect to the straight line which passes through points $i-1$ and $i+3$, since these are the only possible realizations that maintain the rigidity of the triangles $(i-1,i,i+1),(i,i+1,i+2),(i+1,i+2,i+3)$, since $i-1$ and $i+3$ are assumed fixed. However, since $i+4$ and $i-2$ are also fixed by assumption, this would break the rigidity of triangles $(i+2,i+3,i+4)$ and $(i,i-1,i-2)$. Therefore $i+3$ cannot be determined. This can then be considered as the base case in an induction argument which goes as follows. Assume \emph{only} $i,\dots,i+p$ are undetermined. Then, by reflecting these points over the line that joins $i-1$ and $i+p+1$ (which are fixed by assumption), we obtain the only other possible realization consistent with the rigidity of the triangles who have all their vertices in $i-1,\dots,i+p+1$. However, this realization is inconsistent with the rigidity of triangles $(i+p,i+p+1,i+p+2)$ and $(i,i-1,i-2)$, therefore $i+p+1$ must not be determined and by induction any point $j$ such that $j>i+2$ must not be determined, which contradicts the assumption that $n$ is determined. As a result, the assumption that there is at least one undetermined point in the graph is false. This implies that the graph has all points determined in $S$, and therefore all relative distances are determined and by definition the graph is globally rigid. This proves the statement.\end{proof}

Although we have shown that graphs $G\in \mathcal G$ are globally rigid, notice that they are \emph{not} chordal. For the graph in Figure \ref{fig:graph}, the cycles $(1,3,5\dots,n-1,1)$ and $(2,4,6\dots,n,2)$ have no chord. Moreover, triangulating this graph in order to make it chordal will necessarily increase (to at least 4) the maximal clique size (which is not sufficient for our purposes since we arrive at the case of \cite{CaeCaeSchBar06}).


Instead, consider the clique graph formed by $G\in \mathcal G$. If there are $n$ nodes, the clique graph will have cliques $(1,2,3),(2,3,4),\dots,(n-2,n-1,n),(n-1,n,1),(n,1,2)$. This clique graph forms a cycle, which is depicted in Figure \ref{fig:cliquegraph}.\footnote{Note that if we connected \emph{every} clique whose nodes intersected, the clique graph would no longer form a cycle; here we have only formed enough connections so that the intersection of any two cliques is shared by the cliques on at least one path between them (similar to the running intersection property for Junction Trees).}

\begin{figure}
\begin{center}
\includegraphics[width=0.45\textwidth]{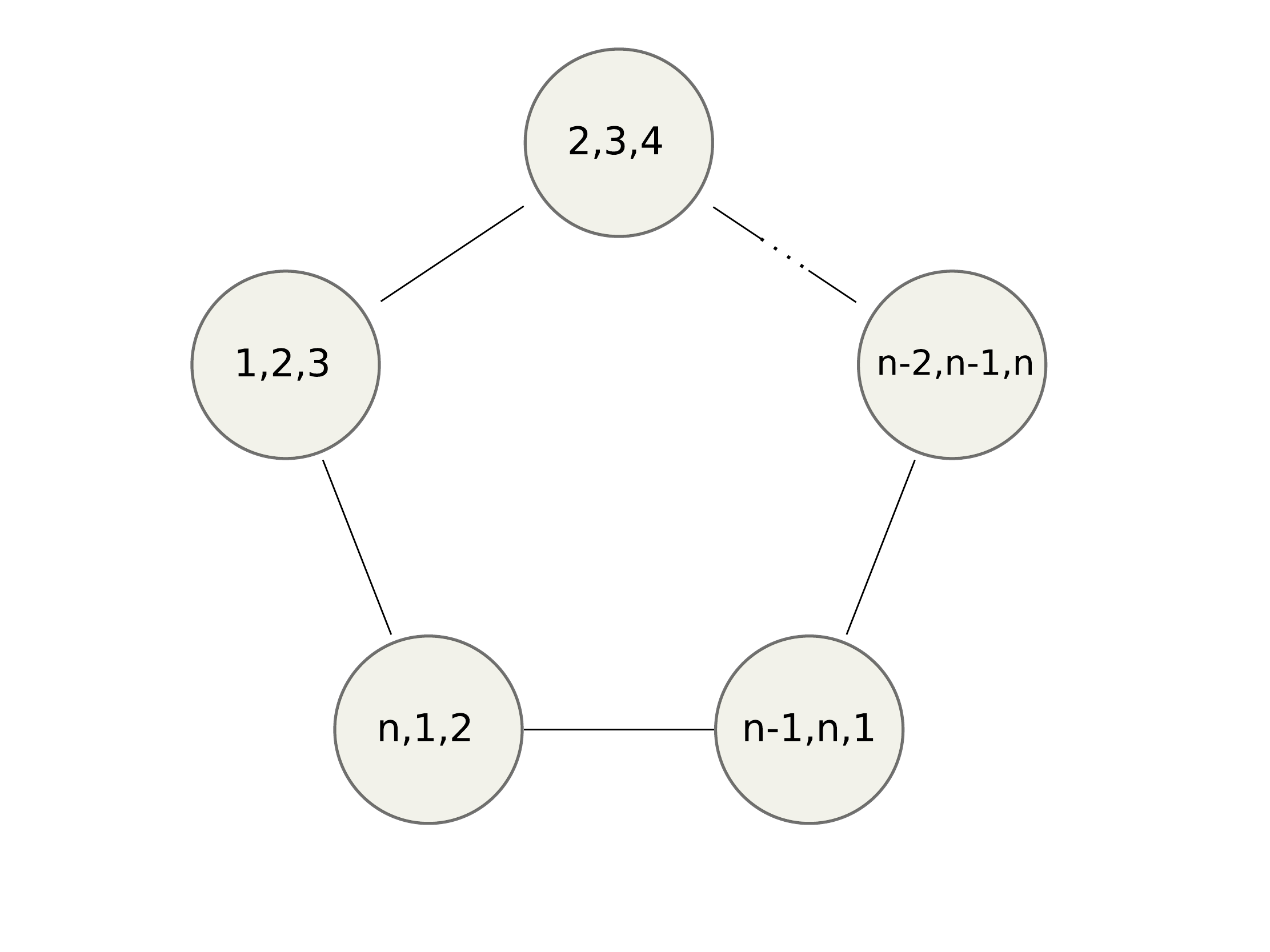}
\end{center}
\caption{The clique graph obtained from the graph in Figure \ref{fig:graph}.}
\label{fig:cliquegraph}
\end{figure}

We now draw on results first obtained by Weiss \cite{Weiss00}, and confirmed elsewhere \cite{IhlFisWil04}. There it is shown that, for graphical models with a single cycle, belief propagation converges to the optimal MAP assignment, although the computed marginals may be incorrect. Note that for our purposes, this is precisely what is needed: we are after the most likely joint realization of the set of random variables, which corresponds to the best match between the template and the scene point patterns. Max-product belief propagation \cite{YedFreWei00} in a cycle graph like the one shown in Figure \ref{fig:cliquegraph} amounts to computing the following messages, iteratively:
\begin{multline}
m_{i\mapsto i+1}(U_i\cap U_{i+1})\\=\max_{U_i\backslash U_{i+1}}\Psi(U_i)m_{i-1\mapsto i}(U_i\cap U_{i-1}),
\end{multline}
\noindent where $U_i$ is the set of singleton variables in clique node $i$, $\Psi(U_i)$ the potential function for clique node $i$ and $m_{i\mapsto i+1}$ the message passed from clique node $i$ to clique node $i+1$. Upon reaching the convergence monitoring threshold, the optimal assignment for singleton variable $j$ in clique node $i$ is then computed by $\argmax_{U_i\backslash j}\Psi(U_i)m_{i-1\mapsto i}(U_i\cap U_{i-1})m_{i+1\mapsto i}(U_i\cap U_{i+1})$.

Unfortunately, the above result is only shown in \cite{Weiss00} when the \emph{graph itself} forms a cycle, whereas we only have that the \emph{clique graph} forms a cycle. However, it is possible to show that the result still holds in our case, by considering a new graphical model in which the \emph{cliques themselves} form the nodes, whose cliques are now just the edges in the clique graph. The result from \cite{Weiss00} can now be used to prove that belief propagation in \emph{this} graph converges to the optimal MAP assignment, which (by appropriately choosing potential functions for the new graph), implies that belief propagation should converge to the optimal solution in the \emph{original} graph also.

To demonstrate this, we need not only show that belief propagation in the new model converges to the optimal assignment, but also that belief propagation in the new model is \emph{equivalent} to belief propagation in the original model.

\begin{proposition}
\label{lemma:originaltopairwise}
The original clique graph (Figure \ref{fig:cliquegraph}) can be transformed into a model containing only pairwise potentials, whose optimal MAP assignment is the same as the original model's.
\end{proposition}

\begin{proof}
Consider a clique ``node'' $C_1 = (X_1,X_2,X_3)$ (in the original graph), whose neighbors share exactly two of its nodes (for instance $C_2 = (X_2,X_3,X_4)$). Where the domain for each node in the original graph was simply $\left\lbrace 1,2\dots |\mathcal S|\right\rbrace$, the domain for each ``node'' in our new graph simply becomes $\left\lbrace 1,2\dots |\mathcal S|\right\rbrace^3$.

In this setting, it is no longer possible to ensure that the assignment chosen for each ``node'' is consistent with the assignment to its neighbor -- that is, for an assignment $(x_1,x_2,x_3)$ to $C_1$, and $(x_2',x_3',x_4)$ to $C_2$, we cannot guarantee that $x_2=x_2'$, or $x_3=x_3'$. Instead, we will simply define the potential functions on this new graph in such a way that the optimal MAP assignment implicitly ensures this equality. Specifically, we shall define the potential functions as follows: for two cliques $C_I=(I_1,I_2,I_3)$ and $C_J=(J_1,J_2,J_3)$ in the original graph (which share two nodes, say $(I_2, I_3)$ and $(J_1, J_2)$), define the \emph{pairwise} potential for the clique ($\Psi'_{I,J}$) in the new graph as follows:
\begin{multline}
 \Psi'_{I,J}(i_{(123)}, j_{(123)})=\\ \left\lbrace \begin{array}{ll} \Psi_I(i_1,i_2,i_3) & \mbox{if\ } (i_2,i_3)=(j_1,j_2),\\ \rho & \mbox{otherwise} \end{array} \right.
\end{multline}
Where $\Psi_I$ is simply the clique potential for the $I^{th}$ clique in the \emph{original} graph; $i_{(123)} \in \text{domain}(I_1)\times\text{domain}(I_2)\times\text{domain}(I_3)$ (sim.~for $j_{(123)}$). That is, we are setting the pairwise potential to simply be the original potential of \emph{one} of the cliques if the assignments are compatible, and $\rho$ otherwise. If we were able to set $\rho=0$, we would guarantee that the optimal MAP assignment was exactly the optimal MAP assignment in the original graph -- however, this is not possible, since the result of \cite{IhlFisWil04} only holds when the potential functions have finite dynamic range. Hence we must simply choose $\rho$ sufficiently small so that the optimal MAP assignment cannot possibly contain an incompatible match -- it is clear that this is possible, for example $\rho = \left(\prod_{C} \max_{\mathbf{x}_C}\Psi_C(\mathbf{x}_C) \right)^{-1}$ will do.

The result of \cite{Weiss00} now implies that belief propagation in \emph{this} graph will converge to the optimal MAP assignment, which we have shown is equal to the optimal MAP assignment in the \emph{original} graph.
\end{proof}

\begin{proposition}
\label{lemma:pairwisetooriginal}
The messages passed in the new model are equivalent to the messages passed in the original model, except for repetition along one axis.
\end{proposition}

\begin{proof}
We use induction on the number of iterations. First, we must show that the outgoing messages are the same during the first iteration (during which the incoming messages are not included). We will denote by $m^i_{(X_1,X_2,X_3)\mapsto (X_2,X_3,X_4)}$ the message from $(X_1,X_2,X_3)$ to $(X_2,X_3,X_4)$ during the $i^{th}$ iteration:

\begin{multline}
m^1_{(X_1,X_2,X_3)\mapsto (X_2,X_3,X_4)}(x_2,x_3)\\
 = \max_{X_1}\Psi_{(X_1,X_2,X_3)}(x_1,x_2,x_3),
\end{multline}
\begin{multline}
m^1_{(X_{(123)}, X_{(234)})\mapsto (X_{(234)}, X_{(345)})}(x_{123}, x_{234})\\
\begin{array}{ll}
 =& \max_{X_{(123)}}\Psi'_{X_{(123)},X_{(234)}} (x_{(123)}, x_{(234)})\\
 =& 1 \times \max_{X_1}\Psi_{(X_1,X_2,X_3)}(x_1,x_2,x_3)\\
 =& m^1_{(X_1,X_2,X_3)\mapsto (X_2,X_3,X_4)}(x_2,x_3).\end{array}
\end{multline}

This result only holds due to the fact that $\rho$ will never be chosen when maximizing along any axis. We now have that the messages are equal during the first iteration (the only difference being that the message for the new model is repeated along one axis).\footnote{To be completely precise, the message for the new model is actually a function of only a \emph{single} variable -- $X_{(234)}$. By ``repeated along one axis'', we mean that for any given $(x_2,x_3,x_4)\in\text{domain}(X_2)\times\text{domain}(X_3)\times\text{domain}(X_4)$, the message at this point is independent of $x_4$, which therefore has no effect when maximizing.} Next, suppose during the $(n-1)^{st}$ iteration, the messages (for both models) are equal to $\kappa(x_1,x_2)$. Then for the $n^{th}$ iteration we have:
\begin{multline}
m^n_{(X_1,X_2,X_3)\mapsto (X_2,X_3,X_4)}(x_2,x_3)\\
 = \max_{X_1}\left\lbrace\Psi_{(X_1,X_2,X_3)}(x_1,x_2,x_3)\kappa(x_1,x_2)\right\rbrace,
\end{multline}
\begin{multline}
m^n_{(X_{(123)}, X_{(234)})\mapsto (X_{(234)}, X_{(345)})}(x_{123}, x_{234})\\
\begin{array}{ll}
 =& \max_{X_{(123)}}\left\lbrace\Psi'_{X_{(123)},X_{(234)}} (x_{(123)}, x_{(234)})\kappa(x_1,x_2)\right\rbrace\\
 =& 1 \times \max_{X_1}\left\lbrace\Psi_{(X_1,X_2,X_3)}(x_1,x_2,x_3)\kappa(x_1,x_2)\right\rbrace\\
 =& m^n_{(X_1,X_2,X_3)\mapsto (X_2,X_3,X_4)}(x_2,x_3).\end{array}
\end{multline}
Hence the two message passing schemes are equivalent by induction.
\end{proof}



We can now state our main result:

\begin{theorem} Let $G\in \mathcal G$ be a graph generated according to the procedure described in Algorithm \ref{algorithm_ring}. Assume that there is a perfect isometric instance of $\mathcal T$ within the scene point pattern $\mathcal S$. Then the MAP assignment $x^*$ obtained by running belief propagation over the clique graph derived from G is such that $\nbr{D(\mathcal T)-D(x^*(\mathcal T))}^2_2=0$.
\end{theorem}

\begin{proof} For the exact matching case, we simply set $f(\cdot)=\delta(\cdot)$ in \eq{eq:potential}. Now, for a graph $G\in \mathcal G$ given by Algorithm \ref{algorithm_ring}, the clique graph will be simply a cycle, as shown in Figure \ref{fig:cliquegraph}, and following  propositions \ref{lemma:originaltopairwise} and \ref{lemma:pairwisetooriginal} as well as the already mentioned result from \cite{Weiss00}, belief propagation will find the correct MAP assignment $x^*$, i.e.
\begin{equation}
\begin{split}
\label{eq:partial}
x^* &= \argmax_x P_{G}(X=x)\\&= \argmax_x \prod_{i,j:(i,j)\in E} \delta(d(X_i,X_j)-d(x_i,x_j)),
\end{split}
\end{equation}
where $P_G$ is the probability distribution for the graphical model induced by the graph $G$. Now, we need to show that $x^*$ \emph{also} maximizes the criterion which ensures isometry, i.e.~we need to show that the above implies
\begin{equation}
\begin{split}
\label{eq:total}
x^{*} &= \argmax_x P_{\text{complete}}(X=x)\\&= \argmax_x \prod_{i,j} \delta(d(X_i,X_j)-d(x_i,x_j)),
\end{split}
\end{equation}

\noindent where $P_{\text{complete}}$ is the probability distribution of the graphical model induced by the complete graph. Note that $x^*$ must be such that the lengths of the edges in $E$ are precisely equal to the lengths of the edges in $E_{\mathcal T'}$ (i.e.~the edges induced in $\mathcal S$ from $E$ by the map $X=x^*$). By the global rigidity of $G$, the lengths of $\bar{E}$ must then be also precisely equal to the lengths of $\bar{E}_{\mathcal T'}$. This implies that $\prod_{i,j:(i,j)\in \bar{E}} \delta(d(X_i,X_j)-d(x^*_i,x^*_j))=1$. Since  \eq{eq:total} can be expanded as

\begin{equation}
\begin{split}
\label{eq:expanded}
x^{*} = \argmax_x \Bigl\lbrace&\prod_{i,j:(i,j)\in E} \delta(d(X_i,X_j)-d(x_i,x_j))\\&\prod_{i,j:(i,j)\in \bar{E}} \delta(d(X_i,X_j)-d(x_i,x_j))\Bigr\rbrace,
\end{split}
\end{equation}

\noindent it becomes clear that $x^*$ will also maximize \eq{eq:total}. This proves the statement.\end{proof}

\section{Experiments}

We have set up a series of experiments comparing the proposed model to that of \cite{CaeCaeSchBar06}. Here we compare graphs of the type shown in Figure \ref{fig:graph} to graphs of the type shown in Figure \ref{fig:3tree}.

The parameters used in our experiments are as follows:

$\epsilon$ -- this parameter controls the noise-level used in our model. Here we apply Gaussian noise to each of the points in $\mathcal T$ (with standard deviation $\epsilon$ in each axis). We have run our experiments on a range of noise levels between $0$ and $4/256$ (where the original points in $\mathcal T$ are chosen randomly between $0$ and $1$). Note that this is the same as the setting used in \cite{CaeCaeSchBar06}.

Potential functions $\psi_{ij}(X_i=x_i,X_j=x_j)=f(d(X_i,X_j)-d(x_i,x_j))$ -- as in \cite{CaeCaeSchBar06}, we use a Gaussian function, i.e.~$\exp \left( \frac{(d(X_i,X_j)-d(x_i,x_j))^2}{2\sigma^2}\right)$. The parameter $\sigma$ is fixed beforehand as $\sigma = 0.4$ for the synthetic data, and $\sigma = 150$ for the real-world data (as is done in \cite{CaeCaeSchBar06}).

Dynamic range -- as mentioned in section \ref{sec:background}, the potential function $\Psi(\mathbf{x})$ is simply the product of $\psi_{ij}(X_i=x_i,X_j=x_j)$ for all edges $(i,j)$ in $\mathbf{x}$ (here each maximal clique $\mathbf{x}$ contains 3 edges). The \emph{dynamic range} of a function is simply defined as its maximum value divided by its minimum value (i.e.~$\frac{\max_{\mathbf{x}}\Psi(\mathbf{x})}{\min_{\mathbf{x}}\Psi(\mathbf{x})}$). In order to prove convergence of our model, it is necessary that the dynamic range of our potential function is finite \cite{IhlFisWil04}. Therefore, rather than using $\Psi(\mathbf{x})$ directly, we use $\Psi'(\mathbf{x}) = (1/d) + (1-1/d)\Psi(\mathbf{x})$. This ensures that the dynamic range of our model is no larger than $d$, and that $\Psi' \rightarrow \Psi$ as $d \rightarrow \infty$. In practice, we found that varying this parameter did not have a significant effect on convergence time. Hence we simply fixed a large finite value ($d=1000$) throughout.

MSE-cutoff -- in order to determine the point at which belief propagation has converged, we compute the marginal distribution of every clique, and compare it to the marginal distribution after the previous iteration. Belief propagation is terminated when this mean-squared error is less than a certain cutoff value for every clique in the graph. When choosing the mode of the marginal distributions after convergence, if two values differ by less than the square-root of this cutoff, both of them are considered as possible MAP-estimates (although this was rarely an issue when the cutoff was sufficiently small). We found that as $|\mathcal S|$ increased, the mean squared error between iterations tended to be smaller, and therefore that smaller cutoff values should be used in these instances. Indeed, although the number of viable matches increases as $|\mathcal S|$ increases, the distributions increase in sparsity at an even faster rate -- hence the distributions tend to be \emph{less peaked} on average, and changes are likely to have less effect on the mean squared error. Hence we decreased the cutoff values by a factor of 10 when $|\mathcal S| \ge 30$.\footnote{Note that this is not a parameter in \cite{CaeCaeSchBar06}, in which only a single iteration is ever required.}

The clique graph in which messages are passed by our belief propagation algorithms is exactly that shown in Figure \ref{fig:cliquegraph}. It is worth noting, however, that we also tried running belief propagation using a clique graph in which messages were passed between \emph{all} intersecting cliques; we found that this made no difference to the performance of the algorithm,\footnote{Apart from one slight difference: including the additional edges appears to provide convergence in fewer iterations. However, since the number of messages being passed is doubled, the overall running-time for both clique graphs was ultimately similar.} and we have therefore restricted our experiments to the clique graph of Figure \ref{fig:cliquegraph} in respect of its optimality guarantees.

For the sake of running-time comparison, we implemented the proposed model, as well as that of \cite{CaeCaeSchBar06} using the \emph{Elefant} belief propagation libraries in Python.\footnote{\scriptsize{\texttt{http://elefant.developer.nicta.com.au/}}} However, to ensure that the results presented are consistent with those of \cite{CaeCaeSchBar06}, we simply used code that the authors provided when reporting the matching accuracy of their model.


Figure \ref{fig:class} compares the matching accuracy of our model with that of \cite{CaeCaeSchBar06} for $|\mathcal S| = 10,20,30,$ and $40$ (here we fix $|\mathcal T| = 10$). The performance of our algorithm is indistinguishable from that of the Junction Tree algorithm.

\begin{figure}
\begin{center}
\small{
\input{plot_class10}
\input{plot_class20}
\input{plot_class30}
\input{plot_class40}
}
\end{center}
\caption{Matching accuracy of our model against that of \cite{CaeCaeSchBar06}. The performance of our model is statistically indistinguishable from that of \cite{CaeCaeSchBar06} for all noise levels. The error bars indicate the average and standard error of 50 experiments.}
\label{fig:class}
\end{figure}

Figures \ref{fig:time} and \ref{fig:classm} show the running-time and matching accuracy (respectively) of our model, as we vary the mean-squared error cutoff. Obviously, it is necessary to use a sufficiently low cutoff in order to ensure that our model has converged, but choosing too small a value may adversely effect its running-time. We found that the mean-squared error varied largely during the first few iterations, and we therefore enforced a minimum number of iterations (here we chose at least 5) in order to ensure that belief-propagation was not terminated prematurely. Figure \ref{fig:time} reveals that the running-time is not significantly altered when increasing the MSE-cutoff -- revealing that the model has almost always reached the lower cutoff value after 5 iterations (in which case we should expect a speed-up of precisely $|\mathcal S|/5$). Furthermore, decreasing the MSE-cutoff does not significantly improve the matching accuracy for larger point sets (Figure \ref{fig:classm}), so choosing the lower cutoff does little harm if running-time is major a concern. Alternately, the Junction Tree model (which only requires a single iteration), took (for $S=10$ to $40$), 3, 44, 250, and 1031 seconds respectively. These models differ only in the topology of the network (see section \ref{sec:improved}), and the size of the messages being passed; our method easily achieves an order of magnitude improvement for large networks.\footnote{In fact, the speed-up appears to be \emph{more} than an order of magnitude for the large graphs, which is likely a side effect of the large memory requirements of the Junction Tree algorithm.}


\begin{figure}
\begin{center}
\small{
\input{plot_time10}
\input{plot_time20}
\input{plot_time30}
\input{plot_time40}
}
\end{center}
\caption{Running-time of our model as the jitter varies, for different MSE-cutoffs. Speed-ups are almost exactly one order of magnitude.}
\label{fig:time}
\end{figure}


\begin{figure}
\begin{center}
\small{
\input{plot_classm10}
\input{plot_classm20}
\input{plot_classm30}
\input{plot_classm40}
}
\end{center}
\caption{Matching accuracy of our model as the MSE-cutoff varies. This figure suggests that the higher cutoff value should be sufficient when matching larger point sets.}
\label{fig:classm}
\end{figure}
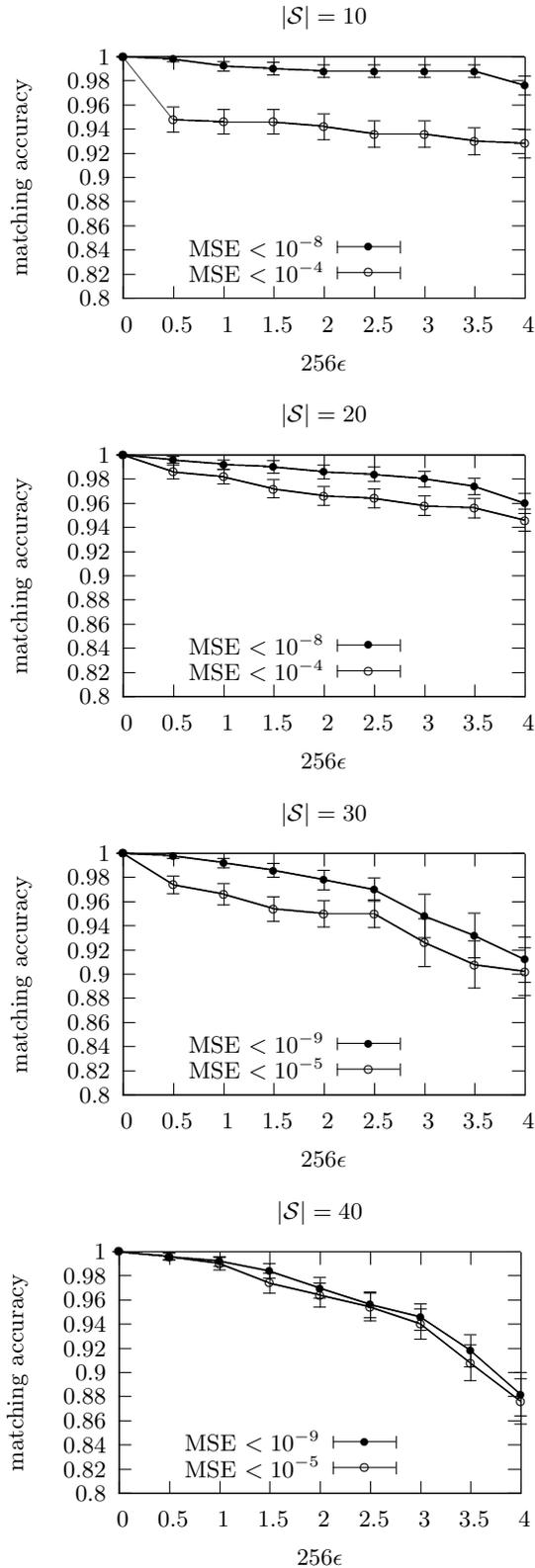

Finally, we present matching results using data from the CMU house sequence.\footnote{\scriptsize{\texttt{http://vasc.ri.cmu.edu/idb/html/motion/house/index.html}}.} In this dataset, 30 points corresponding to certain features of the house are available over 111 frames. Figure \ref{fig:matches} shows the 71st and the last (111th) frames from this dataset. Overlayed on these images are the 30 significant points, together with the matches generated by the Junction Tree algorithm and our own (matching the first 20 points); in this instance, the Junction Tree algorithm correctly matched 16 points, and ours 17. Figure \ref{fig:house} shows how accurately points between frames are matched as the baseline (separation between frames) varies. We also vary the number of points in the template set ($|\mathcal T|$) from 15 to 30. Our model seems to outperform the Junction Tree model for small baselines, whereas for large baselines and larger point sets the Junction Tree model seems to be the best. It is however difficult to draw conclusions from both models in these cases, since they are designed for the near-isometric case, which is violated for larger baselines.


\section{Conclusions}

We have shown that the near-isometric point pattern matching problem can be solved much more efficiently than what is currently reported as the state-of-the-art, while maintaining the same optimality guarantees for the noiseless case and comparable accuracy for the noisy case. This was achieved by identifying a new type of graph with the same global rigidity property of previous graphs but in which exact inference is far more efficient. Although exact inference is not \emph{directly} possible by means of the Junction Tree algorithm since the graph is not chordal, what we managed to show is that loopy belief propagation in such graph does converge to the optimal solution in a sufficiently small number of iterations. In the end, the advantage of the smaller clique size of our model dominates the disadvantage caused by the need for more than a single iteration.

\begin{figure}
 \includegraphics[width=0.48\textwidth,height=0.1667\textwidth]{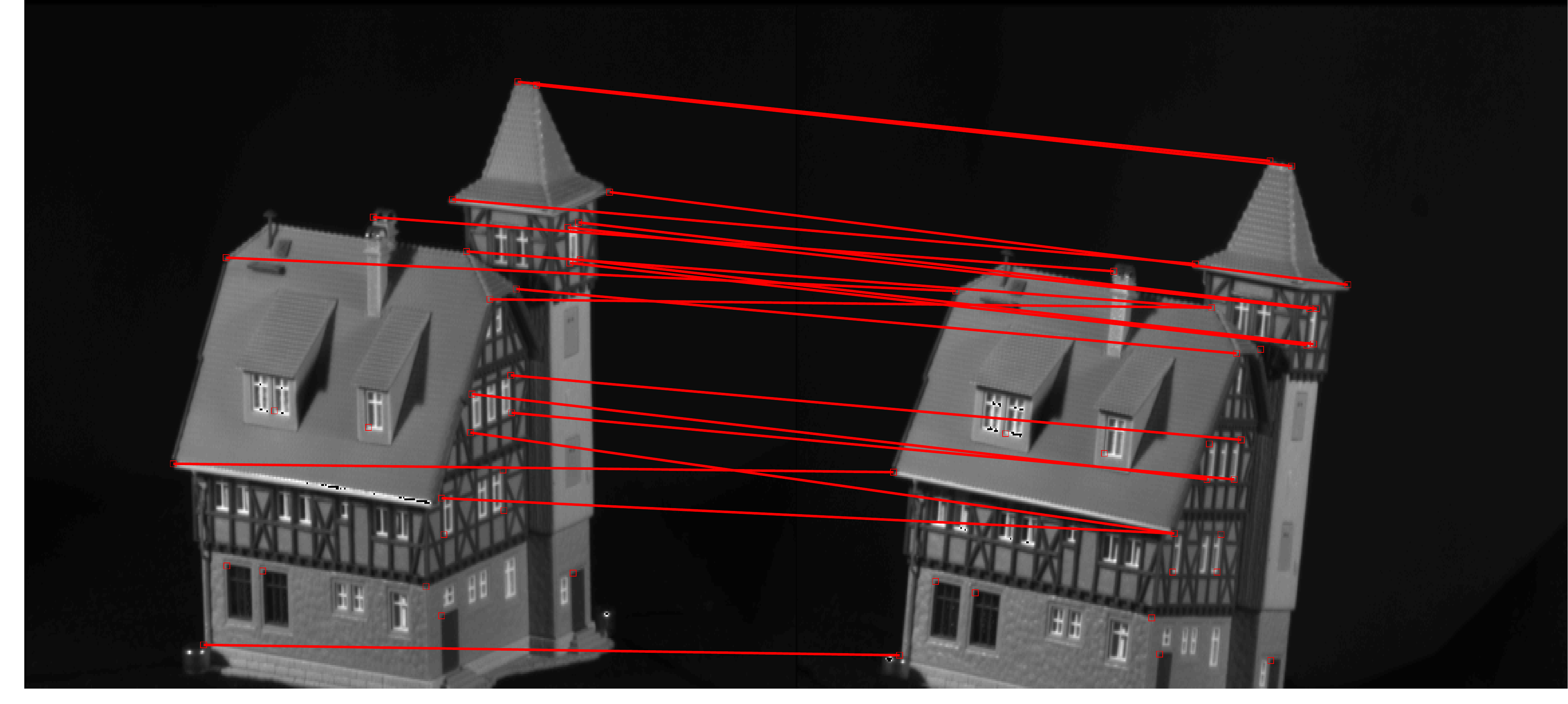}\\
 \includegraphics[width=0.48\textwidth,height=0.1667\textwidth]{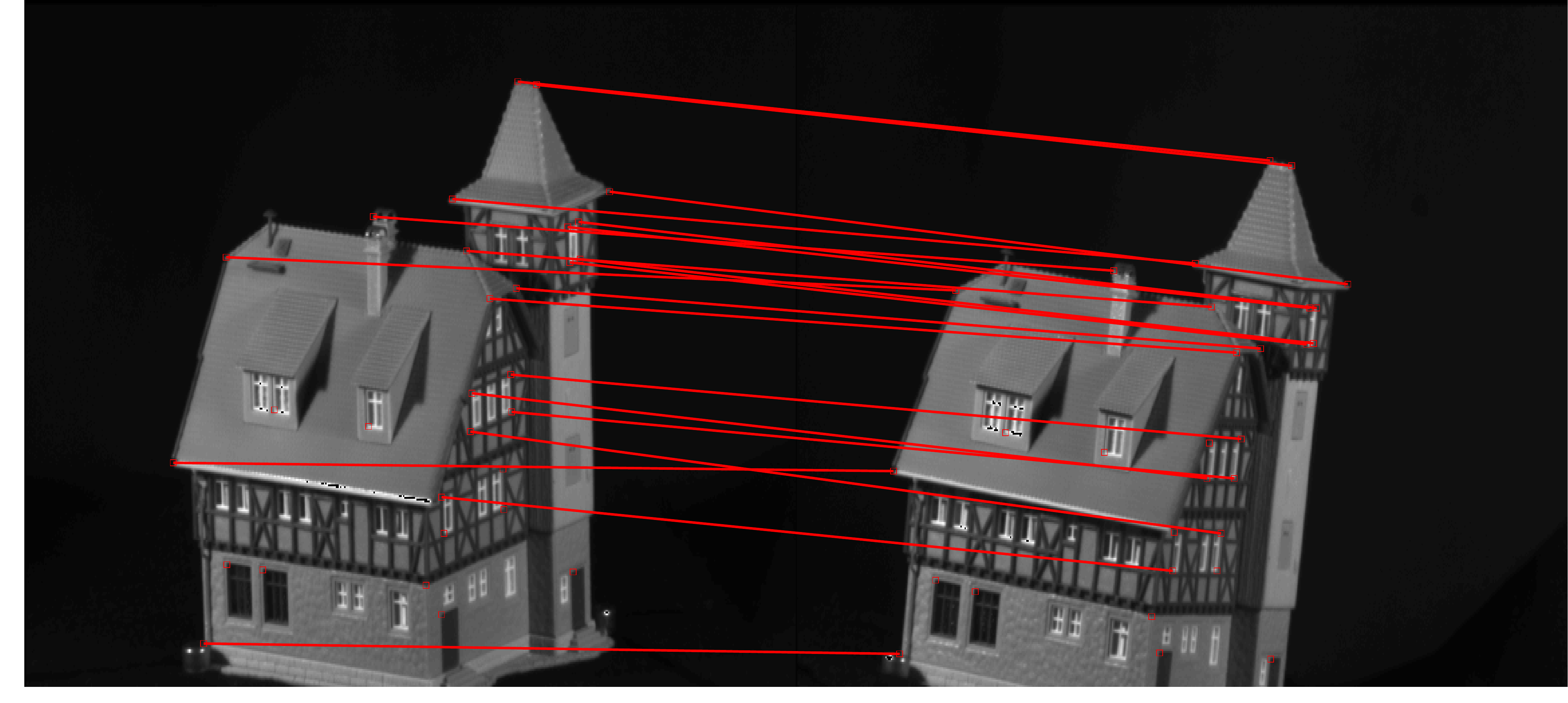}
\caption{Top: points matched using the Junction Tree algorithm (the points in the left frame were matched to the corresponding points in the right frame); 16 points are correctly matched by this algorithm. Bottom: points matched using our algorithm; 17 points are correctly matched.}
\label{fig:matches}
\end{figure}

\begin{figure}[ht]
\begin{center}
\small{
\input{plot_house15}
\input{plot_house20}
\input{plot_house25}
\input{plot_house30}
}
\end{center}
\caption{Matching accuracy of our model using the ``house'' dataset, as the baseline (separation between frames) varies.}
\label{fig:house}
\end{figure}
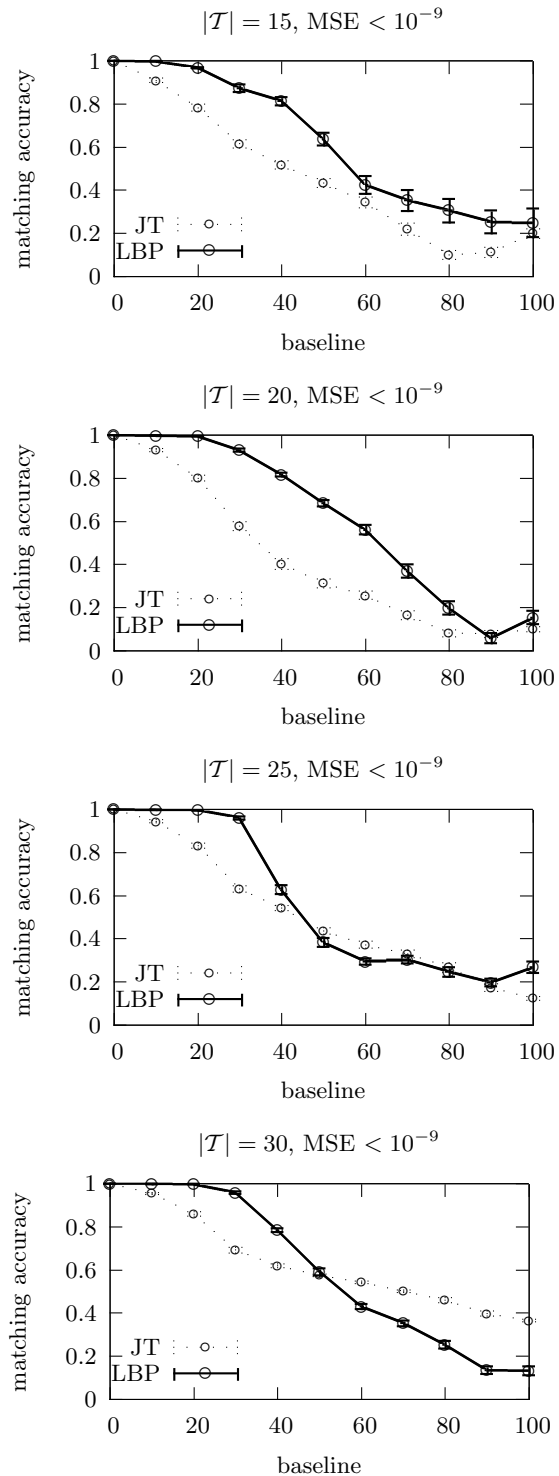

\end{document}

%% file: plot_class10.tex
\setlength{\unitlength}{0.240900pt}
\ifx\plotpoint\undefined\newsavebox{\plotpoint}\fi
\sbox{\plotpoint}{\rule[-0.200pt]{0.400pt}{0.400pt}}%
\begin{picture}(900,629)(0,0)
\sbox{\plotpoint}{\rule[-0.200pt]{0.400pt}{0.400pt}}%
\put(201.0,123.0){\rule[-0.200pt]{4.818pt}{0.400pt}}
\put(181,123){\makebox(0,0)[r]{ 0.8}}
\put(819.0,123.0){\rule[-0.200pt]{4.818pt}{0.400pt}}
\put(201.0,161.0){\rule[-0.200pt]{4.818pt}{0.400pt}}
\put(181,161){\makebox(0,0)[r]{ 0.82}}
\put(819.0,161.0){\rule[-0.200pt]{4.818pt}{0.400pt}}
\put(201.0,200.0){\rule[-0.200pt]{4.818pt}{0.400pt}}
\put(181,200){\makebox(0,0)[r]{ 0.84}}
\put(819.0,200.0){\rule[-0.200pt]{4.818pt}{0.400pt}}
\put(201.0,238.0){\rule[-0.200pt]{4.818pt}{0.400pt}}
\put(181,238){\makebox(0,0)[r]{ 0.86}}
\put(819.0,238.0){\rule[-0.200pt]{4.818pt}{0.400pt}}
\put(201.0,277.0){\rule[-0.200pt]{4.818pt}{0.400pt}}
\put(181,277){\makebox(0,0)[r]{ 0.88}}
\put(819.0,277.0){\rule[-0.200pt]{4.818pt}{0.400pt}}
\put(201.0,315.0){\rule[-0.200pt]{4.818pt}{0.400pt}}
\put(181,315){\makebox(0,0)[r]{ 0.9}}
\put(819.0,315.0){\rule[-0.200pt]{4.818pt}{0.400pt}}
\put(201.0,353.0){\rule[-0.200pt]{4.818pt}{0.400pt}}
\put(181,353){\makebox(0,0)[r]{ 0.92}}
\put(819.0,353.0){\rule[-0.200pt]{4.818pt}{0.400pt}}
\put(201.0,392.0){\rule[-0.200pt]{4.818pt}{0.400pt}}
\put(181,392){\makebox(0,0)[r]{ 0.94}}
\put(819.0,392.0){\rule[-0.200pt]{4.818pt}{0.400pt}}
\put(201.0,430.0){\rule[-0.200pt]{4.818pt}{0.400pt}}
\put(181,430){\makebox(0,0)[r]{ 0.96}}
\put(819.0,430.0){\rule[-0.200pt]{4.818pt}{0.400pt}}
\put(201.0,469.0){\rule[-0.200pt]{4.818pt}{0.400pt}}
\put(181,469){\makebox(0,0)[r]{ 0.98}}
\put(819.0,469.0){\rule[-0.200pt]{4.818pt}{0.400pt}}
\put(201.0,507.0){\rule[-0.200pt]{4.818pt}{0.400pt}}
\put(181,507){\makebox(0,0)[r]{ 1}}
\put(819.0,507.0){\rule[-0.200pt]{4.818pt}{0.400pt}}
\put(201.0,123.0){\rule[-0.200pt]{0.400pt}{4.818pt}}
\put(201,82){\makebox(0,0){ 0}}
\put(201.0,487.0){\rule[-0.200pt]{0.400pt}{4.818pt}}
\put(281.0,123.0){\rule[-0.200pt]{0.400pt}{4.818pt}}
\put(281,82){\makebox(0,0){ 0.5}}
\put(281.0,487.0){\rule[-0.200pt]{0.400pt}{4.818pt}}
\put(361.0,123.0){\rule[-0.200pt]{0.400pt}{4.818pt}}
\put(361,82){\makebox(0,0){ 1}}
\put(361.0,487.0){\rule[-0.200pt]{0.400pt}{4.818pt}}
\put(440.0,123.0){\rule[-0.200pt]{0.400pt}{4.818pt}}
\put(440,82){\makebox(0,0){ 1.5}}
\put(440.0,487.0){\rule[-0.200pt]{0.400pt}{4.818pt}}
\put(520.0,123.0){\rule[-0.200pt]{0.400pt}{4.818pt}}
\put(520,82){\makebox(0,0){ 2}}
\put(520.0,487.0){\rule[-0.200pt]{0.400pt}{4.818pt}}
\put(600.0,123.0){\rule[-0.200pt]{0.400pt}{4.818pt}}
\put(600,82){\makebox(0,0){ 2.5}}
\put(600.0,487.0){\rule[-0.200pt]{0.400pt}{4.818pt}}
\put(680.0,123.0){\rule[-0.200pt]{0.400pt}{4.818pt}}
\put(680,82){\makebox(0,0){ 3}}
\put(680.0,487.0){\rule[-0.200pt]{0.400pt}{4.818pt}}
\put(759.0,123.0){\rule[-0.200pt]{0.400pt}{4.818pt}}
\put(759,82){\makebox(0,0){ 3.5}}
\put(759.0,487.0){\rule[-0.200pt]{0.400pt}{4.818pt}}
\put(839.0,123.0){\rule[-0.200pt]{0.400pt}{4.818pt}}
\put(839,82){\makebox(0,0){ 4}}
\put(839.0,487.0){\rule[-0.200pt]{0.400pt}{4.818pt}}
\put(201.0,123.0){\rule[-0.200pt]{153.694pt}{0.400pt}}
\put(839.0,123.0){\rule[-0.200pt]{0.400pt}{92.506pt}}
\put(201.0,507.0){\rule[-0.200pt]{153.694pt}{0.400pt}}
\put(201.0,123.0){\rule[-0.200pt]{0.400pt}{92.506pt}}
\put(40,315){\makebox(0,0){\rotatebox{90}{matching accuracy}}}
\put(520,21){\makebox(0,0){$256\epsilon$}}
\put(520,569){\makebox(0,0){$|\mathcal S|= 10$, MSE $< 10^{-8}$}}
\put(281,205){\makebox(0,0)[r]{JT}}
\multiput(301,205)(20.756,0.000){5}{\usebox{\plotpoint}}
\put(401,205){\usebox{\plotpoint}}
\put(301.00,215.00){\usebox{\plotpoint}}
\put(301,195){\usebox{\plotpoint}}
\put(401.00,215.00){\usebox{\plotpoint}}
\put(401,195){\usebox{\plotpoint}}
\put(201,507){\usebox{\plotpoint}}
\multiput(201,507)(20.756,0.000){4}{\usebox{\plotpoint}}
\multiput(281,507)(20.652,-2.065){4}{\usebox{\plotpoint}}
\multiput(361,499)(20.729,-1.050){4}{\usebox{\plotpoint}}
\multiput(440,495)(20.741,-0.778){4}{\usebox{\plotpoint}}
\multiput(520,492)(20.730,-1.036){4}{\usebox{\plotpoint}}
\multiput(600,488)(20.652,-2.065){4}{\usebox{\plotpoint}}
\multiput(680,480)(20.729,-1.050){3}{\usebox{\plotpoint}}
\multiput(759,476)(20.677,-1.809){4}{\usebox{\plotpoint}}
\put(839,469){\usebox{\plotpoint}}
\put(201,507){\usebox{\plotpoint}}
\put(201,507){\usebox{\plotpoint}}
\put(191.00,507.00){\usebox{\plotpoint}}
\put(211,507){\usebox{\plotpoint}}
\put(191.00,507.00){\usebox{\plotpoint}}
\put(211,507){\usebox{\plotpoint}}
\put(281,507){\usebox{\plotpoint}}
\put(281,507){\usebox{\plotpoint}}
\put(271.00,507.00){\usebox{\plotpoint}}
\put(291,507){\usebox{\plotpoint}}
\put(271.00,507.00){\usebox{\plotpoint}}
\put(291,507){\usebox{\plotpoint}}
\put(361.00,494.00){\usebox{\plotpoint}}
\put(361,505){\usebox{\plotpoint}}
\put(351.00,494.00){\usebox{\plotpoint}}
\put(371,494){\usebox{\plotpoint}}
\put(351.00,505.00){\usebox{\plotpoint}}
\put(371,505){\usebox{\plotpoint}}
\put(440.00,489.00){\usebox{\plotpoint}}
\put(440,502){\usebox{\plotpoint}}
\put(430.00,489.00){\usebox{\plotpoint}}
\put(450,489){\usebox{\plotpoint}}
\put(430.00,502.00){\usebox{\plotpoint}}
\put(450,502){\usebox{\plotpoint}}
\put(520.00,484.00){\usebox{\plotpoint}}
\put(520,499){\usebox{\plotpoint}}
\put(510.00,484.00){\usebox{\plotpoint}}
\put(530,484){\usebox{\plotpoint}}
\put(510.00,499.00){\usebox{\plotpoint}}
\put(530,499){\usebox{\plotpoint}}
\put(600.00,480.00){\usebox{\plotpoint}}
\put(600,496){\usebox{\plotpoint}}
\put(590.00,480.00){\usebox{\plotpoint}}
\put(610,480){\usebox{\plotpoint}}
\put(590.00,496.00){\usebox{\plotpoint}}
\put(610,496){\usebox{\plotpoint}}
\put(680.00,471.00){\usebox{\plotpoint}}
\put(680,490){\usebox{\plotpoint}}
\put(670.00,471.00){\usebox{\plotpoint}}
\put(690,471){\usebox{\plotpoint}}
\put(670.00,490.00){\usebox{\plotpoint}}
\put(690,490){\usebox{\plotpoint}}
\put(759.00,466.00){\usebox{\plotpoint}}
\put(759,486){\usebox{\plotpoint}}
\put(749.00,466.00){\usebox{\plotpoint}}
\put(769,466){\usebox{\plotpoint}}
\put(749.00,486.00){\usebox{\plotpoint}}
\put(769,486){\usebox{\plotpoint}}
\multiput(839,456)(0.000,20.756){2}{\usebox{\plotpoint}}
\put(839,481){\usebox{\plotpoint}}
\put(829.00,456.00){\usebox{\plotpoint}}
\put(849,456){\usebox{\plotpoint}}
\put(829.00,481.00){\usebox{\plotpoint}}
\put(849,481){\usebox{\plotpoint}}
\put(201,507){\circle{12}}
\put(281,507){\circle{12}}
\put(361,499){\circle{12}}
\put(440,495){\circle{12}}
\put(520,492){\circle{12}}
\put(600,488){\circle{12}}
\put(680,480){\circle{12}}
\put(759,476){\circle{12}}
\put(839,469){\circle{12}}
\put(351,205){\circle{12}}
\sbox{\plotpoint}{\rule[-0.400pt]{0.800pt}{0.800pt}}%
\sbox{\plotpoint}{\rule[-0.200pt]{0.400pt}{0.400pt}}%
\put(281,164){\makebox(0,0)[r]{LBP}}
\sbox{\plotpoint}{\rule[-0.400pt]{0.800pt}{0.800pt}}%
\put(301.0,164.0){\rule[-0.400pt]{24.090pt}{0.800pt}}
\put(301.0,154.0){\rule[-0.400pt]{0.800pt}{4.818pt}}
\put(401.0,154.0){\rule[-0.400pt]{0.800pt}{4.818pt}}
\put(201,507){\usebox{\plotpoint}}
\put(201,503.34){\rule{16.200pt}{0.800pt}}
\multiput(201.00,505.34)(46.376,-4.000){2}{\rule{8.100pt}{0.800pt}}
\multiput(281.00,501.08)(3.930,-0.512){15}{\rule{6.018pt}{0.123pt}}
\multiput(281.00,501.34)(67.509,-11.000){2}{\rule{3.009pt}{0.800pt}}
\put(361,488.34){\rule{16.000pt}{0.800pt}}
\multiput(361.00,490.34)(45.791,-4.000){2}{\rule{8.000pt}{0.800pt}}
\put(440,484.34){\rule{16.200pt}{0.800pt}}
\multiput(440.00,486.34)(46.376,-4.000){2}{\rule{8.100pt}{0.800pt}}
\multiput(759.00,482.09)(1.779,-0.505){39}{\rule{2.983pt}{0.122pt}}
\multiput(759.00,482.34)(73.809,-23.000){2}{\rule{1.491pt}{0.800pt}}
\put(520.0,484.0){\rule[-0.400pt]{57.575pt}{0.800pt}}
\put(201,507){\usebox{\plotpoint}}
\put(191.0,507.0){\rule[-0.400pt]{4.818pt}{0.800pt}}
\put(191.0,507.0){\rule[-0.400pt]{4.818pt}{0.800pt}}
\put(281.0,499.0){\rule[-0.400pt]{0.800pt}{1.927pt}}
\put(271.0,499.0){\rule[-0.400pt]{4.818pt}{0.800pt}}
\put(271.0,507.0){\rule[-0.400pt]{4.818pt}{0.800pt}}
\put(361.0,484.0){\rule[-0.400pt]{0.800pt}{3.613pt}}
\put(351.0,484.0){\rule[-0.400pt]{4.818pt}{0.800pt}}
\put(351.0,499.0){\rule[-0.400pt]{4.818pt}{0.800pt}}
\put(440.0,478.0){\rule[-0.400pt]{0.800pt}{4.818pt}}
\put(430.0,478.0){\rule[-0.400pt]{4.818pt}{0.800pt}}
\put(430.0,498.0){\rule[-0.400pt]{4.818pt}{0.800pt}}
\put(520.0,474.0){\rule[-0.400pt]{0.800pt}{4.818pt}}
\put(510.0,474.0){\rule[-0.400pt]{4.818pt}{0.800pt}}
\put(510.0,494.0){\rule[-0.400pt]{4.818pt}{0.800pt}}
\put(600.0,474.0){\rule[-0.400pt]{0.800pt}{4.818pt}}
\put(590.0,474.0){\rule[-0.400pt]{4.818pt}{0.800pt}}
\put(590.0,494.0){\rule[-0.400pt]{4.818pt}{0.800pt}}
\put(680.0,474.0){\rule[-0.400pt]{0.800pt}{4.818pt}}
\put(670.0,474.0){\rule[-0.400pt]{4.818pt}{0.800pt}}
\put(670.0,494.0){\rule[-0.400pt]{4.818pt}{0.800pt}}
\put(759.0,474.0){\rule[-0.400pt]{0.800pt}{4.818pt}}
\put(749.0,474.0){\rule[-0.400pt]{4.818pt}{0.800pt}}
\put(749.0,494.0){\rule[-0.400pt]{4.818pt}{0.800pt}}
\put(839.0,446.0){\rule[-0.400pt]{0.800pt}{7.227pt}}
\put(829.0,446.0){\rule[-0.400pt]{4.818pt}{0.800pt}}
\put(201,507){\circle{18}}
\put(281,503){\circle{18}}
\put(361,492){\circle{18}}
\put(440,488){\circle{18}}
\put(520,484){\circle{18}}
\put(600,484){\circle{18}}
\put(680,484){\circle{18}}
\put(759,484){\circle{18}}
\put(839,461){\circle{18}}
\put(351,164){\circle{18}}
\put(829.0,476.0){\rule[-0.400pt]{4.818pt}{0.800pt}}
\sbox{\plotpoint}{\rule[-0.200pt]{0.400pt}{0.400pt}}%
\put(201.0,123.0){\rule[-0.200pt]{153.694pt}{0.400pt}}
\put(839.0,123.0){\rule[-0.200pt]{0.400pt}{92.506pt}}
\put(201.0,507.0){\rule[-0.200pt]{153.694pt}{0.400pt}}
\put(201.0,123.0){\rule[-0.200pt]{0.400pt}{92.506pt}}
\end{picture}

%% file: plot_class20.tex
\setlength{\unitlength}{0.240900pt}
\ifx\plotpoint\undefined\newsavebox{\plotpoint}\fi
\sbox{\plotpoint}{\rule[-0.200pt]{0.400pt}{0.400pt}}%
\begin{picture}(900,629)(0,0)
\sbox{\plotpoint}{\rule[-0.200pt]{0.400pt}{0.400pt}}%
\put(201.0,123.0){\rule[-0.200pt]{4.818pt}{0.400pt}}
\put(181,123){\makebox(0,0)[r]{ 0.8}}
\put(819.0,123.0){\rule[-0.200pt]{4.818pt}{0.400pt}}
\put(201.0,161.0){\rule[-0.200pt]{4.818pt}{0.400pt}}
\put(181,161){\makebox(0,0)[r]{ 0.82}}
\put(819.0,161.0){\rule[-0.200pt]{4.818pt}{0.400pt}}
\put(201.0,200.0){\rule[-0.200pt]{4.818pt}{0.400pt}}
\put(181,200){\makebox(0,0)[r]{ 0.84}}
\put(819.0,200.0){\rule[-0.200pt]{4.818pt}{0.400pt}}
\put(201.0,238.0){\rule[-0.200pt]{4.818pt}{0.400pt}}
\put(181,238){\makebox(0,0)[r]{ 0.86}}
\put(819.0,238.0){\rule[-0.200pt]{4.818pt}{0.400pt}}
\put(201.0,277.0){\rule[-0.200pt]{4.818pt}{0.400pt}}
\put(181,277){\makebox(0,0)[r]{ 0.88}}
\put(819.0,277.0){\rule[-0.200pt]{4.818pt}{0.400pt}}
\put(201.0,315.0){\rule[-0.200pt]{4.818pt}{0.400pt}}
\put(181,315){\makebox(0,0)[r]{ 0.9}}
\put(819.0,315.0){\rule[-0.200pt]{4.818pt}{0.400pt}}
\put(201.0,353.0){\rule[-0.200pt]{4.818pt}{0.400pt}}
\put(181,353){\makebox(0,0)[r]{ 0.92}}
\put(819.0,353.0){\rule[-0.200pt]{4.818pt}{0.400pt}}
\put(201.0,392.0){\rule[-0.200pt]{4.818pt}{0.400pt}}
\put(181,392){\makebox(0,0)[r]{ 0.94}}
\put(819.0,392.0){\rule[-0.200pt]{4.818pt}{0.400pt}}
\put(201.0,430.0){\rule[-0.200pt]{4.818pt}{0.400pt}}
\put(181,430){\makebox(0,0)[r]{ 0.96}}
\put(819.0,430.0){\rule[-0.200pt]{4.818pt}{0.400pt}}
\put(201.0,469.0){\rule[-0.200pt]{4.818pt}{0.400pt}}
\put(181,469){\makebox(0,0)[r]{ 0.98}}
\put(819.0,469.0){\rule[-0.200pt]{4.818pt}{0.400pt}}
\put(201.0,507.0){\rule[-0.200pt]{4.818pt}{0.400pt}}
\put(181,507){\makebox(0,0)[r]{ 1}}
\put(819.0,507.0){\rule[-0.200pt]{4.818pt}{0.400pt}}
\put(201.0,123.0){\rule[-0.200pt]{0.400pt}{4.818pt}}
\put(201,82){\makebox(0,0){ 0}}
\put(201.0,487.0){\rule[-0.200pt]{0.400pt}{4.818pt}}
\put(281.0,123.0){\rule[-0.200pt]{0.400pt}{4.818pt}}
\put(281,82){\makebox(0,0){ 0.5}}
\put(281.0,487.0){\rule[-0.200pt]{0.400pt}{4.818pt}}
\put(361.0,123.0){\rule[-0.200pt]{0.400pt}{4.818pt}}
\put(361,82){\makebox(0,0){ 1}}
\put(361.0,487.0){\rule[-0.200pt]{0.400pt}{4.818pt}}
\put(440.0,123.0){\rule[-0.200pt]{0.400pt}{4.818pt}}
\put(440,82){\makebox(0,0){ 1.5}}
\put(440.0,487.0){\rule[-0.200pt]{0.400pt}{4.818pt}}
\put(520.0,123.0){\rule[-0.200pt]{0.400pt}{4.818pt}}
\put(520,82){\makebox(0,0){ 2}}
\put(520.0,487.0){\rule[-0.200pt]{0.400pt}{4.818pt}}
\put(600.0,123.0){\rule[-0.200pt]{0.400pt}{4.818pt}}
\put(600,82){\makebox(0,0){ 2.5}}
\put(600.0,487.0){\rule[-0.200pt]{0.400pt}{4.818pt}}
\put(680.0,123.0){\rule[-0.200pt]{0.400pt}{4.818pt}}
\put(680,82){\makebox(0,0){ 3}}
\put(680.0,487.0){\rule[-0.200pt]{0.400pt}{4.818pt}}
\put(759.0,123.0){\rule[-0.200pt]{0.400pt}{4.818pt}}
\put(759,82){\makebox(0,0){ 3.5}}
\put(759.0,487.0){\rule[-0.200pt]{0.400pt}{4.818pt}}
\put(839.0,123.0){\rule[-0.200pt]{0.400pt}{4.818pt}}
\put(839,82){\makebox(0,0){ 4}}
\put(839.0,487.0){\rule[-0.200pt]{0.400pt}{4.818pt}}
\put(201.0,123.0){\rule[-0.200pt]{153.694pt}{0.400pt}}
\put(839.0,123.0){\rule[-0.200pt]{0.400pt}{92.506pt}}
\put(201.0,507.0){\rule[-0.200pt]{153.694pt}{0.400pt}}
\put(201.0,123.0){\rule[-0.200pt]{0.400pt}{92.506pt}}
\put(40,315){\makebox(0,0){\rotatebox{90}{matching accuracy}}}
\put(520,21){\makebox(0,0){$256\epsilon$}}
\put(520,569){\makebox(0,0){$|\mathcal S|= 20$, MSE $< 10^{-8}$}}
\put(281,205){\makebox(0,0)[r]{JT}}
\multiput(301,205)(20.756,0.000){5}{\usebox{\plotpoint}}
\put(401,205){\usebox{\plotpoint}}
\put(301.00,215.00){\usebox{\plotpoint}}
\put(301,195){\usebox{\plotpoint}}
\put(401.00,215.00){\usebox{\plotpoint}}
\put(401,195){\usebox{\plotpoint}}
\put(201,507){\usebox{\plotpoint}}
\multiput(201,507)(20.756,0.000){4}{\usebox{\plotpoint}}
\multiput(281,507)(20.652,-2.065){4}{\usebox{\plotpoint}}
\multiput(361,499)(20.557,-2.862){4}{\usebox{\plotpoint}}
\multiput(440,488)(20.730,-1.036){4}{\usebox{\plotpoint}}
\multiput(520,484)(20.526,-3.079){4}{\usebox{\plotpoint}}
\multiput(600,472)(20.562,-2.827){4}{\usebox{\plotpoint}}
\multiput(680,461)(20.520,-3.117){4}{\usebox{\plotpoint}}
\multiput(759,449)(19.947,-5.735){4}{\usebox{\plotpoint}}
\put(839,426){\usebox{\plotpoint}}
\put(201,507){\usebox{\plotpoint}}
\put(201,507){\usebox{\plotpoint}}
\put(191.00,507.00){\usebox{\plotpoint}}
\put(211,507){\usebox{\plotpoint}}
\put(191.00,507.00){\usebox{\plotpoint}}
\put(211,507){\usebox{\plotpoint}}
\put(281,507){\usebox{\plotpoint}}
\put(281,507){\usebox{\plotpoint}}
\put(271.00,507.00){\usebox{\plotpoint}}
\put(291,507){\usebox{\plotpoint}}
\put(271.00,507.00){\usebox{\plotpoint}}
\put(291,507){\usebox{\plotpoint}}
\put(361.00,494.00){\usebox{\plotpoint}}
\put(361,505){\usebox{\plotpoint}}
\put(351.00,494.00){\usebox{\plotpoint}}
\put(371,494){\usebox{\plotpoint}}
\put(351.00,505.00){\usebox{\plotpoint}}
\put(371,505){\usebox{\plotpoint}}
\put(440.00,480.00){\usebox{\plotpoint}}
\put(440,496){\usebox{\plotpoint}}
\put(430.00,480.00){\usebox{\plotpoint}}
\put(450,480){\usebox{\plotpoint}}
\put(430.00,496.00){\usebox{\plotpoint}}
\put(450,496){\usebox{\plotpoint}}
\put(520.00,475.00){\usebox{\plotpoint}}
\put(520,493){\usebox{\plotpoint}}
\put(510.00,475.00){\usebox{\plotpoint}}
\put(530,475){\usebox{\plotpoint}}
\put(510.00,493.00){\usebox{\plotpoint}}
\put(530,493){\usebox{\plotpoint}}
\multiput(600,462)(0.000,20.756){2}{\usebox{\plotpoint}}
\put(600,483){\usebox{\plotpoint}}
\put(590.00,462.00){\usebox{\plotpoint}}
\put(610,462){\usebox{\plotpoint}}
\put(590.00,483.00){\usebox{\plotpoint}}
\put(610,483){\usebox{\plotpoint}}
\multiput(680,449)(0.000,20.756){2}{\usebox{\plotpoint}}
\put(680,473){\usebox{\plotpoint}}
\put(670.00,449.00){\usebox{\plotpoint}}
\put(690,449){\usebox{\plotpoint}}
\put(670.00,473.00){\usebox{\plotpoint}}
\put(690,473){\usebox{\plotpoint}}
\multiput(759,437)(0.000,20.756){2}{\usebox{\plotpoint}}
\put(759,462){\usebox{\plotpoint}}
\put(749.00,437.00){\usebox{\plotpoint}}
\put(769,437){\usebox{\plotpoint}}
\put(749.00,462.00){\usebox{\plotpoint}}
\put(769,462){\usebox{\plotpoint}}
\multiput(839,411)(0.000,20.756){2}{\usebox{\plotpoint}}
\put(839,442){\usebox{\plotpoint}}
\put(829.00,411.00){\usebox{\plotpoint}}
\put(849,411){\usebox{\plotpoint}}
\put(829.00,442.00){\usebox{\plotpoint}}
\put(849,442){\usebox{\plotpoint}}
\put(201,507){\circle{12}}
\put(281,507){\circle{12}}
\put(361,499){\circle{12}}
\put(440,488){\circle{12}}
\put(520,484){\circle{12}}
\put(600,472){\circle{12}}
\put(680,461){\circle{12}}
\put(759,449){\circle{12}}
\put(839,426){\circle{12}}
\put(351,205){\circle{12}}
\sbox{\plotpoint}{\rule[-0.400pt]{0.800pt}{0.800pt}}%
\sbox{\plotpoint}{\rule[-0.200pt]{0.400pt}{0.400pt}}%
\put(281,164){\makebox(0,0)[r]{LBP}}
\sbox{\plotpoint}{\rule[-0.400pt]{0.800pt}{0.800pt}}%
\put(301.0,164.0){\rule[-0.400pt]{24.090pt}{0.800pt}}
\put(301.0,154.0){\rule[-0.400pt]{0.800pt}{4.818pt}}
\put(401.0,154.0){\rule[-0.400pt]{0.800pt}{4.818pt}}
\put(201,507){\usebox{\plotpoint}}
\multiput(201.00,505.08)(5.745,-0.520){9}{\rule{8.200pt}{0.125pt}}
\multiput(201.00,505.34)(62.980,-8.000){2}{\rule{4.100pt}{0.800pt}}
\multiput(281.00,497.08)(6.869,-0.526){7}{\rule{9.343pt}{0.127pt}}
\multiput(281.00,497.34)(60.608,-7.000){2}{\rule{4.671pt}{0.800pt}}
\put(361,488.34){\rule{16.000pt}{0.800pt}}
\multiput(361.00,490.34)(45.791,-4.000){2}{\rule{8.000pt}{0.800pt}}
\multiput(440.00,486.08)(5.745,-0.520){9}{\rule{8.200pt}{0.125pt}}
\multiput(440.00,486.34)(62.980,-8.000){2}{\rule{4.100pt}{0.800pt}}
\put(520,476.34){\rule{16.200pt}{0.800pt}}
\multiput(520.00,478.34)(46.376,-4.000){2}{\rule{8.100pt}{0.800pt}}
\multiput(600.00,474.08)(6.869,-0.526){7}{\rule{9.343pt}{0.127pt}}
\multiput(600.00,474.34)(60.608,-7.000){2}{\rule{4.671pt}{0.800pt}}
\multiput(680.00,467.08)(3.519,-0.511){17}{\rule{5.467pt}{0.123pt}}
\multiput(680.00,467.34)(67.654,-12.000){2}{\rule{2.733pt}{0.800pt}}
\multiput(759.00,455.09)(1.507,-0.504){47}{\rule{2.570pt}{0.121pt}}
\multiput(759.00,455.34)(74.665,-27.000){2}{\rule{1.285pt}{0.800pt}}
\put(201,507){\usebox{\plotpoint}}
\put(191.0,507.0){\rule[-0.400pt]{4.818pt}{0.800pt}}
\put(191.0,507.0){\rule[-0.400pt]{4.818pt}{0.800pt}}
\put(281.0,494.0){\rule[-0.400pt]{0.800pt}{2.650pt}}
\put(271.0,494.0){\rule[-0.400pt]{4.818pt}{0.800pt}}
\put(271.0,505.0){\rule[-0.400pt]{4.818pt}{0.800pt}}
\put(361.0,484.0){\rule[-0.400pt]{0.800pt}{3.613pt}}
\put(351.0,484.0){\rule[-0.400pt]{4.818pt}{0.800pt}}
\put(351.0,499.0){\rule[-0.400pt]{4.818pt}{0.800pt}}
\put(440.0,478.0){\rule[-0.400pt]{0.800pt}{4.818pt}}
\put(430.0,478.0){\rule[-0.400pt]{4.818pt}{0.800pt}}
\put(430.0,498.0){\rule[-0.400pt]{4.818pt}{0.800pt}}
\put(520.0,469.0){\rule[-0.400pt]{0.800pt}{5.300pt}}
\put(510.0,469.0){\rule[-0.400pt]{4.818pt}{0.800pt}}
\put(510.0,491.0){\rule[-0.400pt]{4.818pt}{0.800pt}}
\put(600.0,465.0){\rule[-0.400pt]{0.800pt}{5.541pt}}
\put(590.0,465.0){\rule[-0.400pt]{4.818pt}{0.800pt}}
\put(590.0,488.0){\rule[-0.400pt]{4.818pt}{0.800pt}}
\put(680.0,456.0){\rule[-0.400pt]{0.800pt}{6.022pt}}
\put(670.0,456.0){\rule[-0.400pt]{4.818pt}{0.800pt}}
\put(670.0,481.0){\rule[-0.400pt]{4.818pt}{0.800pt}}
\put(759.0,444.0){\rule[-0.400pt]{0.800pt}{6.263pt}}
\put(749.0,444.0){\rule[-0.400pt]{4.818pt}{0.800pt}}
\put(749.0,470.0){\rule[-0.400pt]{4.818pt}{0.800pt}}
\put(839.0,414.0){\rule[-0.400pt]{0.800pt}{7.709pt}}
\put(829.0,414.0){\rule[-0.400pt]{4.818pt}{0.800pt}}
\put(201,507){\circle{18}}
\put(281,499){\circle{18}}
\put(361,492){\circle{18}}
\put(440,488){\circle{18}}
\put(520,480){\circle{18}}
\put(600,476){\circle{18}}
\put(680,469){\circle{18}}
\put(759,457){\circle{18}}
\put(839,430){\circle{18}}
\put(351,164){\circle{18}}
\put(829.0,446.0){\rule[-0.400pt]{4.818pt}{0.800pt}}
\sbox{\plotpoint}{\rule[-0.200pt]{0.400pt}{0.400pt}}%
\put(201.0,123.0){\rule[-0.200pt]{153.694pt}{0.400pt}}
\put(839.0,123.0){\rule[-0.200pt]{0.400pt}{92.506pt}}
\put(201.0,507.0){\rule[-0.200pt]{153.694pt}{0.400pt}}
\put(201.0,123.0){\rule[-0.200pt]{0.400pt}{92.506pt}}
\end{picture}

%% file: plot_class30.tex
\setlength{\unitlength}{0.240900pt}
\ifx\plotpoint\undefined\newsavebox{\plotpoint}\fi
\sbox{\plotpoint}{\rule[-0.200pt]{0.400pt}{0.400pt}}%
\begin{picture}(900,629)(0,0)
\sbox{\plotpoint}{\rule[-0.200pt]{0.400pt}{0.400pt}}%
\put(201.0,123.0){\rule[-0.200pt]{4.818pt}{0.400pt}}
\put(181,123){\makebox(0,0)[r]{ 0.8}}
\put(819.0,123.0){\rule[-0.200pt]{4.818pt}{0.400pt}}
\put(201.0,161.0){\rule[-0.200pt]{4.818pt}{0.400pt}}
\put(181,161){\makebox(0,0)[r]{ 0.82}}
\put(819.0,161.0){\rule[-0.200pt]{4.818pt}{0.400pt}}
\put(201.0,200.0){\rule[-0.200pt]{4.818pt}{0.400pt}}
\put(181,200){\makebox(0,0)[r]{ 0.84}}
\put(819.0,200.0){\rule[-0.200pt]{4.818pt}{0.400pt}}
\put(201.0,238.0){\rule[-0.200pt]{4.818pt}{0.400pt}}
\put(181,238){\makebox(0,0)[r]{ 0.86}}
\put(819.0,238.0){\rule[-0.200pt]{4.818pt}{0.400pt}}
\put(201.0,277.0){\rule[-0.200pt]{4.818pt}{0.400pt}}
\put(181,277){\makebox(0,0)[r]{ 0.88}}
\put(819.0,277.0){\rule[-0.200pt]{4.818pt}{0.400pt}}
\put(201.0,315.0){\rule[-0.200pt]{4.818pt}{0.400pt}}
\put(181,315){\makebox(0,0)[r]{ 0.9}}
\put(819.0,315.0){\rule[-0.200pt]{4.818pt}{0.400pt}}
\put(201.0,353.0){\rule[-0.200pt]{4.818pt}{0.400pt}}
\put(181,353){\makebox(0,0)[r]{ 0.92}}
\put(819.0,353.0){\rule[-0.200pt]{4.818pt}{0.400pt}}
\put(201.0,392.0){\rule[-0.200pt]{4.818pt}{0.400pt}}
\put(181,392){\makebox(0,0)[r]{ 0.94}}
\put(819.0,392.0){\rule[-0.200pt]{4.818pt}{0.400pt}}
\put(201.0,430.0){\rule[-0.200pt]{4.818pt}{0.400pt}}
\put(181,430){\makebox(0,0)[r]{ 0.96}}
\put(819.0,430.0){\rule[-0.200pt]{4.818pt}{0.400pt}}
\put(201.0,469.0){\rule[-0.200pt]{4.818pt}{0.400pt}}
\put(181,469){\makebox(0,0)[r]{ 0.98}}
\put(819.0,469.0){\rule[-0.200pt]{4.818pt}{0.400pt}}
\put(201.0,507.0){\rule[-0.200pt]{4.818pt}{0.400pt}}
\put(181,507){\makebox(0,0)[r]{ 1}}
\put(819.0,507.0){\rule[-0.200pt]{4.818pt}{0.400pt}}
\put(201.0,123.0){\rule[-0.200pt]{0.400pt}{4.818pt}}
\put(201,82){\makebox(0,0){ 0}}
\put(201.0,487.0){\rule[-0.200pt]{0.400pt}{4.818pt}}
\put(281.0,123.0){\rule[-0.200pt]{0.400pt}{4.818pt}}
\put(281,82){\makebox(0,0){ 0.5}}
\put(281.0,487.0){\rule[-0.200pt]{0.400pt}{4.818pt}}
\put(361.0,123.0){\rule[-0.200pt]{0.400pt}{4.818pt}}
\put(361,82){\makebox(0,0){ 1}}
\put(361.0,487.0){\rule[-0.200pt]{0.400pt}{4.818pt}}
\put(440.0,123.0){\rule[-0.200pt]{0.400pt}{4.818pt}}
\put(440,82){\makebox(0,0){ 1.5}}
\put(440.0,487.0){\rule[-0.200pt]{0.400pt}{4.818pt}}
\put(520.0,123.0){\rule[-0.200pt]{0.400pt}{4.818pt}}
\put(520,82){\makebox(0,0){ 2}}
\put(520.0,487.0){\rule[-0.200pt]{0.400pt}{4.818pt}}
\put(600.0,123.0){\rule[-0.200pt]{0.400pt}{4.818pt}}
\put(600,82){\makebox(0,0){ 2.5}}
\put(600.0,487.0){\rule[-0.200pt]{0.400pt}{4.818pt}}
\put(680.0,123.0){\rule[-0.200pt]{0.400pt}{4.818pt}}
\put(680,82){\makebox(0,0){ 3}}
\put(680.0,487.0){\rule[-0.200pt]{0.400pt}{4.818pt}}
\put(759.0,123.0){\rule[-0.200pt]{0.400pt}{4.818pt}}
\put(759,82){\makebox(0,0){ 3.5}}
\put(759.0,487.0){\rule[-0.200pt]{0.400pt}{4.818pt}}
\put(839.0,123.0){\rule[-0.200pt]{0.400pt}{4.818pt}}
\put(839,82){\makebox(0,0){ 4}}
\put(839.0,487.0){\rule[-0.200pt]{0.400pt}{4.818pt}}
\put(201.0,123.0){\rule[-0.200pt]{153.694pt}{0.400pt}}
\put(839.0,123.0){\rule[-0.200pt]{0.400pt}{92.506pt}}
\put(201.0,507.0){\rule[-0.200pt]{153.694pt}{0.400pt}}
\put(201.0,123.0){\rule[-0.200pt]{0.400pt}{92.506pt}}
\put(40,315){\makebox(0,0){\rotatebox{90}{matching accuracy}}}
\put(520,21){\makebox(0,0){$256\epsilon$}}
\put(520,569){\makebox(0,0){$|\mathcal S|= 30$, MSE $< 10^{-9}$}}
\put(281,205){\makebox(0,0)[r]{JT}}
\multiput(301,205)(20.756,0.000){5}{\usebox{\plotpoint}}
\put(401,205){\usebox{\plotpoint}}
\put(301.00,215.00){\usebox{\plotpoint}}
\put(301,195){\usebox{\plotpoint}}
\put(401.00,215.00){\usebox{\plotpoint}}
\put(401,195){\usebox{\plotpoint}}
\put(201,507){\usebox{\plotpoint}}
\multiput(201,507)(20.756,0.000){4}{\usebox{\plotpoint}}
\multiput(281,507)(20.526,-3.079){4}{\usebox{\plotpoint}}
\multiput(361,495)(19.928,-5.802){4}{\usebox{\plotpoint}}
\multiput(440,472)(20.400,-3.825){4}{\usebox{\plotpoint}}
\multiput(520,457)(19.666,-6.637){4}{\usebox{\plotpoint}}
\multiput(600,430)(20.562,-2.827){4}{\usebox{\plotpoint}}
\multiput(680,419)(19.321,-7.582){4}{\usebox{\plotpoint}}
\multiput(759,388)(19.353,-7.499){4}{\usebox{\plotpoint}}
\put(839,357){\usebox{\plotpoint}}
\put(201,507){\usebox{\plotpoint}}
\put(201,507){\usebox{\plotpoint}}
\put(191.00,507.00){\usebox{\plotpoint}}
\put(211,507){\usebox{\plotpoint}}
\put(191.00,507.00){\usebox{\plotpoint}}
\put(211,507){\usebox{\plotpoint}}
\put(281,507){\usebox{\plotpoint}}
\put(281,507){\usebox{\plotpoint}}
\put(271.00,507.00){\usebox{\plotpoint}}
\put(291,507){\usebox{\plotpoint}}
\put(271.00,507.00){\usebox{\plotpoint}}
\put(291,507){\usebox{\plotpoint}}
\put(361.00,489.00){\usebox{\plotpoint}}
\put(361,502){\usebox{\plotpoint}}
\put(351.00,489.00){\usebox{\plotpoint}}
\put(371,489){\usebox{\plotpoint}}
\put(351.00,502.00){\usebox{\plotpoint}}
\put(371,502){\usebox{\plotpoint}}
\multiput(440,462)(0.000,20.756){2}{\usebox{\plotpoint}}
\put(440,483){\usebox{\plotpoint}}
\put(430.00,462.00){\usebox{\plotpoint}}
\put(450,462){\usebox{\plotpoint}}
\put(430.00,483.00){\usebox{\plotpoint}}
\put(450,483){\usebox{\plotpoint}}
\multiput(520,444)(0.000,20.756){2}{\usebox{\plotpoint}}
\put(520,470){\usebox{\plotpoint}}
\put(510.00,444.00){\usebox{\plotpoint}}
\put(530,444){\usebox{\plotpoint}}
\put(510.00,470.00){\usebox{\plotpoint}}
\put(530,470){\usebox{\plotpoint}}
\multiput(600,415)(0.000,20.756){2}{\usebox{\plotpoint}}
\put(600,446){\usebox{\plotpoint}}
\put(590.00,415.00){\usebox{\plotpoint}}
\put(610,415){\usebox{\plotpoint}}
\put(590.00,446.00){\usebox{\plotpoint}}
\put(610,446){\usebox{\plotpoint}}
\multiput(680,402)(0.000,20.756){2}{\usebox{\plotpoint}}
\put(680,435){\usebox{\plotpoint}}
\put(670.00,402.00){\usebox{\plotpoint}}
\put(690,402){\usebox{\plotpoint}}
\put(670.00,435.00){\usebox{\plotpoint}}
\put(690,435){\usebox{\plotpoint}}
\multiput(759,369)(0.000,20.756){2}{\usebox{\plotpoint}}
\put(759,407){\usebox{\plotpoint}}
\put(749.00,369.00){\usebox{\plotpoint}}
\put(769,369){\usebox{\plotpoint}}
\put(749.00,407.00){\usebox{\plotpoint}}
\put(769,407){\usebox{\plotpoint}}
\multiput(839,337)(0.000,20.756){2}{\usebox{\plotpoint}}
\put(839,377){\usebox{\plotpoint}}
\put(829.00,337.00){\usebox{\plotpoint}}
\put(849,337){\usebox{\plotpoint}}
\put(829.00,377.00){\usebox{\plotpoint}}
\put(849,377){\usebox{\plotpoint}}
\put(201,507){\circle{12}}
\put(281,507){\circle{12}}
\put(361,495){\circle{12}}
\put(440,472){\circle{12}}
\put(520,457){\circle{12}}
\put(600,430){\circle{12}}
\put(680,419){\circle{12}}
\put(759,388){\circle{12}}
\put(839,357){\circle{12}}
\put(351,205){\circle{12}}
\sbox{\plotpoint}{\rule[-0.400pt]{0.800pt}{0.800pt}}%
\sbox{\plotpoint}{\rule[-0.200pt]{0.400pt}{0.400pt}}%
\put(281,164){\makebox(0,0)[r]{LBP}}
\sbox{\plotpoint}{\rule[-0.400pt]{0.800pt}{0.800pt}}%
\put(301.0,164.0){\rule[-0.400pt]{24.090pt}{0.800pt}}
\put(301.0,154.0){\rule[-0.400pt]{0.800pt}{4.818pt}}
\put(401.0,154.0){\rule[-0.400pt]{0.800pt}{4.818pt}}
\put(201,507){\usebox{\plotpoint}}
\put(201,503.34){\rule{16.200pt}{0.800pt}}
\multiput(201.00,505.34)(46.376,-4.000){2}{\rule{8.100pt}{0.800pt}}
\multiput(281.00,501.08)(3.930,-0.512){15}{\rule{6.018pt}{0.123pt}}
\multiput(281.00,501.34)(67.509,-11.000){2}{\rule{3.009pt}{0.800pt}}
\multiput(361.00,490.08)(3.519,-0.511){17}{\rule{5.467pt}{0.123pt}}
\multiput(361.00,490.34)(67.654,-12.000){2}{\rule{2.733pt}{0.800pt}}
\multiput(440.00,478.09)(2.794,-0.508){23}{\rule{4.467pt}{0.122pt}}
\multiput(440.00,478.34)(70.729,-15.000){2}{\rule{2.233pt}{0.800pt}}
\multiput(520.00,463.09)(2.607,-0.507){25}{\rule{4.200pt}{0.122pt}}
\multiput(520.00,463.34)(71.283,-16.000){2}{\rule{2.100pt}{0.800pt}}
\multiput(600.00,447.09)(0.958,-0.502){77}{\rule{1.724pt}{0.121pt}}
\multiput(600.00,447.34)(76.422,-42.000){2}{\rule{0.862pt}{0.800pt}}
\multiput(680.00,405.09)(1.291,-0.503){55}{\rule{2.239pt}{0.121pt}}
\multiput(680.00,405.34)(74.353,-31.000){2}{\rule{1.119pt}{0.800pt}}
\multiput(759.00,374.09)(1.061,-0.503){69}{\rule{1.884pt}{0.121pt}}
\multiput(759.00,374.34)(76.089,-38.000){2}{\rule{0.942pt}{0.800pt}}
\put(201,507){\usebox{\plotpoint}}
\put(191.0,507.0){\rule[-0.400pt]{4.818pt}{0.800pt}}
\put(191.0,507.0){\rule[-0.400pt]{4.818pt}{0.800pt}}
\put(281.0,499.0){\rule[-0.400pt]{0.800pt}{1.927pt}}
\put(271.0,499.0){\rule[-0.400pt]{4.818pt}{0.800pt}}
\put(271.0,507.0){\rule[-0.400pt]{4.818pt}{0.800pt}}
\put(361.0,484.0){\rule[-0.400pt]{0.800pt}{3.613pt}}
\put(351.0,484.0){\rule[-0.400pt]{4.818pt}{0.800pt}}
\put(351.0,499.0){\rule[-0.400pt]{4.818pt}{0.800pt}}
\put(440.0,469.0){\rule[-0.400pt]{0.800pt}{5.300pt}}
\put(430.0,469.0){\rule[-0.400pt]{4.818pt}{0.800pt}}
\put(430.0,491.0){\rule[-0.400pt]{4.818pt}{0.800pt}}
\put(520.0,449.0){\rule[-0.400pt]{0.800pt}{7.468pt}}
\put(510.0,449.0){\rule[-0.400pt]{4.818pt}{0.800pt}}
\put(510.0,480.0){\rule[-0.400pt]{4.818pt}{0.800pt}}
\put(600.0,431.0){\rule[-0.400pt]{0.800pt}{8.913pt}}
\put(590.0,431.0){\rule[-0.400pt]{4.818pt}{0.800pt}}
\put(590.0,468.0){\rule[-0.400pt]{4.818pt}{0.800pt}}
\put(680.0,373.0){\rule[-0.400pt]{0.800pt}{16.622pt}}
\put(670.0,373.0){\rule[-0.400pt]{4.818pt}{0.800pt}}
\put(670.0,442.0){\rule[-0.400pt]{4.818pt}{0.800pt}}
\put(759.0,341.0){\rule[-0.400pt]{0.800pt}{17.104pt}}
\put(749.0,341.0){\rule[-0.400pt]{4.818pt}{0.800pt}}
\put(749.0,412.0){\rule[-0.400pt]{4.818pt}{0.800pt}}
\put(839.0,302.0){\rule[-0.400pt]{0.800pt}{17.345pt}}
\put(829.0,302.0){\rule[-0.400pt]{4.818pt}{0.800pt}}
\put(201,507){\circle{18}}
\put(281,503){\circle{18}}
\put(361,492){\circle{18}}
\put(440,480){\circle{18}}
\put(520,465){\circle{18}}
\put(600,449){\circle{18}}
\put(680,407){\circle{18}}
\put(759,376){\circle{18}}
\put(839,338){\circle{18}}
\put(351,164){\circle{18}}
\put(829.0,374.0){\rule[-0.400pt]{4.818pt}{0.800pt}}
\sbox{\plotpoint}{\rule[-0.200pt]{0.400pt}{0.400pt}}%
\put(201.0,123.0){\rule[-0.200pt]{153.694pt}{0.400pt}}
\put(839.0,123.0){\rule[-0.200pt]{0.400pt}{92.506pt}}
\put(201.0,507.0){\rule[-0.200pt]{153.694pt}{0.400pt}}
\put(201.0,123.0){\rule[-0.200pt]{0.400pt}{92.506pt}}
\end{picture}

%% file: plot_class40.tex
\setlength{\unitlength}{0.240900pt}
\ifx\plotpoint\undefined\newsavebox{\plotpoint}\fi
\sbox{\plotpoint}{\rule[-0.200pt]{0.400pt}{0.400pt}}%
\begin{picture}(900,629)(0,0)
\sbox{\plotpoint}{\rule[-0.200pt]{0.400pt}{0.400pt}}%
\put(201.0,123.0){\rule[-0.200pt]{4.818pt}{0.400pt}}
\put(181,123){\makebox(0,0)[r]{ 0.8}}
\put(819.0,123.0){\rule[-0.200pt]{4.818pt}{0.400pt}}
\put(201.0,161.0){\rule[-0.200pt]{4.818pt}{0.400pt}}
\put(181,161){\makebox(0,0)[r]{ 0.82}}
\put(819.0,161.0){\rule[-0.200pt]{4.818pt}{0.400pt}}
\put(201.0,200.0){\rule[-0.200pt]{4.818pt}{0.400pt}}
\put(181,200){\makebox(0,0)[r]{ 0.84}}
\put(819.0,200.0){\rule[-0.200pt]{4.818pt}{0.400pt}}
\put(201.0,238.0){\rule[-0.200pt]{4.818pt}{0.400pt}}
\put(181,238){\makebox(0,0)[r]{ 0.86}}
\put(819.0,238.0){\rule[-0.200pt]{4.818pt}{0.400pt}}
\put(201.0,277.0){\rule[-0.200pt]{4.818pt}{0.400pt}}
\put(181,277){\makebox(0,0)[r]{ 0.88}}
\put(819.0,277.0){\rule[-0.200pt]{4.818pt}{0.400pt}}
\put(201.0,315.0){\rule[-0.200pt]{4.818pt}{0.400pt}}
\put(181,315){\makebox(0,0)[r]{ 0.9}}
\put(819.0,315.0){\rule[-0.200pt]{4.818pt}{0.400pt}}
\put(201.0,353.0){\rule[-0.200pt]{4.818pt}{0.400pt}}
\put(181,353){\makebox(0,0)[r]{ 0.92}}
\put(819.0,353.0){\rule[-0.200pt]{4.818pt}{0.400pt}}
\put(201.0,392.0){\rule[-0.200pt]{4.818pt}{0.400pt}}
\put(181,392){\makebox(0,0)[r]{ 0.94}}
\put(819.0,392.0){\rule[-0.200pt]{4.818pt}{0.400pt}}
\put(201.0,430.0){\rule[-0.200pt]{4.818pt}{0.400pt}}
\put(181,430){\makebox(0,0)[r]{ 0.96}}
\put(819.0,430.0){\rule[-0.200pt]{4.818pt}{0.400pt}}
\put(201.0,469.0){\rule[-0.200pt]{4.818pt}{0.400pt}}
\put(181,469){\makebox(0,0)[r]{ 0.98}}
\put(819.0,469.0){\rule[-0.200pt]{4.818pt}{0.400pt}}
\put(201.0,507.0){\rule[-0.200pt]{4.818pt}{0.400pt}}
\put(181,507){\makebox(0,0)[r]{ 1}}
\put(819.0,507.0){\rule[-0.200pt]{4.818pt}{0.400pt}}
\put(201.0,123.0){\rule[-0.200pt]{0.400pt}{4.818pt}}
\put(201,82){\makebox(0,0){ 0}}
\put(201.0,487.0){\rule[-0.200pt]{0.400pt}{4.818pt}}
\put(281.0,123.0){\rule[-0.200pt]{0.400pt}{4.818pt}}
\put(281,82){\makebox(0,0){ 0.5}}
\put(281.0,487.0){\rule[-0.200pt]{0.400pt}{4.818pt}}
\put(361.0,123.0){\rule[-0.200pt]{0.400pt}{4.818pt}}
\put(361,82){\makebox(0,0){ 1}}
\put(361.0,487.0){\rule[-0.200pt]{0.400pt}{4.818pt}}
\put(440.0,123.0){\rule[-0.200pt]{0.400pt}{4.818pt}}
\put(440,82){\makebox(0,0){ 1.5}}
\put(440.0,487.0){\rule[-0.200pt]{0.400pt}{4.818pt}}
\put(520.0,123.0){\rule[-0.200pt]{0.400pt}{4.818pt}}
\put(520,82){\makebox(0,0){ 2}}
\put(520.0,487.0){\rule[-0.200pt]{0.400pt}{4.818pt}}
\put(600.0,123.0){\rule[-0.200pt]{0.400pt}{4.818pt}}
\put(600,82){\makebox(0,0){ 2.5}}
\put(600.0,487.0){\rule[-0.200pt]{0.400pt}{4.818pt}}
\put(680.0,123.0){\rule[-0.200pt]{0.400pt}{4.818pt}}
\put(680,82){\makebox(0,0){ 3}}
\put(680.0,487.0){\rule[-0.200pt]{0.400pt}{4.818pt}}
\put(759.0,123.0){\rule[-0.200pt]{0.400pt}{4.818pt}}
\put(759,82){\makebox(0,0){ 3.5}}
\put(759.0,487.0){\rule[-0.200pt]{0.400pt}{4.818pt}}
\put(839.0,123.0){\rule[-0.200pt]{0.400pt}{4.818pt}}
\put(839,82){\makebox(0,0){ 4}}
\put(839.0,487.0){\rule[-0.200pt]{0.400pt}{4.818pt}}
\put(201.0,123.0){\rule[-0.200pt]{153.694pt}{0.400pt}}
\put(839.0,123.0){\rule[-0.200pt]{0.400pt}{92.506pt}}
\put(201.0,507.0){\rule[-0.200pt]{153.694pt}{0.400pt}}
\put(201.0,123.0){\rule[-0.200pt]{0.400pt}{92.506pt}}
\put(40,315){\makebox(0,0){\rotatebox{90}{matching accuracy}}}
\put(520,21){\makebox(0,0){$256\epsilon$}}
\put(520,569){\makebox(0,0){$|\mathcal S|= 40$, MSE $< 10^{-9}$}}
\put(281,205){\makebox(0,0)[r]{JT}}
\multiput(301,205)(20.756,0.000){5}{\usebox{\plotpoint}}
\put(401,205){\usebox{\plotpoint}}
\put(301.00,215.00){\usebox{\plotpoint}}
\put(301,195){\usebox{\plotpoint}}
\put(401.00,215.00){\usebox{\plotpoint}}
\put(401,195){\usebox{\plotpoint}}
\put(201,507){\usebox{\plotpoint}}
\multiput(201,507)(20.756,0.000){4}{\usebox{\plotpoint}}
\multiput(281,507)(20.526,-3.079){4}{\usebox{\plotpoint}}
\multiput(361,495)(19.928,-5.802){4}{\usebox{\plotpoint}}
\multiput(440,472)(19.947,-5.735){4}{\usebox{\plotpoint}}
\multiput(520,449)(17.601,-11.000){5}{\usebox{\plotpoint}}
\multiput(600,399)(19.739,-6.415){4}{\usebox{\plotpoint}}
\multiput(680,373)(18.230,-9.923){4}{\usebox{\plotpoint}}
\multiput(759,330)(19.666,-6.637){4}{\usebox{\plotpoint}}
\put(839,303){\usebox{\plotpoint}}
\put(201,507){\usebox{\plotpoint}}
\put(201,507){\usebox{\plotpoint}}
\put(191.00,507.00){\usebox{\plotpoint}}
\put(211,507){\usebox{\plotpoint}}
\put(191.00,507.00){\usebox{\plotpoint}}
\put(211,507){\usebox{\plotpoint}}
\put(281,507){\usebox{\plotpoint}}
\put(281,507){\usebox{\plotpoint}}
\put(271.00,507.00){\usebox{\plotpoint}}
\put(291,507){\usebox{\plotpoint}}
\put(271.00,507.00){\usebox{\plotpoint}}
\put(291,507){\usebox{\plotpoint}}
\put(361.00,489.00){\usebox{\plotpoint}}
\put(361,502){\usebox{\plotpoint}}
\put(351.00,489.00){\usebox{\plotpoint}}
\put(371,489){\usebox{\plotpoint}}
\put(351.00,502.00){\usebox{\plotpoint}}
\put(371,502){\usebox{\plotpoint}}
\multiput(440,462)(0.000,20.756){2}{\usebox{\plotpoint}}
\put(440,483){\usebox{\plotpoint}}
\put(430.00,462.00){\usebox{\plotpoint}}
\put(450,462){\usebox{\plotpoint}}
\put(430.00,483.00){\usebox{\plotpoint}}
\put(450,483){\usebox{\plotpoint}}
\multiput(520,436)(0.000,20.756){2}{\usebox{\plotpoint}}
\put(520,463){\usebox{\plotpoint}}
\put(510.00,436.00){\usebox{\plotpoint}}
\put(530,436){\usebox{\plotpoint}}
\put(510.00,463.00){\usebox{\plotpoint}}
\put(530,463){\usebox{\plotpoint}}
\multiput(600,383)(0.000,20.756){2}{\usebox{\plotpoint}}
\put(600,416){\usebox{\plotpoint}}
\put(590.00,383.00){\usebox{\plotpoint}}
\put(610,383){\usebox{\plotpoint}}
\put(590.00,416.00){\usebox{\plotpoint}}
\put(610,416){\usebox{\plotpoint}}
\multiput(680,354)(0.000,20.756){2}{\usebox{\plotpoint}}
\put(680,392){\usebox{\plotpoint}}
\put(670.00,354.00){\usebox{\plotpoint}}
\put(690,354){\usebox{\plotpoint}}
\put(670.00,392.00){\usebox{\plotpoint}}
\put(690,392){\usebox{\plotpoint}}
\multiput(759,309)(0.000,20.756){3}{\usebox{\plotpoint}}
\put(759,352){\usebox{\plotpoint}}
\put(749.00,309.00){\usebox{\plotpoint}}
\put(769,309){\usebox{\plotpoint}}
\put(749.00,352.00){\usebox{\plotpoint}}
\put(769,352){\usebox{\plotpoint}}
\multiput(839,283)(0.000,20.756){2}{\usebox{\plotpoint}}
\put(839,324){\usebox{\plotpoint}}
\put(829.00,283.00){\usebox{\plotpoint}}
\put(849,283){\usebox{\plotpoint}}
\put(829.00,324.00){\usebox{\plotpoint}}
\put(849,324){\usebox{\plotpoint}}
\put(201,507){\circle{12}}
\put(281,507){\circle{12}}
\put(361,495){\circle{12}}
\put(440,472){\circle{12}}
\put(520,449){\circle{12}}
\put(600,399){\circle{12}}
\put(680,373){\circle{12}}
\put(759,330){\circle{12}}
\put(839,303){\circle{12}}
\put(351,205){\circle{12}}
\sbox{\plotpoint}{\rule[-0.400pt]{0.800pt}{0.800pt}}%
\sbox{\plotpoint}{\rule[-0.200pt]{0.400pt}{0.400pt}}%
\put(281,164){\makebox(0,0)[r]{LBP}}
\sbox{\plotpoint}{\rule[-0.400pt]{0.800pt}{0.800pt}}%
\put(301.0,164.0){\rule[-0.400pt]{24.090pt}{0.800pt}}
\put(301.0,154.0){\rule[-0.400pt]{0.800pt}{4.818pt}}
\put(401.0,154.0){\rule[-0.400pt]{0.800pt}{4.818pt}}
\put(201,507){\usebox{\plotpoint}}
\multiput(201.00,505.08)(5.745,-0.520){9}{\rule{8.200pt}{0.125pt}}
\multiput(201.00,505.34)(62.980,-8.000){2}{\rule{4.100pt}{0.800pt}}
\multiput(281.00,497.08)(6.869,-0.526){7}{\rule{9.343pt}{0.127pt}}
\multiput(281.00,497.34)(60.608,-7.000){2}{\rule{4.671pt}{0.800pt}}
\multiput(361.00,490.09)(2.574,-0.507){25}{\rule{4.150pt}{0.122pt}}
\multiput(361.00,490.34)(70.386,-16.000){2}{\rule{2.075pt}{0.800pt}}
\multiput(440.00,474.09)(1.507,-0.504){47}{\rule{2.570pt}{0.121pt}}
\multiput(440.00,474.34)(74.665,-27.000){2}{\rule{1.285pt}{0.800pt}}
\multiput(520.00,447.09)(1.567,-0.504){45}{\rule{2.662pt}{0.121pt}}
\multiput(520.00,447.34)(74.476,-26.000){2}{\rule{1.331pt}{0.800pt}}
\multiput(600.00,421.09)(2.059,-0.505){33}{\rule{3.400pt}{0.122pt}}
\multiput(600.00,421.34)(72.943,-20.000){2}{\rule{1.700pt}{0.800pt}}
\multiput(680.00,401.09)(0.747,-0.502){99}{\rule{1.392pt}{0.121pt}}
\multiput(680.00,401.34)(76.110,-53.000){2}{\rule{0.696pt}{0.800pt}}
\multiput(759.00,348.09)(0.571,-0.501){133}{\rule{1.114pt}{0.121pt}}
\multiput(759.00,348.34)(77.687,-70.000){2}{\rule{0.557pt}{0.800pt}}
\put(201,507){\usebox{\plotpoint}}
\put(191.0,507.0){\rule[-0.400pt]{4.818pt}{0.800pt}}
\put(191.0,507.0){\rule[-0.400pt]{4.818pt}{0.800pt}}
\put(281.0,494.0){\rule[-0.400pt]{0.800pt}{2.650pt}}
\put(271.0,494.0){\rule[-0.400pt]{4.818pt}{0.800pt}}
\put(271.0,505.0){\rule[-0.400pt]{4.818pt}{0.800pt}}
\put(361.0,484.0){\rule[-0.400pt]{0.800pt}{3.613pt}}
\put(351.0,484.0){\rule[-0.400pt]{4.818pt}{0.800pt}}
\put(351.0,499.0){\rule[-0.400pt]{4.818pt}{0.800pt}}
\put(440.0,465.0){\rule[-0.400pt]{0.800pt}{5.541pt}}
\put(430.0,465.0){\rule[-0.400pt]{4.818pt}{0.800pt}}
\put(430.0,488.0){\rule[-0.400pt]{4.818pt}{0.800pt}}
\put(520.0,433.0){\rule[-0.400pt]{0.800pt}{7.950pt}}
\put(510.0,433.0){\rule[-0.400pt]{4.818pt}{0.800pt}}
\put(510.0,466.0){\rule[-0.400pt]{4.818pt}{0.800pt}}
\put(600.0,402.0){\rule[-0.400pt]{0.800pt}{9.877pt}}
\put(590.0,402.0){\rule[-0.400pt]{4.818pt}{0.800pt}}
\put(590.0,443.0){\rule[-0.400pt]{4.818pt}{0.800pt}}
\put(680.0,382.0){\rule[-0.400pt]{0.800pt}{10.118pt}}
\put(670.0,382.0){\rule[-0.400pt]{4.818pt}{0.800pt}}
\put(670.0,424.0){\rule[-0.400pt]{4.818pt}{0.800pt}}
\put(759.0,324.0){\rule[-0.400pt]{0.800pt}{12.286pt}}
\put(749.0,324.0){\rule[-0.400pt]{4.818pt}{0.800pt}}
\put(749.0,375.0){\rule[-0.400pt]{4.818pt}{0.800pt}}
\put(839.0,246.0){\rule[-0.400pt]{0.800pt}{16.622pt}}
\put(829.0,246.0){\rule[-0.400pt]{4.818pt}{0.800pt}}
\put(201,507){\circle{18}}
\put(281,499){\circle{18}}
\put(361,492){\circle{18}}
\put(440,476){\circle{18}}
\put(520,449){\circle{18}}
\put(600,423){\circle{18}}
\put(680,403){\circle{18}}
\put(759,350){\circle{18}}
\put(839,280){\circle{18}}
\put(351,164){\circle{18}}
\put(829.0,315.0){\rule[-0.400pt]{4.818pt}{0.800pt}}
\sbox{\plotpoint}{\rule[-0.200pt]{0.400pt}{0.400pt}}%
\put(201.0,123.0){\rule[-0.200pt]{153.694pt}{0.400pt}}
\put(839.0,123.0){\rule[-0.200pt]{0.400pt}{92.506pt}}
\put(201.0,507.0){\rule[-0.200pt]{153.694pt}{0.400pt}}
\put(201.0,123.0){\rule[-0.200pt]{0.400pt}{92.506pt}}
\end{picture}

%% file: plot_time10.tex
\setlength{\unitlength}{0.240900pt}
\ifx\plotpoint\undefined\newsavebox{\plotpoint}\fi
\sbox{\plotpoint}{\rule[-0.200pt]{0.400pt}{0.400pt}}%
\begin{picture}(900,629)(0,0)
\sbox{\plotpoint}{\rule[-0.200pt]{0.400pt}{0.400pt}}%
\put(181.0,123.0){\rule[-0.200pt]{4.818pt}{0.400pt}}
\put(161,123){\makebox(0,0)[r]{ 0}}
\put(819.0,123.0){\rule[-0.200pt]{4.818pt}{0.400pt}}
\put(181.0,200.0){\rule[-0.200pt]{4.818pt}{0.400pt}}
\put(161,200){\makebox(0,0)[r]{ 0.1}}
\put(819.0,200.0){\rule[-0.200pt]{4.818pt}{0.400pt}}
\put(181.0,277.0){\rule[-0.200pt]{4.818pt}{0.400pt}}
\put(161,277){\makebox(0,0)[r]{ 0.2}}
\put(819.0,277.0){\rule[-0.200pt]{4.818pt}{0.400pt}}
\put(181.0,353.0){\rule[-0.200pt]{4.818pt}{0.400pt}}
\put(161,353){\makebox(0,0)[r]{ 0.3}}
\put(819.0,353.0){\rule[-0.200pt]{4.818pt}{0.400pt}}
\put(181.0,430.0){\rule[-0.200pt]{4.818pt}{0.400pt}}
\put(161,430){\makebox(0,0)[r]{ 0.4}}
\put(819.0,430.0){\rule[-0.200pt]{4.818pt}{0.400pt}}
\put(181.0,507.0){\rule[-0.200pt]{4.818pt}{0.400pt}}
\put(161,507){\makebox(0,0)[r]{ 0.5}}
\put(819.0,507.0){\rule[-0.200pt]{4.818pt}{0.400pt}}
\put(181.0,123.0){\rule[-0.200pt]{0.400pt}{4.818pt}}
\put(181,82){\makebox(0,0){ 0}}
\put(181.0,487.0){\rule[-0.200pt]{0.400pt}{4.818pt}}
\put(263.0,123.0){\rule[-0.200pt]{0.400pt}{4.818pt}}
\put(263,82){\makebox(0,0){ 0.5}}
\put(263.0,487.0){\rule[-0.200pt]{0.400pt}{4.818pt}}
\put(346.0,123.0){\rule[-0.200pt]{0.400pt}{4.818pt}}
\put(346,82){\makebox(0,0){ 1}}
\put(346.0,487.0){\rule[-0.200pt]{0.400pt}{4.818pt}}
\put(428.0,123.0){\rule[-0.200pt]{0.400pt}{4.818pt}}
\put(428,82){\makebox(0,0){ 1.5}}
\put(428.0,487.0){\rule[-0.200pt]{0.400pt}{4.818pt}}
\put(510.0,123.0){\rule[-0.200pt]{0.400pt}{4.818pt}}
\put(510,82){\makebox(0,0){ 2}}
\put(510.0,487.0){\rule[-0.200pt]{0.400pt}{4.818pt}}
\put(592.0,123.0){\rule[-0.200pt]{0.400pt}{4.818pt}}
\put(592,82){\makebox(0,0){ 2.5}}
\put(592.0,487.0){\rule[-0.200pt]{0.400pt}{4.818pt}}
\put(675.0,123.0){\rule[-0.200pt]{0.400pt}{4.818pt}}
\put(675,82){\makebox(0,0){ 3}}
\put(675.0,487.0){\rule[-0.200pt]{0.400pt}{4.818pt}}
\put(757.0,123.0){\rule[-0.200pt]{0.400pt}{4.818pt}}
\put(757,82){\makebox(0,0){ 3.5}}
\put(757.0,487.0){\rule[-0.200pt]{0.400pt}{4.818pt}}
\put(839.0,123.0){\rule[-0.200pt]{0.400pt}{4.818pt}}
\put(839,82){\makebox(0,0){ 4}}
\put(839.0,487.0){\rule[-0.200pt]{0.400pt}{4.818pt}}
\put(181.0,123.0){\rule[-0.200pt]{158.512pt}{0.400pt}}
\put(839.0,123.0){\rule[-0.200pt]{0.400pt}{92.506pt}}
\put(181.0,507.0){\rule[-0.200pt]{158.512pt}{0.400pt}}
\put(181.0,123.0){\rule[-0.200pt]{0.400pt}{92.506pt}}
\put(40,315){\makebox(0,0){\rotatebox{90}{time (seconds)}}}
\put(510,21){\makebox(0,0){$256\epsilon$}}
\put(510,569){\makebox(0,0){$|\mathcal S|= 10$ (JT time = 3 seconds)}}
\sbox{\plotpoint}{\rule[-0.400pt]{0.800pt}{0.800pt}}%
\sbox{\plotpoint}{\rule[-0.200pt]{0.400pt}{0.400pt}}%
\put(501,205){\makebox(0,0)[r]{MSE $< 10^{-8}$}}
\sbox{\plotpoint}{\rule[-0.400pt]{0.800pt}{0.800pt}}%
\put(521.0,205.0){\rule[-0.400pt]{24.090pt}{0.800pt}}
\put(521.0,195.0){\rule[-0.400pt]{0.800pt}{4.818pt}}
\put(621.0,195.0){\rule[-0.400pt]{0.800pt}{4.818pt}}
\put(181,380){\usebox{\plotpoint}}
\multiput(181.00,381.41)(0.708,0.502){109}{\rule{1.331pt}{0.121pt}}
\multiput(181.00,378.34)(79.237,58.000){2}{\rule{0.666pt}{0.800pt}}
\put(510,436.84){\rule{19.754pt}{0.800pt}}
\multiput(510.00,436.34)(41.000,1.000){2}{\rule{9.877pt}{0.800pt}}
\put(263.0,438.0){\rule[-0.400pt]{59.502pt}{0.800pt}}
\put(757,437.84){\rule{19.754pt}{0.800pt}}
\multiput(757.00,437.34)(41.000,1.000){2}{\rule{9.877pt}{0.800pt}}
\put(592.0,439.0){\rule[-0.400pt]{39.748pt}{0.800pt}}
\put(181.0,378.0){\rule[-0.400pt]{0.800pt}{0.964pt}}
\put(171.0,378.0){\rule[-0.400pt]{4.818pt}{0.800pt}}
\put(171.0,382.0){\rule[-0.400pt]{4.818pt}{0.800pt}}
\put(263.0,435.0){\rule[-0.400pt]{0.800pt}{1.445pt}}
\put(253.0,435.0){\rule[-0.400pt]{4.818pt}{0.800pt}}
\put(253.0,441.0){\rule[-0.400pt]{4.818pt}{0.800pt}}
\put(346.0,435.0){\rule[-0.400pt]{0.800pt}{1.445pt}}
\put(336.0,435.0){\rule[-0.400pt]{4.818pt}{0.800pt}}
\put(336.0,441.0){\rule[-0.400pt]{4.818pt}{0.800pt}}
\put(428.0,436.0){\rule[-0.400pt]{0.800pt}{1.204pt}}
\put(418.0,436.0){\rule[-0.400pt]{4.818pt}{0.800pt}}
\put(418.0,441.0){\rule[-0.400pt]{4.818pt}{0.800pt}}
\put(510.0,435.0){\rule[-0.400pt]{0.800pt}{1.204pt}}
\put(500.0,435.0){\rule[-0.400pt]{4.818pt}{0.800pt}}
\put(500.0,440.0){\rule[-0.400pt]{4.818pt}{0.800pt}}
\put(592.0,436.0){\rule[-0.400pt]{0.800pt}{1.204pt}}
\put(582.0,436.0){\rule[-0.400pt]{4.818pt}{0.800pt}}
\put(582.0,441.0){\rule[-0.400pt]{4.818pt}{0.800pt}}
\put(675.0,436.0){\rule[-0.400pt]{0.800pt}{1.445pt}}
\put(665.0,436.0){\rule[-0.400pt]{4.818pt}{0.800pt}}
\put(665.0,442.0){\rule[-0.400pt]{4.818pt}{0.800pt}}
\put(757.0,436.0){\rule[-0.400pt]{0.800pt}{1.445pt}}
\put(747.0,436.0){\rule[-0.400pt]{4.818pt}{0.800pt}}
\put(747.0,442.0){\rule[-0.400pt]{4.818pt}{0.800pt}}
\put(839.0,437.0){\rule[-0.400pt]{0.800pt}{1.445pt}}
\put(829.0,437.0){\rule[-0.400pt]{4.818pt}{0.800pt}}
\put(181,380){\circle{18}}
\put(263,438){\circle{18}}
\put(346,438){\circle{18}}
\put(428,438){\circle{18}}
\put(510,438){\circle{18}}
\put(592,439){\circle{18}}
\put(675,439){\circle{18}}
\put(757,439){\circle{18}}
\put(839,440){\circle{18}}
\put(571,205){\circle{18}}
\put(829.0,443.0){\rule[-0.400pt]{4.818pt}{0.800pt}}
\sbox{\plotpoint}{\rule[-0.200pt]{0.400pt}{0.400pt}}%
\put(501,164){\makebox(0,0)[r]{MSE $< 10^{-4}$}}
\put(521.0,164.0){\rule[-0.200pt]{24.090pt}{0.400pt}}
\put(521.0,154.0){\rule[-0.200pt]{0.400pt}{4.818pt}}
\put(621.0,154.0){\rule[-0.200pt]{0.400pt}{4.818pt}}
\put(181,363){\usebox{\plotpoint}}
\multiput(181.00,363.58)(1.591,0.497){49}{\rule{1.362pt}{0.120pt}}
\multiput(181.00,362.17)(79.174,26.000){2}{\rule{0.681pt}{0.400pt}}
\put(263,387.67){\rule{19.995pt}{0.400pt}}
\multiput(263.00,388.17)(41.500,-1.000){2}{\rule{9.997pt}{0.400pt}}
\put(346,387.67){\rule{19.754pt}{0.400pt}}
\multiput(346.00,387.17)(41.000,1.000){2}{\rule{9.877pt}{0.400pt}}
\put(428.0,389.0){\rule[-0.200pt]{99.010pt}{0.400pt}}
\put(181.0,362.0){\rule[-0.200pt]{0.400pt}{0.723pt}}
\put(171.0,362.0){\rule[-0.200pt]{4.818pt}{0.400pt}}
\put(171.0,365.0){\rule[-0.200pt]{4.818pt}{0.400pt}}
\put(263.0,388.0){\rule[-0.200pt]{0.400pt}{0.723pt}}
\put(253.0,388.0){\rule[-0.200pt]{4.818pt}{0.400pt}}
\put(253.0,391.0){\rule[-0.200pt]{4.818pt}{0.400pt}}
\put(346.0,387.0){\rule[-0.200pt]{0.400pt}{0.723pt}}
\put(336.0,387.0){\rule[-0.200pt]{4.818pt}{0.400pt}}
\put(336.0,390.0){\rule[-0.200pt]{4.818pt}{0.400pt}}
\put(428.0,387.0){\rule[-0.200pt]{0.400pt}{0.723pt}}
\put(418.0,387.0){\rule[-0.200pt]{4.818pt}{0.400pt}}
\put(418.0,390.0){\rule[-0.200pt]{4.818pt}{0.400pt}}
\put(510.0,387.0){\rule[-0.200pt]{0.400pt}{0.723pt}}
\put(500.0,387.0){\rule[-0.200pt]{4.818pt}{0.400pt}}
\put(500.0,390.0){\rule[-0.200pt]{4.818pt}{0.400pt}}
\put(592.0,388.0){\rule[-0.200pt]{0.400pt}{0.723pt}}
\put(582.0,388.0){\rule[-0.200pt]{4.818pt}{0.400pt}}
\put(582.0,391.0){\rule[-0.200pt]{4.818pt}{0.400pt}}
\put(675.0,388.0){\rule[-0.200pt]{0.400pt}{0.723pt}}
\put(665.0,388.0){\rule[-0.200pt]{4.818pt}{0.400pt}}
\put(665.0,391.0){\rule[-0.200pt]{4.818pt}{0.400pt}}
\put(757.0,387.0){\rule[-0.200pt]{0.400pt}{0.723pt}}
\put(747.0,387.0){\rule[-0.200pt]{4.818pt}{0.400pt}}
\put(747.0,390.0){\rule[-0.200pt]{4.818pt}{0.400pt}}
\put(839.0,387.0){\rule[-0.200pt]{0.400pt}{0.723pt}}
\put(829.0,387.0){\rule[-0.200pt]{4.818pt}{0.400pt}}
\put(181,363){\circle{12}}
\put(263,389){\circle{12}}
\put(346,388){\circle{12}}
\put(428,389){\circle{12}}
\put(510,389){\circle{12}}
\put(592,389){\circle{12}}
\put(675,389){\circle{12}}
\put(757,389){\circle{12}}
\put(839,389){\circle{12}}
\put(571,164){\circle{12}}
\put(829.0,390.0){\rule[-0.200pt]{4.818pt}{0.400pt}}
\put(181.0,123.0){\rule[-0.200pt]{158.512pt}{0.400pt}}
\put(839.0,123.0){\rule[-0.200pt]{0.400pt}{92.506pt}}
\put(181.0,507.0){\rule[-0.200pt]{158.512pt}{0.400pt}}
\put(181.0,123.0){\rule[-0.200pt]{0.400pt}{92.506pt}}
\end{picture}

%% file: plot_time20.tex
\setlength{\unitlength}{0.240900pt}
\ifx\plotpoint\undefined\newsavebox{\plotpoint}\fi
\sbox{\plotpoint}{\rule[-0.200pt]{0.400pt}{0.400pt}}%
\begin{picture}(900,629)(0,0)
\sbox{\plotpoint}{\rule[-0.200pt]{0.400pt}{0.400pt}}%
\put(181.0,123.0){\rule[-0.200pt]{4.818pt}{0.400pt}}
\put(161,123){\makebox(0,0)[r]{ 2}}
\put(819.0,123.0){\rule[-0.200pt]{4.818pt}{0.400pt}}
\put(181.0,200.0){\rule[-0.200pt]{4.818pt}{0.400pt}}
\put(161,200){\makebox(0,0)[r]{ 2.1}}
\put(819.0,200.0){\rule[-0.200pt]{4.818pt}{0.400pt}}
\put(181.0,277.0){\rule[-0.200pt]{4.818pt}{0.400pt}}
\put(161,277){\makebox(0,0)[r]{ 2.2}}
\put(819.0,277.0){\rule[-0.200pt]{4.818pt}{0.400pt}}
\put(181.0,353.0){\rule[-0.200pt]{4.818pt}{0.400pt}}
\put(161,353){\makebox(0,0)[r]{ 2.3}}
\put(819.0,353.0){\rule[-0.200pt]{4.818pt}{0.400pt}}
\put(181.0,430.0){\rule[-0.200pt]{4.818pt}{0.400pt}}
\put(161,430){\makebox(0,0)[r]{ 2.4}}
\put(819.0,430.0){\rule[-0.200pt]{4.818pt}{0.400pt}}
\put(181.0,507.0){\rule[-0.200pt]{4.818pt}{0.400pt}}
\put(161,507){\makebox(0,0)[r]{ 2.5}}
\put(819.0,507.0){\rule[-0.200pt]{4.818pt}{0.400pt}}
\put(181.0,123.0){\rule[-0.200pt]{0.400pt}{4.818pt}}
\put(181,82){\makebox(0,0){ 0}}
\put(181.0,487.0){\rule[-0.200pt]{0.400pt}{4.818pt}}
\put(263.0,123.0){\rule[-0.200pt]{0.400pt}{4.818pt}}
\put(263,82){\makebox(0,0){ 0.5}}
\put(263.0,487.0){\rule[-0.200pt]{0.400pt}{4.818pt}}
\put(346.0,123.0){\rule[-0.200pt]{0.400pt}{4.818pt}}
\put(346,82){\makebox(0,0){ 1}}
\put(346.0,487.0){\rule[-0.200pt]{0.400pt}{4.818pt}}
\put(428.0,123.0){\rule[-0.200pt]{0.400pt}{4.818pt}}
\put(428,82){\makebox(0,0){ 1.5}}
\put(428.0,487.0){\rule[-0.200pt]{0.400pt}{4.818pt}}
\put(510.0,123.0){\rule[-0.200pt]{0.400pt}{4.818pt}}
\put(510,82){\makebox(0,0){ 2}}
\put(510.0,487.0){\rule[-0.200pt]{0.400pt}{4.818pt}}
\put(592.0,123.0){\rule[-0.200pt]{0.400pt}{4.818pt}}
\put(592,82){\makebox(0,0){ 2.5}}
\put(592.0,487.0){\rule[-0.200pt]{0.400pt}{4.818pt}}
\put(675.0,123.0){\rule[-0.200pt]{0.400pt}{4.818pt}}
\put(675,82){\makebox(0,0){ 3}}
\put(675.0,487.0){\rule[-0.200pt]{0.400pt}{4.818pt}}
\put(757.0,123.0){\rule[-0.200pt]{0.400pt}{4.818pt}}
\put(757,82){\makebox(0,0){ 3.5}}
\put(757.0,487.0){\rule[-0.200pt]{0.400pt}{4.818pt}}
\put(839.0,123.0){\rule[-0.200pt]{0.400pt}{4.818pt}}
\put(839,82){\makebox(0,0){ 4}}
\put(839.0,487.0){\rule[-0.200pt]{0.400pt}{4.818pt}}
\put(181.0,123.0){\rule[-0.200pt]{158.512pt}{0.400pt}}
\put(839.0,123.0){\rule[-0.200pt]{0.400pt}{92.506pt}}
\put(181.0,507.0){\rule[-0.200pt]{158.512pt}{0.400pt}}
\put(181.0,123.0){\rule[-0.200pt]{0.400pt}{92.506pt}}
\put(40,315){\makebox(0,0){\rotatebox{90}{time (seconds)}}}
\put(510,21){\makebox(0,0){$256\epsilon$}}
\put(510,569){\makebox(0,0){$|\mathcal S|= 20$ (JT time = 44 seconds)}}
\sbox{\plotpoint}{\rule[-0.400pt]{0.800pt}{0.800pt}}%
\sbox{\plotpoint}{\rule[-0.200pt]{0.400pt}{0.400pt}}%
\put(501,205){\makebox(0,0)[r]{MSE $< 10^{-8}$}}
\sbox{\plotpoint}{\rule[-0.400pt]{0.800pt}{0.800pt}}%
\put(521.0,205.0){\rule[-0.400pt]{24.090pt}{0.800pt}}
\put(521.0,195.0){\rule[-0.400pt]{0.800pt}{4.818pt}}
\put(621.0,195.0){\rule[-0.400pt]{0.800pt}{4.818pt}}
\put(181,247){\usebox{\plotpoint}}
\multiput(182.41,247.00)(0.501,1.028){157}{\rule{0.121pt}{1.839pt}}
\multiput(179.34,247.00)(82.000,164.183){2}{\rule{0.800pt}{0.920pt}}
\put(263,412.34){\rule{19.995pt}{0.800pt}}
\multiput(263.00,413.34)(41.500,-2.000){2}{\rule{9.997pt}{0.800pt}}
\put(346,412.84){\rule{19.754pt}{0.800pt}}
\multiput(346.00,411.34)(41.000,3.000){2}{\rule{9.877pt}{0.800pt}}
\put(428,413.84){\rule{19.754pt}{0.800pt}}
\multiput(428.00,414.34)(41.000,-1.000){2}{\rule{9.877pt}{0.800pt}}
\put(510,413.84){\rule{19.754pt}{0.800pt}}
\multiput(510.00,413.34)(41.000,1.000){2}{\rule{9.877pt}{0.800pt}}
\multiput(592.00,417.40)(7.132,0.526){7}{\rule{9.686pt}{0.127pt}}
\multiput(592.00,414.34)(62.897,7.000){2}{\rule{4.843pt}{0.800pt}}
\multiput(675.00,421.08)(7.045,-0.526){7}{\rule{9.571pt}{0.127pt}}
\multiput(675.00,421.34)(62.134,-7.000){2}{\rule{4.786pt}{0.800pt}}
\put(757.0,416.0){\rule[-0.400pt]{19.754pt}{0.800pt}}
\put(181.0,233.0){\rule[-0.400pt]{0.800pt}{6.986pt}}
\put(171.0,233.0){\rule[-0.400pt]{4.818pt}{0.800pt}}
\put(171.0,262.0){\rule[-0.400pt]{4.818pt}{0.800pt}}
\put(263.0,400.0){\rule[-0.400pt]{0.800pt}{7.227pt}}
\put(253.0,400.0){\rule[-0.400pt]{4.818pt}{0.800pt}}
\put(253.0,430.0){\rule[-0.400pt]{4.818pt}{0.800pt}}
\put(346.0,398.0){\rule[-0.400pt]{0.800pt}{7.227pt}}
\put(336.0,398.0){\rule[-0.400pt]{4.818pt}{0.800pt}}
\put(336.0,428.0){\rule[-0.400pt]{4.818pt}{0.800pt}}
\put(428.0,401.0){\rule[-0.400pt]{0.800pt}{7.227pt}}
\put(418.0,401.0){\rule[-0.400pt]{4.818pt}{0.800pt}}
\put(418.0,431.0){\rule[-0.400pt]{4.818pt}{0.800pt}}
\put(510.0,400.0){\rule[-0.400pt]{0.800pt}{7.227pt}}
\put(500.0,400.0){\rule[-0.400pt]{4.818pt}{0.800pt}}
\put(500.0,430.0){\rule[-0.400pt]{4.818pt}{0.800pt}}
\put(592.0,401.0){\rule[-0.400pt]{0.800pt}{7.227pt}}
\put(582.0,401.0){\rule[-0.400pt]{4.818pt}{0.800pt}}
\put(582.0,431.0){\rule[-0.400pt]{4.818pt}{0.800pt}}
\put(675.0,408.0){\rule[-0.400pt]{0.800pt}{7.227pt}}
\put(665.0,408.0){\rule[-0.400pt]{4.818pt}{0.800pt}}
\put(665.0,438.0){\rule[-0.400pt]{4.818pt}{0.800pt}}
\put(757.0,401.0){\rule[-0.400pt]{0.800pt}{7.227pt}}
\put(747.0,401.0){\rule[-0.400pt]{4.818pt}{0.800pt}}
\put(747.0,431.0){\rule[-0.400pt]{4.818pt}{0.800pt}}
\put(839.0,401.0){\rule[-0.400pt]{0.800pt}{7.227pt}}
\put(829.0,401.0){\rule[-0.400pt]{4.818pt}{0.800pt}}
\put(181,247){\circle{18}}
\put(263,415){\circle{18}}
\put(346,413){\circle{18}}
\put(428,416){\circle{18}}
\put(510,415){\circle{18}}
\put(592,416){\circle{18}}
\put(675,423){\circle{18}}
\put(757,416){\circle{18}}
\put(839,416){\circle{18}}
\put(571,205){\circle{18}}
\put(829.0,431.0){\rule[-0.400pt]{4.818pt}{0.800pt}}
\sbox{\plotpoint}{\rule[-0.200pt]{0.400pt}{0.400pt}}%
\put(501,164){\makebox(0,0)[r]{MSE $< 10^{-4}$}}
\put(521.0,164.0){\rule[-0.200pt]{24.090pt}{0.400pt}}
\put(521.0,154.0){\rule[-0.200pt]{0.400pt}{4.818pt}}
\put(621.0,154.0){\rule[-0.200pt]{0.400pt}{4.818pt}}
\put(181,219){\usebox{\plotpoint}}
\multiput(181.58,219.00)(0.499,0.983){161}{\rule{0.120pt}{0.885pt}}
\multiput(180.17,219.00)(82.000,159.162){2}{\rule{0.400pt}{0.443pt}}
\put(263,378.17){\rule{16.700pt}{0.400pt}}
\multiput(263.00,379.17)(48.338,-2.000){2}{\rule{8.350pt}{0.400pt}}
\put(346,378.17){\rule{16.500pt}{0.400pt}}
\multiput(346.00,377.17)(47.753,2.000){2}{\rule{8.250pt}{0.400pt}}
\multiput(592.00,380.59)(7.452,0.482){9}{\rule{5.633pt}{0.116pt}}
\multiput(592.00,379.17)(71.308,6.000){2}{\rule{2.817pt}{0.400pt}}
\multiput(675.00,384.93)(6.213,-0.485){11}{\rule{4.786pt}{0.117pt}}
\multiput(675.00,385.17)(72.067,-7.000){2}{\rule{2.393pt}{0.400pt}}
\put(757,379.17){\rule{16.500pt}{0.400pt}}
\multiput(757.00,378.17)(47.753,2.000){2}{\rule{8.250pt}{0.400pt}}
\put(428.0,380.0){\rule[-0.200pt]{39.508pt}{0.400pt}}
\put(181.0,205.0){\rule[-0.200pt]{0.400pt}{6.986pt}}
\put(171.0,205.0){\rule[-0.200pt]{4.818pt}{0.400pt}}
\put(171.0,234.0){\rule[-0.200pt]{4.818pt}{0.400pt}}
\put(263.0,366.0){\rule[-0.200pt]{0.400pt}{6.986pt}}
\put(253.0,366.0){\rule[-0.200pt]{4.818pt}{0.400pt}}
\put(253.0,395.0){\rule[-0.200pt]{4.818pt}{0.400pt}}
\put(346.0,363.0){\rule[-0.200pt]{0.400pt}{7.227pt}}
\put(336.0,363.0){\rule[-0.200pt]{4.818pt}{0.400pt}}
\put(336.0,393.0){\rule[-0.200pt]{4.818pt}{0.400pt}}
\put(428.0,365.0){\rule[-0.200pt]{0.400pt}{6.986pt}}
\put(418.0,365.0){\rule[-0.200pt]{4.818pt}{0.400pt}}
\put(418.0,394.0){\rule[-0.200pt]{4.818pt}{0.400pt}}
\put(510.0,366.0){\rule[-0.200pt]{0.400pt}{6.986pt}}
\put(500.0,366.0){\rule[-0.200pt]{4.818pt}{0.400pt}}
\put(500.0,395.0){\rule[-0.200pt]{4.818pt}{0.400pt}}
\put(592.0,365.0){\rule[-0.200pt]{0.400pt}{6.986pt}}
\put(582.0,365.0){\rule[-0.200pt]{4.818pt}{0.400pt}}
\put(582.0,394.0){\rule[-0.200pt]{4.818pt}{0.400pt}}
\put(675.0,371.0){\rule[-0.200pt]{0.400pt}{7.227pt}}
\put(665.0,371.0){\rule[-0.200pt]{4.818pt}{0.400pt}}
\put(665.0,401.0){\rule[-0.200pt]{4.818pt}{0.400pt}}
\put(757.0,364.0){\rule[-0.200pt]{0.400pt}{7.227pt}}
\put(747.0,364.0){\rule[-0.200pt]{4.818pt}{0.400pt}}
\put(747.0,394.0){\rule[-0.200pt]{4.818pt}{0.400pt}}
\put(839.0,366.0){\rule[-0.200pt]{0.400pt}{6.986pt}}
\put(829.0,366.0){\rule[-0.200pt]{4.818pt}{0.400pt}}
\put(181,219){\circle{12}}
\put(263,380){\circle{12}}
\put(346,378){\circle{12}}
\put(428,380){\circle{12}}
\put(510,380){\circle{12}}
\put(592,380){\circle{12}}
\put(675,386){\circle{12}}
\put(757,379){\circle{12}}
\put(839,381){\circle{12}}
\put(571,164){\circle{12}}
\put(829.0,395.0){\rule[-0.200pt]{4.818pt}{0.400pt}}
\put(181.0,123.0){\rule[-0.200pt]{158.512pt}{0.400pt}}
\put(839.0,123.0){\rule[-0.200pt]{0.400pt}{92.506pt}}
\put(181.0,507.0){\rule[-0.200pt]{158.512pt}{0.400pt}}
\put(181.0,123.0){\rule[-0.200pt]{0.400pt}{92.506pt}}
\end{picture}

%% file: plot_time30.tex
\setlength{\unitlength}{0.240900pt}
\ifx\plotpoint\undefined\newsavebox{\plotpoint}\fi
\sbox{\plotpoint}{\rule[-0.200pt]{0.400pt}{0.400pt}}%
\begin{picture}(900,629)(0,0)
\sbox{\plotpoint}{\rule[-0.200pt]{0.400pt}{0.400pt}}%
\put(181.0,123.0){\rule[-0.200pt]{4.818pt}{0.400pt}}
\put(161,123){\makebox(0,0)[r]{ 7}}
\put(819.0,123.0){\rule[-0.200pt]{4.818pt}{0.400pt}}
\put(181.0,174.0){\rule[-0.200pt]{4.818pt}{0.400pt}}
\put(161,174){\makebox(0,0)[r]{ 7.2}}
\put(819.0,174.0){\rule[-0.200pt]{4.818pt}{0.400pt}}
\put(181.0,225.0){\rule[-0.200pt]{4.818pt}{0.400pt}}
\put(161,225){\makebox(0,0)[r]{ 7.4}}
\put(819.0,225.0){\rule[-0.200pt]{4.818pt}{0.400pt}}
\put(181.0,277.0){\rule[-0.200pt]{4.818pt}{0.400pt}}
\put(161,277){\makebox(0,0)[r]{ 7.6}}
\put(819.0,277.0){\rule[-0.200pt]{4.818pt}{0.400pt}}
\put(181.0,328.0){\rule[-0.200pt]{4.818pt}{0.400pt}}
\put(161,328){\makebox(0,0)[r]{ 7.8}}
\put(819.0,328.0){\rule[-0.200pt]{4.818pt}{0.400pt}}
\put(181.0,379.0){\rule[-0.200pt]{4.818pt}{0.400pt}}
\put(161,379){\makebox(0,0)[r]{ 8}}
\put(819.0,379.0){\rule[-0.200pt]{4.818pt}{0.400pt}}
\put(181.0,430.0){\rule[-0.200pt]{4.818pt}{0.400pt}}
\put(161,430){\makebox(0,0)[r]{ 8.2}}
\put(819.0,430.0){\rule[-0.200pt]{4.818pt}{0.400pt}}
\put(181.0,481.0){\rule[-0.200pt]{4.818pt}{0.400pt}}
\put(161,481){\makebox(0,0)[r]{ 8.4}}
\put(819.0,481.0){\rule[-0.200pt]{4.818pt}{0.400pt}}
\put(181.0,123.0){\rule[-0.200pt]{0.400pt}{4.818pt}}
\put(181,82){\makebox(0,0){ 0}}
\put(181.0,487.0){\rule[-0.200pt]{0.400pt}{4.818pt}}
\put(263.0,123.0){\rule[-0.200pt]{0.400pt}{4.818pt}}
\put(263,82){\makebox(0,0){ 0.5}}
\put(263.0,487.0){\rule[-0.200pt]{0.400pt}{4.818pt}}
\put(346.0,123.0){\rule[-0.200pt]{0.400pt}{4.818pt}}
\put(346,82){\makebox(0,0){ 1}}
\put(346.0,487.0){\rule[-0.200pt]{0.400pt}{4.818pt}}
\put(428.0,123.0){\rule[-0.200pt]{0.400pt}{4.818pt}}
\put(428,82){\makebox(0,0){ 1.5}}
\put(428.0,487.0){\rule[-0.200pt]{0.400pt}{4.818pt}}
\put(510.0,123.0){\rule[-0.200pt]{0.400pt}{4.818pt}}
\put(510,82){\makebox(0,0){ 2}}
\put(510.0,487.0){\rule[-0.200pt]{0.400pt}{4.818pt}}
\put(592.0,123.0){\rule[-0.200pt]{0.400pt}{4.818pt}}
\put(592,82){\makebox(0,0){ 2.5}}
\put(592.0,487.0){\rule[-0.200pt]{0.400pt}{4.818pt}}
\put(675.0,123.0){\rule[-0.200pt]{0.400pt}{4.818pt}}
\put(675,82){\makebox(0,0){ 3}}
\put(675.0,487.0){\rule[-0.200pt]{0.400pt}{4.818pt}}
\put(757.0,123.0){\rule[-0.200pt]{0.400pt}{4.818pt}}
\put(757,82){\makebox(0,0){ 3.5}}
\put(757.0,487.0){\rule[-0.200pt]{0.400pt}{4.818pt}}
\put(839.0,123.0){\rule[-0.200pt]{0.400pt}{4.818pt}}
\put(839,82){\makebox(0,0){ 4}}
\put(839.0,487.0){\rule[-0.200pt]{0.400pt}{4.818pt}}
\put(181.0,123.0){\rule[-0.200pt]{158.512pt}{0.400pt}}
\put(839.0,123.0){\rule[-0.200pt]{0.400pt}{92.506pt}}
\put(181.0,507.0){\rule[-0.200pt]{158.512pt}{0.400pt}}
\put(181.0,123.0){\rule[-0.200pt]{0.400pt}{92.506pt}}
\put(40,315){\makebox(0,0){\rotatebox{90}{time (seconds)}}}
\put(510,21){\makebox(0,0){$256\epsilon$}}
\put(510,569){\makebox(0,0){$|\mathcal S|= 30$ (JT time = 250 seconds)}}
\sbox{\plotpoint}{\rule[-0.400pt]{0.800pt}{0.800pt}}%
\sbox{\plotpoint}{\rule[-0.200pt]{0.400pt}{0.400pt}}%
\put(501,205){\makebox(0,0)[r]{MSE $< 10^{-9}$}}
\sbox{\plotpoint}{\rule[-0.400pt]{0.800pt}{0.800pt}}%
\put(521.0,205.0){\rule[-0.400pt]{24.090pt}{0.800pt}}
\put(521.0,195.0){\rule[-0.400pt]{0.800pt}{4.818pt}}
\put(621.0,195.0){\rule[-0.400pt]{0.800pt}{4.818pt}}
\put(181,226){\usebox{\plotpoint}}
\multiput(182.41,226.00)(0.501,1.040){157}{\rule{0.121pt}{1.859pt}}
\multiput(179.34,226.00)(82.000,166.143){2}{\rule{0.800pt}{0.929pt}}
\put(263,393.34){\rule{19.995pt}{0.800pt}}
\multiput(263.00,394.34)(41.500,-2.000){2}{\rule{9.997pt}{0.800pt}}
\put(346,393.34){\rule{19.754pt}{0.800pt}}
\multiput(346.00,392.34)(41.000,2.000){2}{\rule{9.877pt}{0.800pt}}
\multiput(428.00,397.38)(13.354,0.560){3}{\rule{13.320pt}{0.135pt}}
\multiput(428.00,394.34)(54.354,5.000){2}{\rule{6.660pt}{0.800pt}}
\put(510,399.84){\rule{19.754pt}{0.800pt}}
\multiput(510.00,399.34)(41.000,1.000){2}{\rule{9.877pt}{0.800pt}}
\multiput(592.00,400.08)(7.132,-0.526){7}{\rule{9.686pt}{0.127pt}}
\multiput(592.00,400.34)(62.897,-7.000){2}{\rule{4.843pt}{0.800pt}}
\multiput(757.00,396.38)(13.354,0.560){3}{\rule{13.320pt}{0.135pt}}
\multiput(757.00,393.34)(54.354,5.000){2}{\rule{6.660pt}{0.800pt}}
\put(675.0,395.0){\rule[-0.400pt]{19.754pt}{0.800pt}}
\put(181.0,206.0){\rule[-0.400pt]{0.800pt}{9.395pt}}
\put(171.0,206.0){\rule[-0.400pt]{4.818pt}{0.800pt}}
\put(171.0,245.0){\rule[-0.400pt]{4.818pt}{0.800pt}}
\put(263.0,377.0){\rule[-0.400pt]{0.800pt}{9.395pt}}
\put(253.0,377.0){\rule[-0.400pt]{4.818pt}{0.800pt}}
\put(253.0,416.0){\rule[-0.400pt]{4.818pt}{0.800pt}}
\put(346.0,375.0){\rule[-0.400pt]{0.800pt}{9.154pt}}
\put(336.0,375.0){\rule[-0.400pt]{4.818pt}{0.800pt}}
\put(336.0,413.0){\rule[-0.400pt]{4.818pt}{0.800pt}}
\put(428.0,377.0){\rule[-0.400pt]{0.800pt}{9.395pt}}
\put(418.0,377.0){\rule[-0.400pt]{4.818pt}{0.800pt}}
\put(418.0,416.0){\rule[-0.400pt]{4.818pt}{0.800pt}}
\put(510.0,381.0){\rule[-0.400pt]{0.800pt}{9.395pt}}
\put(500.0,381.0){\rule[-0.400pt]{4.818pt}{0.800pt}}
\put(500.0,420.0){\rule[-0.400pt]{4.818pt}{0.800pt}}
\put(592.0,383.0){\rule[-0.400pt]{0.800pt}{9.395pt}}
\put(582.0,383.0){\rule[-0.400pt]{4.818pt}{0.800pt}}
\put(582.0,422.0){\rule[-0.400pt]{4.818pt}{0.800pt}}
\put(675.0,376.0){\rule[-0.400pt]{0.800pt}{9.154pt}}
\put(665.0,376.0){\rule[-0.400pt]{4.818pt}{0.800pt}}
\put(665.0,414.0){\rule[-0.400pt]{4.818pt}{0.800pt}}
\put(757.0,376.0){\rule[-0.400pt]{0.800pt}{8.913pt}}
\put(747.0,376.0){\rule[-0.400pt]{4.818pt}{0.800pt}}
\put(747.0,413.0){\rule[-0.400pt]{4.818pt}{0.800pt}}
\put(839.0,381.0){\rule[-0.400pt]{0.800pt}{9.395pt}}
\put(829.0,381.0){\rule[-0.400pt]{4.818pt}{0.800pt}}
\put(181,226){\circle{18}}
\put(263,396){\circle{18}}
\put(346,394){\circle{18}}
\put(428,396){\circle{18}}
\put(510,401){\circle{18}}
\put(592,402){\circle{18}}
\put(675,395){\circle{18}}
\put(757,395){\circle{18}}
\put(839,400){\circle{18}}
\put(571,205){\circle{18}}
\put(829.0,420.0){\rule[-0.400pt]{4.818pt}{0.800pt}}
\sbox{\plotpoint}{\rule[-0.200pt]{0.400pt}{0.400pt}}%
\put(501,164){\makebox(0,0)[r]{MSE $< 10^{-5}$}}
\put(521.0,164.0){\rule[-0.200pt]{24.090pt}{0.400pt}}
\put(521.0,154.0){\rule[-0.200pt]{0.400pt}{4.818pt}}
\put(621.0,154.0){\rule[-0.200pt]{0.400pt}{4.818pt}}
\put(181,211){\usebox{\plotpoint}}
\multiput(181.58,211.00)(0.499,1.026){161}{\rule{0.120pt}{0.920pt}}
\multiput(180.17,211.00)(82.000,166.092){2}{\rule{0.400pt}{0.460pt}}
\put(263,377.17){\rule{16.700pt}{0.400pt}}
\multiput(263.00,378.17)(48.338,-2.000){2}{\rule{8.350pt}{0.400pt}}
\put(346,377.17){\rule{16.500pt}{0.400pt}}
\multiput(346.00,376.17)(47.753,2.000){2}{\rule{8.250pt}{0.400pt}}
\multiput(428.00,379.60)(11.886,0.468){5}{\rule{8.300pt}{0.113pt}}
\multiput(428.00,378.17)(64.773,4.000){2}{\rule{4.150pt}{0.400pt}}
\put(510,383.17){\rule{16.500pt}{0.400pt}}
\multiput(510.00,382.17)(47.753,2.000){2}{\rule{8.250pt}{0.400pt}}
\multiput(592.00,383.93)(6.290,-0.485){11}{\rule{4.843pt}{0.117pt}}
\multiput(592.00,384.17)(72.948,-7.000){2}{\rule{2.421pt}{0.400pt}}
\put(675,376.67){\rule{19.754pt}{0.400pt}}
\multiput(675.00,377.17)(41.000,-1.000){2}{\rule{9.877pt}{0.400pt}}
\multiput(757.00,377.59)(7.362,0.482){9}{\rule{5.567pt}{0.116pt}}
\multiput(757.00,376.17)(70.446,6.000){2}{\rule{2.783pt}{0.400pt}}
\put(181.0,192.0){\rule[-0.200pt]{0.400pt}{9.395pt}}
\put(171.0,192.0){\rule[-0.200pt]{4.818pt}{0.400pt}}
\put(171.0,231.0){\rule[-0.200pt]{4.818pt}{0.400pt}}
\put(263.0,360.0){\rule[-0.200pt]{0.400pt}{9.395pt}}
\put(253.0,360.0){\rule[-0.200pt]{4.818pt}{0.400pt}}
\put(253.0,399.0){\rule[-0.200pt]{4.818pt}{0.400pt}}
\put(346.0,358.0){\rule[-0.200pt]{0.400pt}{9.154pt}}
\put(336.0,358.0){\rule[-0.200pt]{4.818pt}{0.400pt}}
\put(336.0,396.0){\rule[-0.200pt]{4.818pt}{0.400pt}}
\put(428.0,360.0){\rule[-0.200pt]{0.400pt}{9.154pt}}
\put(418.0,360.0){\rule[-0.200pt]{4.818pt}{0.400pt}}
\put(418.0,398.0){\rule[-0.200pt]{4.818pt}{0.400pt}}
\put(510.0,364.0){\rule[-0.200pt]{0.400pt}{9.154pt}}
\put(500.0,364.0){\rule[-0.200pt]{4.818pt}{0.400pt}}
\put(500.0,402.0){\rule[-0.200pt]{4.818pt}{0.400pt}}
\put(592.0,366.0){\rule[-0.200pt]{0.400pt}{9.395pt}}
\put(582.0,366.0){\rule[-0.200pt]{4.818pt}{0.400pt}}
\put(582.0,405.0){\rule[-0.200pt]{4.818pt}{0.400pt}}
\put(675.0,359.0){\rule[-0.200pt]{0.400pt}{9.154pt}}
\put(665.0,359.0){\rule[-0.200pt]{4.818pt}{0.400pt}}
\put(665.0,397.0){\rule[-0.200pt]{4.818pt}{0.400pt}}
\put(757.0,359.0){\rule[-0.200pt]{0.400pt}{8.672pt}}
\put(747.0,359.0){\rule[-0.200pt]{4.818pt}{0.400pt}}
\put(747.0,395.0){\rule[-0.200pt]{4.818pt}{0.400pt}}
\put(839.0,364.0){\rule[-0.200pt]{0.400pt}{9.395pt}}
\put(829.0,364.0){\rule[-0.200pt]{4.818pt}{0.400pt}}
\put(181,211){\circle{12}}
\put(263,379){\circle{12}}
\put(346,377){\circle{12}}
\put(428,379){\circle{12}}
\put(510,383){\circle{12}}
\put(592,385){\circle{12}}
\put(675,378){\circle{12}}
\put(757,377){\circle{12}}
\put(839,383){\circle{12}}
\put(571,164){\circle{12}}
\put(829.0,403.0){\rule[-0.200pt]{4.818pt}{0.400pt}}
\put(181.0,123.0){\rule[-0.200pt]{158.512pt}{0.400pt}}
\put(839.0,123.0){\rule[-0.200pt]{0.400pt}{92.506pt}}
\put(181.0,507.0){\rule[-0.200pt]{158.512pt}{0.400pt}}
\put(181.0,123.0){\rule[-0.200pt]{0.400pt}{92.506pt}}
\end{picture}

%% file: plot_time40.tex
\setlength{\unitlength}{0.240900pt}
\ifx\plotpoint\undefined\newsavebox{\plotpoint}\fi
\sbox{\plotpoint}{\rule[-0.200pt]{0.400pt}{0.400pt}}%
\begin{picture}(900,629)(0,0)
\sbox{\plotpoint}{\rule[-0.200pt]{0.400pt}{0.400pt}}%
\put(201.0,123.0){\rule[-0.200pt]{4.818pt}{0.400pt}}
\put(181,123){\makebox(0,0)[r]{ 18}}
\put(819.0,123.0){\rule[-0.200pt]{4.818pt}{0.400pt}}
\put(201.0,187.0){\rule[-0.200pt]{4.818pt}{0.400pt}}
\put(181,187){\makebox(0,0)[r]{ 18.5}}
\put(819.0,187.0){\rule[-0.200pt]{4.818pt}{0.400pt}}
\put(201.0,251.0){\rule[-0.200pt]{4.818pt}{0.400pt}}
\put(181,251){\makebox(0,0)[r]{ 19}}
\put(819.0,251.0){\rule[-0.200pt]{4.818pt}{0.400pt}}
\put(201.0,315.0){\rule[-0.200pt]{4.818pt}{0.400pt}}
\put(181,315){\makebox(0,0)[r]{ 19.5}}
\put(819.0,315.0){\rule[-0.200pt]{4.818pt}{0.400pt}}
\put(201.0,379.0){\rule[-0.200pt]{4.818pt}{0.400pt}}
\put(181,379){\makebox(0,0)[r]{ 20}}
\put(819.0,379.0){\rule[-0.200pt]{4.818pt}{0.400pt}}
\put(201.0,443.0){\rule[-0.200pt]{4.818pt}{0.400pt}}
\put(181,443){\makebox(0,0)[r]{ 20.5}}
\put(819.0,443.0){\rule[-0.200pt]{4.818pt}{0.400pt}}
\put(201.0,507.0){\rule[-0.200pt]{4.818pt}{0.400pt}}
\put(181,507){\makebox(0,0)[r]{ 21}}
\put(819.0,507.0){\rule[-0.200pt]{4.818pt}{0.400pt}}
\put(201.0,123.0){\rule[-0.200pt]{0.400pt}{4.818pt}}
\put(201,82){\makebox(0,0){ 0}}
\put(201.0,487.0){\rule[-0.200pt]{0.400pt}{4.818pt}}
\put(281.0,123.0){\rule[-0.200pt]{0.400pt}{4.818pt}}
\put(281,82){\makebox(0,0){ 0.5}}
\put(281.0,487.0){\rule[-0.200pt]{0.400pt}{4.818pt}}
\put(361.0,123.0){\rule[-0.200pt]{0.400pt}{4.818pt}}
\put(361,82){\makebox(0,0){ 1}}
\put(361.0,487.0){\rule[-0.200pt]{0.400pt}{4.818pt}}
\put(440.0,123.0){\rule[-0.200pt]{0.400pt}{4.818pt}}
\put(440,82){\makebox(0,0){ 1.5}}
\put(440.0,487.0){\rule[-0.200pt]{0.400pt}{4.818pt}}
\put(520.0,123.0){\rule[-0.200pt]{0.400pt}{4.818pt}}
\put(520,82){\makebox(0,0){ 2}}
\put(520.0,487.0){\rule[-0.200pt]{0.400pt}{4.818pt}}
\put(600.0,123.0){\rule[-0.200pt]{0.400pt}{4.818pt}}
\put(600,82){\makebox(0,0){ 2.5}}
\put(600.0,487.0){\rule[-0.200pt]{0.400pt}{4.818pt}}
\put(680.0,123.0){\rule[-0.200pt]{0.400pt}{4.818pt}}
\put(680,82){\makebox(0,0){ 3}}
\put(680.0,487.0){\rule[-0.200pt]{0.400pt}{4.818pt}}
\put(759.0,123.0){\rule[-0.200pt]{0.400pt}{4.818pt}}
\put(759,82){\makebox(0,0){ 3.5}}
\put(759.0,487.0){\rule[-0.200pt]{0.400pt}{4.818pt}}
\put(839.0,123.0){\rule[-0.200pt]{0.400pt}{4.818pt}}
\put(839,82){\makebox(0,0){ 4}}
\put(839.0,487.0){\rule[-0.200pt]{0.400pt}{4.818pt}}
\put(201.0,123.0){\rule[-0.200pt]{153.694pt}{0.400pt}}
\put(839.0,123.0){\rule[-0.200pt]{0.400pt}{92.506pt}}
\put(201.0,507.0){\rule[-0.200pt]{153.694pt}{0.400pt}}
\put(201.0,123.0){\rule[-0.200pt]{0.400pt}{92.506pt}}
\put(40,315){\makebox(0,0){\rotatebox{90}{time (seconds)}}}
\put(520,21){\makebox(0,0){$256\epsilon$}}
\put(520,569){\makebox(0,0){$|\mathcal S|= 40$ (JT time = 1031 seconds)}}
\sbox{\plotpoint}{\rule[-0.400pt]{0.800pt}{0.800pt}}%
\sbox{\plotpoint}{\rule[-0.200pt]{0.400pt}{0.400pt}}%
\put(521,205){\makebox(0,0)[r]{MSE $< 10^{-9}$}}
\sbox{\plotpoint}{\rule[-0.400pt]{0.800pt}{0.800pt}}%
\put(541.0,205.0){\rule[-0.400pt]{24.090pt}{0.800pt}}
\put(541.0,195.0){\rule[-0.400pt]{0.800pt}{4.818pt}}
\put(641.0,195.0){\rule[-0.400pt]{0.800pt}{4.818pt}}
\put(201,174){\usebox{\plotpoint}}
\multiput(202.41,174.00)(0.501,1.230){153}{\rule{0.121pt}{2.160pt}}
\multiput(199.34,174.00)(80.000,191.517){2}{\rule{0.800pt}{1.080pt}}
\multiput(281.00,368.08)(5.745,-0.520){9}{\rule{8.200pt}{0.125pt}}
\multiput(281.00,368.34)(62.980,-8.000){2}{\rule{4.100pt}{0.800pt}}
\multiput(361.00,360.08)(3.222,-0.509){19}{\rule{5.062pt}{0.123pt}}
\multiput(361.00,360.34)(68.495,-13.000){2}{\rule{2.531pt}{0.800pt}}
\multiput(440.00,347.08)(5.745,-0.520){9}{\rule{8.200pt}{0.125pt}}
\multiput(440.00,347.34)(62.980,-8.000){2}{\rule{4.100pt}{0.800pt}}
\multiput(520.00,339.07)(8.723,-0.536){5}{\rule{10.867pt}{0.129pt}}
\multiput(520.00,339.34)(57.446,-6.000){2}{\rule{5.433pt}{0.800pt}}
\multiput(600.00,336.41)(3.564,0.511){17}{\rule{5.533pt}{0.123pt}}
\multiput(600.00,333.34)(68.515,12.000){2}{\rule{2.767pt}{0.800pt}}
\multiput(680.00,348.41)(1.211,0.503){59}{\rule{2.115pt}{0.121pt}}
\multiput(680.00,345.34)(74.610,33.000){2}{\rule{1.058pt}{0.800pt}}
\multiput(759.00,378.08)(4.384,-0.514){13}{\rule{6.600pt}{0.124pt}}
\multiput(759.00,378.34)(66.301,-10.000){2}{\rule{3.300pt}{0.800pt}}
\put(201.0,147.0){\rule[-0.400pt]{0.800pt}{13.249pt}}
\put(191.0,147.0){\rule[-0.400pt]{4.818pt}{0.800pt}}
\put(191.0,202.0){\rule[-0.400pt]{4.818pt}{0.800pt}}
\put(281.0,343.0){\rule[-0.400pt]{0.800pt}{13.009pt}}
\put(271.0,343.0){\rule[-0.400pt]{4.818pt}{0.800pt}}
\put(271.0,397.0){\rule[-0.400pt]{4.818pt}{0.800pt}}
\put(361.0,335.0){\rule[-0.400pt]{0.800pt}{12.768pt}}
\put(351.0,335.0){\rule[-0.400pt]{4.818pt}{0.800pt}}
\put(351.0,388.0){\rule[-0.400pt]{4.818pt}{0.800pt}}
\put(440.0,323.0){\rule[-0.400pt]{0.800pt}{12.768pt}}
\put(430.0,323.0){\rule[-0.400pt]{4.818pt}{0.800pt}}
\put(430.0,376.0){\rule[-0.400pt]{4.818pt}{0.800pt}}
\put(520.0,315.0){\rule[-0.400pt]{0.800pt}{12.286pt}}
\put(510.0,315.0){\rule[-0.400pt]{4.818pt}{0.800pt}}
\put(510.0,366.0){\rule[-0.400pt]{4.818pt}{0.800pt}}
\put(600.0,310.0){\rule[-0.400pt]{0.800pt}{12.045pt}}
\put(590.0,310.0){\rule[-0.400pt]{4.818pt}{0.800pt}}
\put(590.0,360.0){\rule[-0.400pt]{4.818pt}{0.800pt}}
\put(680.0,321.0){\rule[-0.400pt]{0.800pt}{12.527pt}}
\put(670.0,321.0){\rule[-0.400pt]{4.818pt}{0.800pt}}
\put(670.0,373.0){\rule[-0.400pt]{4.818pt}{0.800pt}}
\put(759.0,353.0){\rule[-0.400pt]{0.800pt}{13.009pt}}
\put(749.0,353.0){\rule[-0.400pt]{4.818pt}{0.800pt}}
\put(749.0,407.0){\rule[-0.400pt]{4.818pt}{0.800pt}}
\put(839.0,343.0){\rule[-0.400pt]{0.800pt}{13.009pt}}
\put(829.0,343.0){\rule[-0.400pt]{4.818pt}{0.800pt}}
\put(201,174){\circle{18}}
\put(281,370){\circle{18}}
\put(361,362){\circle{18}}
\put(440,349){\circle{18}}
\put(520,341){\circle{18}}
\put(600,335){\circle{18}}
\put(680,347){\circle{18}}
\put(759,380){\circle{18}}
\put(839,370){\circle{18}}
\put(591,205){\circle{18}}
\put(829.0,397.0){\rule[-0.400pt]{4.818pt}{0.800pt}}
\sbox{\plotpoint}{\rule[-0.200pt]{0.400pt}{0.400pt}}%
\put(521,164){\makebox(0,0)[r]{MSE $< 10^{-5}$}}
\put(541.0,164.0){\rule[-0.200pt]{24.090pt}{0.400pt}}
\put(541.0,154.0){\rule[-0.200pt]{0.400pt}{4.818pt}}
\put(641.0,154.0){\rule[-0.200pt]{0.400pt}{4.818pt}}
\put(201,164){\usebox{\plotpoint}}
\multiput(201.58,164.00)(0.499,1.209){157}{\rule{0.120pt}{1.065pt}}
\multiput(200.17,164.00)(80.000,190.790){2}{\rule{0.400pt}{0.533pt}}
\multiput(281.00,355.93)(5.248,-0.488){13}{\rule{4.100pt}{0.117pt}}
\multiput(281.00,356.17)(71.490,-8.000){2}{\rule{2.050pt}{0.400pt}}
\multiput(361.00,347.92)(3.113,-0.493){23}{\rule{2.531pt}{0.119pt}}
\multiput(361.00,348.17)(73.747,-13.000){2}{\rule{1.265pt}{0.400pt}}
\multiput(440.00,334.93)(4.630,-0.489){15}{\rule{3.656pt}{0.118pt}}
\multiput(440.00,335.17)(72.413,-9.000){2}{\rule{1.828pt}{0.400pt}}
\multiput(520.00,325.93)(8.836,-0.477){7}{\rule{6.500pt}{0.115pt}}
\multiput(520.00,326.17)(66.509,-5.000){2}{\rule{3.250pt}{0.400pt}}
\multiput(600.00,322.58)(3.750,0.492){19}{\rule{3.009pt}{0.118pt}}
\multiput(600.00,321.17)(73.754,11.000){2}{\rule{1.505pt}{0.400pt}}
\multiput(680.00,333.58)(1.203,0.497){63}{\rule{1.058pt}{0.120pt}}
\multiput(680.00,332.17)(76.805,33.000){2}{\rule{0.529pt}{0.400pt}}
\multiput(759.00,364.93)(4.630,-0.489){15}{\rule{3.656pt}{0.118pt}}
\multiput(759.00,365.17)(72.413,-9.000){2}{\rule{1.828pt}{0.400pt}}
\put(201.0,136.0){\rule[-0.200pt]{0.400pt}{13.249pt}}
\put(191.0,136.0){\rule[-0.200pt]{4.818pt}{0.400pt}}
\put(191.0,191.0){\rule[-0.200pt]{4.818pt}{0.400pt}}
\put(281.0,330.0){\rule[-0.200pt]{0.400pt}{13.009pt}}
\put(271.0,330.0){\rule[-0.200pt]{4.818pt}{0.400pt}}
\put(271.0,384.0){\rule[-0.200pt]{4.818pt}{0.400pt}}
\put(361.0,322.0){\rule[-0.200pt]{0.400pt}{12.768pt}}
\put(351.0,322.0){\rule[-0.200pt]{4.818pt}{0.400pt}}
\put(351.0,375.0){\rule[-0.200pt]{4.818pt}{0.400pt}}
\put(440.0,310.0){\rule[-0.200pt]{0.400pt}{12.527pt}}
\put(430.0,310.0){\rule[-0.200pt]{4.818pt}{0.400pt}}
\put(430.0,362.0){\rule[-0.200pt]{4.818pt}{0.400pt}}
\put(520.0,302.0){\rule[-0.200pt]{0.400pt}{12.286pt}}
\put(510.0,302.0){\rule[-0.200pt]{4.818pt}{0.400pt}}
\put(510.0,353.0){\rule[-0.200pt]{4.818pt}{0.400pt}}
\put(600.0,296.0){\rule[-0.200pt]{0.400pt}{12.286pt}}
\put(590.0,296.0){\rule[-0.200pt]{4.818pt}{0.400pt}}
\put(590.0,347.0){\rule[-0.200pt]{4.818pt}{0.400pt}}
\put(680.0,307.0){\rule[-0.200pt]{0.400pt}{12.527pt}}
\put(670.0,307.0){\rule[-0.200pt]{4.818pt}{0.400pt}}
\put(670.0,359.0){\rule[-0.200pt]{4.818pt}{0.400pt}}
\put(759.0,339.0){\rule[-0.200pt]{0.400pt}{13.009pt}}
\put(749.0,339.0){\rule[-0.200pt]{4.818pt}{0.400pt}}
\put(749.0,393.0){\rule[-0.200pt]{4.818pt}{0.400pt}}
\put(839.0,330.0){\rule[-0.200pt]{0.400pt}{12.768pt}}
\put(829.0,330.0){\rule[-0.200pt]{4.818pt}{0.400pt}}
\put(201,164){\circle{12}}
\put(281,357){\circle{12}}
\put(361,349){\circle{12}}
\put(440,336){\circle{12}}
\put(520,327){\circle{12}}
\put(600,322){\circle{12}}
\put(680,333){\circle{12}}
\put(759,366){\circle{12}}
\put(839,357){\circle{12}}
\put(591,164){\circle{12}}
\put(829.0,383.0){\rule[-0.200pt]{4.818pt}{0.400pt}}
\put(201.0,123.0){\rule[-0.200pt]{153.694pt}{0.400pt}}
\put(839.0,123.0){\rule[-0.200pt]{0.400pt}{92.506pt}}
\put(201.0,507.0){\rule[-0.200pt]{153.694pt}{0.400pt}}
\put(201.0,123.0){\rule[-0.200pt]{0.400pt}{92.506pt}}
\end{picture}

%% file: plot_classm10.tex
\setlength{\unitlength}{0.240900pt}
\ifx\plotpoint\undefined\newsavebox{\plotpoint}\fi
\sbox{\plotpoint}{\rule[-0.200pt]{0.400pt}{0.400pt}}%
\begin{picture}(900,629)(0,0)
\sbox{\plotpoint}{\rule[-0.200pt]{0.400pt}{0.400pt}}%
\put(201.0,123.0){\rule[-0.200pt]{4.818pt}{0.400pt}}
\put(181,123){\makebox(0,0)[r]{ 0.8}}
\put(819.0,123.0){\rule[-0.200pt]{4.818pt}{0.400pt}}
\put(201.0,161.0){\rule[-0.200pt]{4.818pt}{0.400pt}}
\put(181,161){\makebox(0,0)[r]{ 0.82}}
\put(819.0,161.0){\rule[-0.200pt]{4.818pt}{0.400pt}}
\put(201.0,200.0){\rule[-0.200pt]{4.818pt}{0.400pt}}
\put(181,200){\makebox(0,0)[r]{ 0.84}}
\put(819.0,200.0){\rule[-0.200pt]{4.818pt}{0.400pt}}
\put(201.0,238.0){\rule[-0.200pt]{4.818pt}{0.400pt}}
\put(181,238){\makebox(0,0)[r]{ 0.86}}
\put(819.0,238.0){\rule[-0.200pt]{4.818pt}{0.400pt}}
\put(201.0,277.0){\rule[-0.200pt]{4.818pt}{0.400pt}}
\put(181,277){\makebox(0,0)[r]{ 0.88}}
\put(819.0,277.0){\rule[-0.200pt]{4.818pt}{0.400pt}}
\put(201.0,315.0){\rule[-0.200pt]{4.818pt}{0.400pt}}
\put(181,315){\makebox(0,0)[r]{ 0.9}}
\put(819.0,315.0){\rule[-0.200pt]{4.818pt}{0.400pt}}
\put(201.0,353.0){\rule[-0.200pt]{4.818pt}{0.400pt}}
\put(181,353){\makebox(0,0)[r]{ 0.92}}
\put(819.0,353.0){\rule[-0.200pt]{4.818pt}{0.400pt}}
\put(201.0,392.0){\rule[-0.200pt]{4.818pt}{0.400pt}}
\put(181,392){\makebox(0,0)[r]{ 0.94}}
\put(819.0,392.0){\rule[-0.200pt]{4.818pt}{0.400pt}}
\put(201.0,430.0){\rule[-0.200pt]{4.818pt}{0.400pt}}
\put(181,430){\makebox(0,0)[r]{ 0.96}}
\put(819.0,430.0){\rule[-0.200pt]{4.818pt}{0.400pt}}
\put(201.0,469.0){\rule[-0.200pt]{4.818pt}{0.400pt}}
\put(181,469){\makebox(0,0)[r]{ 0.98}}
\put(819.0,469.0){\rule[-0.200pt]{4.818pt}{0.400pt}}
\put(201.0,507.0){\rule[-0.200pt]{4.818pt}{0.400pt}}
\put(181,507){\makebox(0,0)[r]{ 1}}
\put(819.0,507.0){\rule[-0.200pt]{4.818pt}{0.400pt}}
\put(201.0,123.0){\rule[-0.200pt]{0.400pt}{4.818pt}}
\put(201,82){\makebox(0,0){ 0}}
\put(201.0,487.0){\rule[-0.200pt]{0.400pt}{4.818pt}}
\put(281.0,123.0){\rule[-0.200pt]{0.400pt}{4.818pt}}
\put(281,82){\makebox(0,0){ 0.5}}
\put(281.0,487.0){\rule[-0.200pt]{0.400pt}{4.818pt}}
\put(361.0,123.0){\rule[-0.200pt]{0.400pt}{4.818pt}}
\put(361,82){\makebox(0,0){ 1}}
\put(361.0,487.0){\rule[-0.200pt]{0.400pt}{4.818pt}}
\put(440.0,123.0){\rule[-0.200pt]{0.400pt}{4.818pt}}
\put(440,82){\makebox(0,0){ 1.5}}
\put(440.0,487.0){\rule[-0.200pt]{0.400pt}{4.818pt}}
\put(520.0,123.0){\rule[-0.200pt]{0.400pt}{4.818pt}}
\put(520,82){\makebox(0,0){ 2}}
\put(520.0,487.0){\rule[-0.200pt]{0.400pt}{4.818pt}}
\put(600.0,123.0){\rule[-0.200pt]{0.400pt}{4.818pt}}
\put(600,82){\makebox(0,0){ 2.5}}
\put(600.0,487.0){\rule[-0.200pt]{0.400pt}{4.818pt}}
\put(680.0,123.0){\rule[-0.200pt]{0.400pt}{4.818pt}}
\put(680,82){\makebox(0,0){ 3}}
\put(680.0,487.0){\rule[-0.200pt]{0.400pt}{4.818pt}}
\put(759.0,123.0){\rule[-0.200pt]{0.400pt}{4.818pt}}
\put(759,82){\makebox(0,0){ 3.5}}
\put(759.0,487.0){\rule[-0.200pt]{0.400pt}{4.818pt}}
\put(839.0,123.0){\rule[-0.200pt]{0.400pt}{4.818pt}}
\put(839,82){\makebox(0,0){ 4}}
\put(839.0,487.0){\rule[-0.200pt]{0.400pt}{4.818pt}}
\put(201.0,123.0){\rule[-0.200pt]{153.694pt}{0.400pt}}
\put(839.0,123.0){\rule[-0.200pt]{0.400pt}{92.506pt}}
\put(201.0,507.0){\rule[-0.200pt]{153.694pt}{0.400pt}}
\put(201.0,123.0){\rule[-0.200pt]{0.400pt}{92.506pt}}
\put(40,315){\makebox(0,0){\rotatebox{90}{matching accuracy}}}
\put(520,21){\makebox(0,0){$256\epsilon$}}
\put(520,569){\makebox(0,0){$|\mathcal S|= 10$}}
\put(521,205){\makebox(0,0)[r]{MSE $< 10^{-8}$}}
\put(541.0,205.0){\rule[-0.200pt]{24.090pt}{0.400pt}}
\put(541.0,195.0){\rule[-0.200pt]{0.400pt}{4.818pt}}
\put(641.0,195.0){\rule[-0.200pt]{0.400pt}{4.818pt}}
\put(201,507){\usebox{\plotpoint}}
\multiput(201.00,505.94)(11.594,-0.468){5}{\rule{8.100pt}{0.113pt}}
\multiput(201.00,506.17)(63.188,-4.000){2}{\rule{4.050pt}{0.400pt}}
\multiput(281.00,501.92)(3.750,-0.492){19}{\rule{3.009pt}{0.118pt}}
\multiput(281.00,502.17)(73.754,-11.000){2}{\rule{1.505pt}{0.400pt}}
\multiput(361.00,490.94)(11.448,-0.468){5}{\rule{8.000pt}{0.113pt}}
\multiput(361.00,491.17)(62.396,-4.000){2}{\rule{4.000pt}{0.400pt}}
\multiput(440.00,486.94)(11.594,-0.468){5}{\rule{8.100pt}{0.113pt}}
\multiput(440.00,487.17)(63.188,-4.000){2}{\rule{4.050pt}{0.400pt}}
\multiput(759.00,482.92)(1.757,-0.496){43}{\rule{1.491pt}{0.120pt}}
\multiput(759.00,483.17)(76.905,-23.000){2}{\rule{0.746pt}{0.400pt}}
\put(520.0,484.0){\rule[-0.200pt]{57.575pt}{0.400pt}}
\put(201,507){\usebox{\plotpoint}}
\put(191.0,507.0){\rule[-0.200pt]{4.818pt}{0.400pt}}
\put(191.0,507.0){\rule[-0.200pt]{4.818pt}{0.400pt}}
\put(281.0,499.0){\rule[-0.200pt]{0.400pt}{1.927pt}}
\put(271.0,499.0){\rule[-0.200pt]{4.818pt}{0.400pt}}
\put(271.0,507.0){\rule[-0.200pt]{4.818pt}{0.400pt}}
\put(361.0,484.0){\rule[-0.200pt]{0.400pt}{3.613pt}}
\put(351.0,484.0){\rule[-0.200pt]{4.818pt}{0.400pt}}
\put(351.0,499.0){\rule[-0.200pt]{4.818pt}{0.400pt}}
\put(440.0,478.0){\rule[-0.200pt]{0.400pt}{4.818pt}}
\put(430.0,478.0){\rule[-0.200pt]{4.818pt}{0.400pt}}
\put(430.0,498.0){\rule[-0.200pt]{4.818pt}{0.400pt}}
\put(520.0,474.0){\rule[-0.200pt]{0.400pt}{4.818pt}}
\put(510.0,474.0){\rule[-0.200pt]{4.818pt}{0.400pt}}
\put(510.0,494.0){\rule[-0.200pt]{4.818pt}{0.400pt}}
\put(600.0,474.0){\rule[-0.200pt]{0.400pt}{4.818pt}}
\put(590.0,474.0){\rule[-0.200pt]{4.818pt}{0.400pt}}
\put(590.0,494.0){\rule[-0.200pt]{4.818pt}{0.400pt}}
\put(680.0,474.0){\rule[-0.200pt]{0.400pt}{4.818pt}}
\put(670.0,474.0){\rule[-0.200pt]{4.818pt}{0.400pt}}
\put(670.0,494.0){\rule[-0.200pt]{4.818pt}{0.400pt}}
\put(759.0,474.0){\rule[-0.200pt]{0.400pt}{4.818pt}}
\put(749.0,474.0){\rule[-0.200pt]{4.818pt}{0.400pt}}
\put(749.0,494.0){\rule[-0.200pt]{4.818pt}{0.400pt}}
\put(839.0,446.0){\rule[-0.200pt]{0.400pt}{7.227pt}}
\put(829.0,446.0){\rule[-0.200pt]{4.818pt}{0.400pt}}
\put(201,507){\circle*{12}}
\put(281,503){\circle*{12}}
\put(361,492){\circle*{12}}
\put(440,488){\circle*{12}}
\put(520,484){\circle*{12}}
\put(600,484){\circle*{12}}
\put(680,484){\circle*{12}}
\put(759,484){\circle*{12}}
\put(839,461){\circle*{12}}
\put(591,205){\circle*{12}}
\put(829.0,476.0){\rule[-0.200pt]{4.818pt}{0.400pt}}
\put(521,164){\makebox(0,0)[r]{MSE $< 10^{-4}$}}
\put(541.0,164.0){\rule[-0.200pt]{24.090pt}{0.400pt}}
\put(541.0,154.0){\rule[-0.200pt]{0.400pt}{4.818pt}}
\put(641.0,154.0){\rule[-0.200pt]{0.400pt}{4.818pt}}
\put(201,507){\usebox{\plotpoint}}
\multiput(201.58,504.51)(0.499,-0.625){157}{\rule{0.120pt}{0.600pt}}
\multiput(200.17,505.75)(80.000,-98.755){2}{\rule{0.400pt}{0.300pt}}
\multiput(281.00,405.94)(11.594,-0.468){5}{\rule{8.100pt}{0.113pt}}
\multiput(281.00,406.17)(63.188,-4.000){2}{\rule{4.050pt}{0.400pt}}
\multiput(440.00,401.93)(6.061,-0.485){11}{\rule{4.671pt}{0.117pt}}
\multiput(440.00,402.17)(70.304,-7.000){2}{\rule{2.336pt}{0.400pt}}
\multiput(520.00,394.92)(3.426,-0.492){21}{\rule{2.767pt}{0.119pt}}
\multiput(520.00,395.17)(74.258,-12.000){2}{\rule{1.383pt}{0.400pt}}
\put(361.0,403.0){\rule[-0.200pt]{19.031pt}{0.400pt}}
\multiput(680.00,382.92)(3.703,-0.492){19}{\rule{2.973pt}{0.118pt}}
\multiput(680.00,383.17)(72.830,-11.000){2}{\rule{1.486pt}{0.400pt}}
\multiput(759.00,371.94)(11.594,-0.468){5}{\rule{8.100pt}{0.113pt}}
\multiput(759.00,372.17)(63.188,-4.000){2}{\rule{4.050pt}{0.400pt}}
\put(600.0,384.0){\rule[-0.200pt]{19.272pt}{0.400pt}}
\put(201,507){\usebox{\plotpoint}}
\put(191.0,507.0){\rule[-0.200pt]{4.818pt}{0.400pt}}
\put(191.0,507.0){\rule[-0.200pt]{4.818pt}{0.400pt}}
\put(281.0,387.0){\rule[-0.200pt]{0.400pt}{9.636pt}}
\put(271.0,387.0){\rule[-0.200pt]{4.818pt}{0.400pt}}
\put(271.0,427.0){\rule[-0.200pt]{4.818pt}{0.400pt}}
\put(361.0,384.0){\rule[-0.200pt]{0.400pt}{9.395pt}}
\put(351.0,384.0){\rule[-0.200pt]{4.818pt}{0.400pt}}
\put(351.0,423.0){\rule[-0.200pt]{4.818pt}{0.400pt}}
\put(440.0,384.0){\rule[-0.200pt]{0.400pt}{9.395pt}}
\put(430.0,384.0){\rule[-0.200pt]{4.818pt}{0.400pt}}
\put(430.0,423.0){\rule[-0.200pt]{4.818pt}{0.400pt}}
\put(520.0,375.0){\rule[-0.200pt]{0.400pt}{9.877pt}}
\put(510.0,375.0){\rule[-0.200pt]{4.818pt}{0.400pt}}
\put(510.0,416.0){\rule[-0.200pt]{4.818pt}{0.400pt}}
\put(600.0,363.0){\rule[-0.200pt]{0.400pt}{10.118pt}}
\put(590.0,363.0){\rule[-0.200pt]{4.818pt}{0.400pt}}
\put(590.0,405.0){\rule[-0.200pt]{4.818pt}{0.400pt}}
\put(680.0,363.0){\rule[-0.200pt]{0.400pt}{10.118pt}}
\put(670.0,363.0){\rule[-0.200pt]{4.818pt}{0.400pt}}
\put(670.0,405.0){\rule[-0.200pt]{4.818pt}{0.400pt}}
\put(759.0,351.0){\rule[-0.200pt]{0.400pt}{10.359pt}}
\put(749.0,351.0){\rule[-0.200pt]{4.818pt}{0.400pt}}
\put(749.0,394.0){\rule[-0.200pt]{4.818pt}{0.400pt}}
\put(839.0,346.0){\rule[-0.200pt]{0.400pt}{10.840pt}}
\put(829.0,346.0){\rule[-0.200pt]{4.818pt}{0.400pt}}
\put(201,507){\circle{12}}
\put(281,407){\circle{12}}
\put(361,403){\circle{12}}
\put(440,403){\circle{12}}
\put(520,396){\circle{12}}
\put(600,384){\circle{12}}
\put(680,384){\circle{12}}
\put(759,373){\circle{12}}
\put(839,369){\circle{12}}
\put(591,164){\circle{12}}
\put(829.0,391.0){\rule[-0.200pt]{4.818pt}{0.400pt}}
\put(201.0,123.0){\rule[-0.200pt]{153.694pt}{0.400pt}}
\put(839.0,123.0){\rule[-0.200pt]{0.400pt}{92.506pt}}
\put(201.0,507.0){\rule[-0.200pt]{153.694pt}{0.400pt}}
\put(201.0,123.0){\rule[-0.200pt]{0.400pt}{92.506pt}}
\end{picture}

%% file: plot_classm20.tex
\setlength{\unitlength}{0.240900pt}
\ifx\plotpoint\undefined\newsavebox{\plotpoint}\fi
\sbox{\plotpoint}{\rule[-0.200pt]{0.400pt}{0.400pt}}%
\begin{picture}(900,629)(0,0)
\sbox{\plotpoint}{\rule[-0.200pt]{0.400pt}{0.400pt}}%
\put(201.0,123.0){\rule[-0.200pt]{4.818pt}{0.400pt}}
\put(181,123){\makebox(0,0)[r]{ 0.8}}
\put(819.0,123.0){\rule[-0.200pt]{4.818pt}{0.400pt}}
\put(201.0,161.0){\rule[-0.200pt]{4.818pt}{0.400pt}}
\put(181,161){\makebox(0,0)[r]{ 0.82}}
\put(819.0,161.0){\rule[-0.200pt]{4.818pt}{0.400pt}}
\put(201.0,200.0){\rule[-0.200pt]{4.818pt}{0.400pt}}
\put(181,200){\makebox(0,0)[r]{ 0.84}}
\put(819.0,200.0){\rule[-0.200pt]{4.818pt}{0.400pt}}
\put(201.0,238.0){\rule[-0.200pt]{4.818pt}{0.400pt}}
\put(181,238){\makebox(0,0)[r]{ 0.86}}
\put(819.0,238.0){\rule[-0.200pt]{4.818pt}{0.400pt}}
\put(201.0,277.0){\rule[-0.200pt]{4.818pt}{0.400pt}}
\put(181,277){\makebox(0,0)[r]{ 0.88}}
\put(819.0,277.0){\rule[-0.200pt]{4.818pt}{0.400pt}}
\put(201.0,315.0){\rule[-0.200pt]{4.818pt}{0.400pt}}
\put(181,315){\makebox(0,0)[r]{ 0.9}}
\put(819.0,315.0){\rule[-0.200pt]{4.818pt}{0.400pt}}
\put(201.0,353.0){\rule[-0.200pt]{4.818pt}{0.400pt}}
\put(181,353){\makebox(0,0)[r]{ 0.92}}
\put(819.0,353.0){\rule[-0.200pt]{4.818pt}{0.400pt}}
\put(201.0,392.0){\rule[-0.200pt]{4.818pt}{0.400pt}}
\put(181,392){\makebox(0,0)[r]{ 0.94}}
\put(819.0,392.0){\rule[-0.200pt]{4.818pt}{0.400pt}}
\put(201.0,430.0){\rule[-0.200pt]{4.818pt}{0.400pt}}
\put(181,430){\makebox(0,0)[r]{ 0.96}}
\put(819.0,430.0){\rule[-0.200pt]{4.818pt}{0.400pt}}
\put(201.0,469.0){\rule[-0.200pt]{4.818pt}{0.400pt}}
\put(181,469){\makebox(0,0)[r]{ 0.98}}
\put(819.0,469.0){\rule[-0.200pt]{4.818pt}{0.400pt}}
\put(201.0,507.0){\rule[-0.200pt]{4.818pt}{0.400pt}}
\put(181,507){\makebox(0,0)[r]{ 1}}
\put(819.0,507.0){\rule[-0.200pt]{4.818pt}{0.400pt}}
\put(201.0,123.0){\rule[-0.200pt]{0.400pt}{4.818pt}}
\put(201,82){\makebox(0,0){ 0}}
\put(201.0,487.0){\rule[-0.200pt]{0.400pt}{4.818pt}}
\put(281.0,123.0){\rule[-0.200pt]{0.400pt}{4.818pt}}
\put(281,82){\makebox(0,0){ 0.5}}
\put(281.0,487.0){\rule[-0.200pt]{0.400pt}{4.818pt}}
\put(361.0,123.0){\rule[-0.200pt]{0.400pt}{4.818pt}}
\put(361,82){\makebox(0,0){ 1}}
\put(361.0,487.0){\rule[-0.200pt]{0.400pt}{4.818pt}}
\put(440.0,123.0){\rule[-0.200pt]{0.400pt}{4.818pt}}
\put(440,82){\makebox(0,0){ 1.5}}
\put(440.0,487.0){\rule[-0.200pt]{0.400pt}{4.818pt}}
\put(520.0,123.0){\rule[-0.200pt]{0.400pt}{4.818pt}}
\put(520,82){\makebox(0,0){ 2}}
\put(520.0,487.0){\rule[-0.200pt]{0.400pt}{4.818pt}}
\put(600.0,123.0){\rule[-0.200pt]{0.400pt}{4.818pt}}
\put(600,82){\makebox(0,0){ 2.5}}
\put(600.0,487.0){\rule[-0.200pt]{0.400pt}{4.818pt}}
\put(680.0,123.0){\rule[-0.200pt]{0.400pt}{4.818pt}}
\put(680,82){\makebox(0,0){ 3}}
\put(680.0,487.0){\rule[-0.200pt]{0.400pt}{4.818pt}}
\put(759.0,123.0){\rule[-0.200pt]{0.400pt}{4.818pt}}
\put(759,82){\makebox(0,0){ 3.5}}
\put(759.0,487.0){\rule[-0.200pt]{0.400pt}{4.818pt}}
\put(839.0,123.0){\rule[-0.200pt]{0.400pt}{4.818pt}}
\put(839,82){\makebox(0,0){ 4}}
\put(839.0,487.0){\rule[-0.200pt]{0.400pt}{4.818pt}}
\put(201.0,123.0){\rule[-0.200pt]{153.694pt}{0.400pt}}
\put(839.0,123.0){\rule[-0.200pt]{0.400pt}{92.506pt}}
\put(201.0,507.0){\rule[-0.200pt]{153.694pt}{0.400pt}}
\put(201.0,123.0){\rule[-0.200pt]{0.400pt}{92.506pt}}
\put(40,315){\makebox(0,0){\rotatebox{90}{matching accuracy}}}
\put(520,21){\makebox(0,0){$256\epsilon$}}
\put(520,569){\makebox(0,0){$|\mathcal S|= 20$}}
\put(521,205){\makebox(0,0)[r]{MSE $< 10^{-8}$}}
\put(541.0,205.0){\rule[-0.200pt]{24.090pt}{0.400pt}}
\put(541.0,195.0){\rule[-0.200pt]{0.400pt}{4.818pt}}
\put(641.0,195.0){\rule[-0.200pt]{0.400pt}{4.818pt}}
\put(201,507){\usebox{\plotpoint}}
\multiput(201.00,505.93)(5.248,-0.488){13}{\rule{4.100pt}{0.117pt}}
\multiput(201.00,506.17)(71.490,-8.000){2}{\rule{2.050pt}{0.400pt}}
\multiput(281.00,497.93)(6.061,-0.485){11}{\rule{4.671pt}{0.117pt}}
\multiput(281.00,498.17)(70.304,-7.000){2}{\rule{2.336pt}{0.400pt}}
\multiput(361.00,490.94)(11.448,-0.468){5}{\rule{8.000pt}{0.113pt}}
\multiput(361.00,491.17)(62.396,-4.000){2}{\rule{4.000pt}{0.400pt}}
\multiput(440.00,486.93)(5.248,-0.488){13}{\rule{4.100pt}{0.117pt}}
\multiput(440.00,487.17)(71.490,-8.000){2}{\rule{2.050pt}{0.400pt}}
\multiput(520.00,478.94)(11.594,-0.468){5}{\rule{8.100pt}{0.113pt}}
\multiput(520.00,479.17)(63.188,-4.000){2}{\rule{4.050pt}{0.400pt}}
\multiput(600.00,474.93)(6.061,-0.485){11}{\rule{4.671pt}{0.117pt}}
\multiput(600.00,475.17)(70.304,-7.000){2}{\rule{2.336pt}{0.400pt}}
\multiput(680.00,467.92)(3.383,-0.492){21}{\rule{2.733pt}{0.119pt}}
\multiput(680.00,468.17)(73.327,-12.000){2}{\rule{1.367pt}{0.400pt}}
\multiput(759.00,455.92)(1.493,-0.497){51}{\rule{1.285pt}{0.120pt}}
\multiput(759.00,456.17)(77.333,-27.000){2}{\rule{0.643pt}{0.400pt}}
\put(201,507){\usebox{\plotpoint}}
\put(191.0,507.0){\rule[-0.200pt]{4.818pt}{0.400pt}}
\put(191.0,507.0){\rule[-0.200pt]{4.818pt}{0.400pt}}
\put(281.0,494.0){\rule[-0.200pt]{0.400pt}{2.650pt}}
\put(271.0,494.0){\rule[-0.200pt]{4.818pt}{0.400pt}}
\put(271.0,505.0){\rule[-0.200pt]{4.818pt}{0.400pt}}
\put(361.0,484.0){\rule[-0.200pt]{0.400pt}{3.613pt}}
\put(351.0,484.0){\rule[-0.200pt]{4.818pt}{0.400pt}}
\put(351.0,499.0){\rule[-0.200pt]{4.818pt}{0.400pt}}
\put(440.0,478.0){\rule[-0.200pt]{0.400pt}{4.818pt}}
\put(430.0,478.0){\rule[-0.200pt]{4.818pt}{0.400pt}}
\put(430.0,498.0){\rule[-0.200pt]{4.818pt}{0.400pt}}
\put(520.0,469.0){\rule[-0.200pt]{0.400pt}{5.300pt}}
\put(510.0,469.0){\rule[-0.200pt]{4.818pt}{0.400pt}}
\put(510.0,491.0){\rule[-0.200pt]{4.818pt}{0.400pt}}
\put(600.0,465.0){\rule[-0.200pt]{0.400pt}{5.541pt}}
\put(590.0,465.0){\rule[-0.200pt]{4.818pt}{0.400pt}}
\put(590.0,488.0){\rule[-0.200pt]{4.818pt}{0.400pt}}
\put(680.0,456.0){\rule[-0.200pt]{0.400pt}{6.022pt}}
\put(670.0,456.0){\rule[-0.200pt]{4.818pt}{0.400pt}}
\put(670.0,481.0){\rule[-0.200pt]{4.818pt}{0.400pt}}
\put(759.0,444.0){\rule[-0.200pt]{0.400pt}{6.263pt}}
\put(749.0,444.0){\rule[-0.200pt]{4.818pt}{0.400pt}}
\put(749.0,470.0){\rule[-0.200pt]{4.818pt}{0.400pt}}
\put(839.0,414.0){\rule[-0.200pt]{0.400pt}{7.709pt}}
\put(829.0,414.0){\rule[-0.200pt]{4.818pt}{0.400pt}}
\put(201,507){\circle*{12}}
\put(281,499){\circle*{12}}
\put(361,492){\circle*{12}}
\put(440,488){\circle*{12}}
\put(520,480){\circle*{12}}
\put(600,476){\circle*{12}}
\put(680,469){\circle*{12}}
\put(759,457){\circle*{12}}
\put(839,430){\circle*{12}}
\put(591,205){\circle*{12}}
\put(829.0,446.0){\rule[-0.200pt]{4.818pt}{0.400pt}}
\put(521,164){\makebox(0,0)[r]{MSE $< 10^{-4}$}}
\put(541.0,164.0){\rule[-0.200pt]{24.090pt}{0.400pt}}
\put(541.0,154.0){\rule[-0.200pt]{0.400pt}{4.818pt}}
\put(641.0,154.0){\rule[-0.200pt]{0.400pt}{4.818pt}}
\put(201,507){\usebox{\plotpoint}}
\multiput(201.00,505.92)(1.493,-0.497){51}{\rule{1.285pt}{0.120pt}}
\multiput(201.00,506.17)(77.333,-27.000){2}{\rule{0.643pt}{0.400pt}}
\multiput(281.00,478.93)(5.248,-0.488){13}{\rule{4.100pt}{0.117pt}}
\multiput(281.00,479.17)(71.490,-8.000){2}{\rule{2.050pt}{0.400pt}}
\multiput(361.00,470.92)(2.108,-0.495){35}{\rule{1.763pt}{0.119pt}}
\multiput(361.00,471.17)(75.340,-19.000){2}{\rule{0.882pt}{0.400pt}}
\multiput(440.00,451.92)(3.750,-0.492){19}{\rule{3.009pt}{0.118pt}}
\multiput(440.00,452.17)(73.754,-11.000){2}{\rule{1.505pt}{0.400pt}}
\multiput(520.00,440.94)(11.594,-0.468){5}{\rule{8.100pt}{0.113pt}}
\multiput(520.00,441.17)(63.188,-4.000){2}{\rule{4.050pt}{0.400pt}}
\multiput(600.00,436.92)(3.426,-0.492){21}{\rule{2.767pt}{0.119pt}}
\multiput(600.00,437.17)(74.258,-12.000){2}{\rule{1.383pt}{0.400pt}}
\multiput(680.00,424.95)(17.430,-0.447){3}{\rule{10.633pt}{0.108pt}}
\multiput(680.00,425.17)(56.930,-3.000){2}{\rule{5.317pt}{0.400pt}}
\multiput(759.00,421.92)(2.026,-0.496){37}{\rule{1.700pt}{0.119pt}}
\multiput(759.00,422.17)(76.472,-20.000){2}{\rule{0.850pt}{0.400pt}}
\put(201,507){\usebox{\plotpoint}}
\put(191.0,507.0){\rule[-0.200pt]{4.818pt}{0.400pt}}
\put(191.0,507.0){\rule[-0.200pt]{4.818pt}{0.400pt}}
\put(281.0,469.0){\rule[-0.200pt]{0.400pt}{5.300pt}}
\put(271.0,469.0){\rule[-0.200pt]{4.818pt}{0.400pt}}
\put(271.0,491.0){\rule[-0.200pt]{4.818pt}{0.400pt}}
\put(361.0,461.0){\rule[-0.200pt]{0.400pt}{5.541pt}}
\put(351.0,461.0){\rule[-0.200pt]{4.818pt}{0.400pt}}
\put(351.0,484.0){\rule[-0.200pt]{4.818pt}{0.400pt}}
\put(440.0,439.0){\rule[-0.200pt]{0.400pt}{6.986pt}}
\put(430.0,439.0){\rule[-0.200pt]{4.818pt}{0.400pt}}
\put(430.0,468.0){\rule[-0.200pt]{4.818pt}{0.400pt}}
\put(520.0,427.0){\rule[-0.200pt]{0.400pt}{7.227pt}}
\put(510.0,427.0){\rule[-0.200pt]{4.818pt}{0.400pt}}
\put(510.0,457.0){\rule[-0.200pt]{4.818pt}{0.400pt}}
\put(600.0,423.0){\rule[-0.200pt]{0.400pt}{7.227pt}}
\put(590.0,423.0){\rule[-0.200pt]{4.818pt}{0.400pt}}
\put(590.0,453.0){\rule[-0.200pt]{4.818pt}{0.400pt}}
\put(680.0,411.0){\rule[-0.200pt]{0.400pt}{7.468pt}}
\put(670.0,411.0){\rule[-0.200pt]{4.818pt}{0.400pt}}
\put(670.0,442.0){\rule[-0.200pt]{4.818pt}{0.400pt}}
\put(759.0,407.0){\rule[-0.200pt]{0.400pt}{7.468pt}}
\put(749.0,407.0){\rule[-0.200pt]{4.818pt}{0.400pt}}
\put(749.0,438.0){\rule[-0.200pt]{4.818pt}{0.400pt}}
\put(839.0,386.0){\rule[-0.200pt]{0.400pt}{8.431pt}}
\put(829.0,386.0){\rule[-0.200pt]{4.818pt}{0.400pt}}
\put(201,507){\circle{12}}
\put(281,480){\circle{12}}
\put(361,472){\circle{12}}
\put(440,453){\circle{12}}
\put(520,442){\circle{12}}
\put(600,438){\circle{12}}
\put(680,426){\circle{12}}
\put(759,423){\circle{12}}
\put(839,403){\circle{12}}
\put(591,164){\circle{12}}
\put(829.0,421.0){\rule[-0.200pt]{4.818pt}{0.400pt}}
\put(201.0,123.0){\rule[-0.200pt]{153.694pt}{0.400pt}}
\put(839.0,123.0){\rule[-0.200pt]{0.400pt}{92.506pt}}
\put(201.0,507.0){\rule[-0.200pt]{153.694pt}{0.400pt}}
\put(201.0,123.0){\rule[-0.200pt]{0.400pt}{92.506pt}}
\end{picture}

%% file: plot_classm30.tex
\setlength{\unitlength}{0.240900pt}
\ifx\plotpoint\undefined\newsavebox{\plotpoint}\fi
\sbox{\plotpoint}{\rule[-0.200pt]{0.400pt}{0.400pt}}%
\begin{picture}(900,629)(0,0)
\sbox{\plotpoint}{\rule[-0.200pt]{0.400pt}{0.400pt}}%
\put(201.0,123.0){\rule[-0.200pt]{4.818pt}{0.400pt}}
\put(181,123){\makebox(0,0)[r]{ 0.8}}
\put(819.0,123.0){\rule[-0.200pt]{4.818pt}{0.400pt}}
\put(201.0,161.0){\rule[-0.200pt]{4.818pt}{0.400pt}}
\put(181,161){\makebox(0,0)[r]{ 0.82}}
\put(819.0,161.0){\rule[-0.200pt]{4.818pt}{0.400pt}}
\put(201.0,200.0){\rule[-0.200pt]{4.818pt}{0.400pt}}
\put(181,200){\makebox(0,0)[r]{ 0.84}}
\put(819.0,200.0){\rule[-0.200pt]{4.818pt}{0.400pt}}
\put(201.0,238.0){\rule[-0.200pt]{4.818pt}{0.400pt}}
\put(181,238){\makebox(0,0)[r]{ 0.86}}
\put(819.0,238.0){\rule[-0.200pt]{4.818pt}{0.400pt}}
\put(201.0,277.0){\rule[-0.200pt]{4.818pt}{0.400pt}}
\put(181,277){\makebox(0,0)[r]{ 0.88}}
\put(819.0,277.0){\rule[-0.200pt]{4.818pt}{0.400pt}}
\put(201.0,315.0){\rule[-0.200pt]{4.818pt}{0.400pt}}
\put(181,315){\makebox(0,0)[r]{ 0.9}}
\put(819.0,315.0){\rule[-0.200pt]{4.818pt}{0.400pt}}
\put(201.0,353.0){\rule[-0.200pt]{4.818pt}{0.400pt}}
\put(181,353){\makebox(0,0)[r]{ 0.92}}
\put(819.0,353.0){\rule[-0.200pt]{4.818pt}{0.400pt}}
\put(201.0,392.0){\rule[-0.200pt]{4.818pt}{0.400pt}}
\put(181,392){\makebox(0,0)[r]{ 0.94}}
\put(819.0,392.0){\rule[-0.200pt]{4.818pt}{0.400pt}}
\put(201.0,430.0){\rule[-0.200pt]{4.818pt}{0.400pt}}
\put(181,430){\makebox(0,0)[r]{ 0.96}}
\put(819.0,430.0){\rule[-0.200pt]{4.818pt}{0.400pt}}
\put(201.0,469.0){\rule[-0.200pt]{4.818pt}{0.400pt}}
\put(181,469){\makebox(0,0)[r]{ 0.98}}
\put(819.0,469.0){\rule[-0.200pt]{4.818pt}{0.400pt}}
\put(201.0,507.0){\rule[-0.200pt]{4.818pt}{0.400pt}}
\put(181,507){\makebox(0,0)[r]{ 1}}
\put(819.0,507.0){\rule[-0.200pt]{4.818pt}{0.400pt}}
\put(201.0,123.0){\rule[-0.200pt]{0.400pt}{4.818pt}}
\put(201,82){\makebox(0,0){ 0}}
\put(201.0,487.0){\rule[-0.200pt]{0.400pt}{4.818pt}}
\put(281.0,123.0){\rule[-0.200pt]{0.400pt}{4.818pt}}
\put(281,82){\makebox(0,0){ 0.5}}
\put(281.0,487.0){\rule[-0.200pt]{0.400pt}{4.818pt}}
\put(361.0,123.0){\rule[-0.200pt]{0.400pt}{4.818pt}}
\put(361,82){\makebox(0,0){ 1}}
\put(361.0,487.0){\rule[-0.200pt]{0.400pt}{4.818pt}}
\put(440.0,123.0){\rule[-0.200pt]{0.400pt}{4.818pt}}
\put(440,82){\makebox(0,0){ 1.5}}
\put(440.0,487.0){\rule[-0.200pt]{0.400pt}{4.818pt}}
\put(520.0,123.0){\rule[-0.200pt]{0.400pt}{4.818pt}}
\put(520,82){\makebox(0,0){ 2}}
\put(520.0,487.0){\rule[-0.200pt]{0.400pt}{4.818pt}}
\put(600.0,123.0){\rule[-0.200pt]{0.400pt}{4.818pt}}
\put(600,82){\makebox(0,0){ 2.5}}
\put(600.0,487.0){\rule[-0.200pt]{0.400pt}{4.818pt}}
\put(680.0,123.0){\rule[-0.200pt]{0.400pt}{4.818pt}}
\put(680,82){\makebox(0,0){ 3}}
\put(680.0,487.0){\rule[-0.200pt]{0.400pt}{4.818pt}}
\put(759.0,123.0){\rule[-0.200pt]{0.400pt}{4.818pt}}
\put(759,82){\makebox(0,0){ 3.5}}
\put(759.0,487.0){\rule[-0.200pt]{0.400pt}{4.818pt}}
\put(839.0,123.0){\rule[-0.200pt]{0.400pt}{4.818pt}}
\put(839,82){\makebox(0,0){ 4}}
\put(839.0,487.0){\rule[-0.200pt]{0.400pt}{4.818pt}}
\put(201.0,123.0){\rule[-0.200pt]{153.694pt}{0.400pt}}
\put(839.0,123.0){\rule[-0.200pt]{0.400pt}{92.506pt}}
\put(201.0,507.0){\rule[-0.200pt]{153.694pt}{0.400pt}}
\put(201.0,123.0){\rule[-0.200pt]{0.400pt}{92.506pt}}
\put(40,315){\makebox(0,0){\rotatebox{90}{matching accuracy}}}
\put(520,21){\makebox(0,0){$256\epsilon$}}
\put(520,569){\makebox(0,0){$|\mathcal S|= 30$}}
\put(521,205){\makebox(0,0)[r]{MSE $< 10^{-9}$}}
\put(541.0,205.0){\rule[-0.200pt]{24.090pt}{0.400pt}}
\put(541.0,195.0){\rule[-0.200pt]{0.400pt}{4.818pt}}
\put(641.0,195.0){\rule[-0.200pt]{0.400pt}{4.818pt}}
\put(201,507){\usebox{\plotpoint}}
\multiput(201.00,505.94)(11.594,-0.468){5}{\rule{8.100pt}{0.113pt}}
\multiput(201.00,506.17)(63.188,-4.000){2}{\rule{4.050pt}{0.400pt}}
\multiput(281.00,501.92)(3.750,-0.492){19}{\rule{3.009pt}{0.118pt}}
\multiput(281.00,502.17)(73.754,-11.000){2}{\rule{1.505pt}{0.400pt}}
\multiput(361.00,490.92)(3.383,-0.492){21}{\rule{2.733pt}{0.119pt}}
\multiput(361.00,491.17)(73.327,-12.000){2}{\rule{1.367pt}{0.400pt}}
\multiput(440.00,478.92)(2.720,-0.494){27}{\rule{2.233pt}{0.119pt}}
\multiput(440.00,479.17)(75.365,-15.000){2}{\rule{1.117pt}{0.400pt}}
\multiput(520.00,463.92)(2.546,-0.494){29}{\rule{2.100pt}{0.119pt}}
\multiput(520.00,464.17)(75.641,-16.000){2}{\rule{1.050pt}{0.400pt}}
\multiput(600.00,447.92)(0.955,-0.498){81}{\rule{0.862pt}{0.120pt}}
\multiput(600.00,448.17)(78.211,-42.000){2}{\rule{0.431pt}{0.400pt}}
\multiput(680.00,405.92)(1.282,-0.497){59}{\rule{1.119pt}{0.120pt}}
\multiput(680.00,406.17)(76.677,-31.000){2}{\rule{0.560pt}{0.400pt}}
\multiput(759.00,374.92)(1.057,-0.498){73}{\rule{0.942pt}{0.120pt}}
\multiput(759.00,375.17)(78.045,-38.000){2}{\rule{0.471pt}{0.400pt}}
\put(201,507){\usebox{\plotpoint}}
\put(191.0,507.0){\rule[-0.200pt]{4.818pt}{0.400pt}}
\put(191.0,507.0){\rule[-0.200pt]{4.818pt}{0.400pt}}
\put(281.0,499.0){\rule[-0.200pt]{0.400pt}{1.927pt}}
\put(271.0,499.0){\rule[-0.200pt]{4.818pt}{0.400pt}}
\put(271.0,507.0){\rule[-0.200pt]{4.818pt}{0.400pt}}
\put(361.0,484.0){\rule[-0.200pt]{0.400pt}{3.613pt}}
\put(351.0,484.0){\rule[-0.200pt]{4.818pt}{0.400pt}}
\put(351.0,499.0){\rule[-0.200pt]{4.818pt}{0.400pt}}
\put(440.0,469.0){\rule[-0.200pt]{0.400pt}{5.300pt}}
\put(430.0,469.0){\rule[-0.200pt]{4.818pt}{0.400pt}}
\put(430.0,491.0){\rule[-0.200pt]{4.818pt}{0.400pt}}
\put(520.0,449.0){\rule[-0.200pt]{0.400pt}{7.468pt}}
\put(510.0,449.0){\rule[-0.200pt]{4.818pt}{0.400pt}}
\put(510.0,480.0){\rule[-0.200pt]{4.818pt}{0.400pt}}
\put(600.0,431.0){\rule[-0.200pt]{0.400pt}{8.913pt}}
\put(590.0,431.0){\rule[-0.200pt]{4.818pt}{0.400pt}}
\put(590.0,468.0){\rule[-0.200pt]{4.818pt}{0.400pt}}
\put(680.0,373.0){\rule[-0.200pt]{0.400pt}{16.622pt}}
\put(670.0,373.0){\rule[-0.200pt]{4.818pt}{0.400pt}}
\put(670.0,442.0){\rule[-0.200pt]{4.818pt}{0.400pt}}
\put(759.0,341.0){\rule[-0.200pt]{0.400pt}{17.104pt}}
\put(749.0,341.0){\rule[-0.200pt]{4.818pt}{0.400pt}}
\put(749.0,412.0){\rule[-0.200pt]{4.818pt}{0.400pt}}
\put(839.0,302.0){\rule[-0.200pt]{0.400pt}{17.345pt}}
\put(829.0,302.0){\rule[-0.200pt]{4.818pt}{0.400pt}}
\put(201,507){\circle*{12}}
\put(281,503){\circle*{12}}
\put(361,492){\circle*{12}}
\put(440,480){\circle*{12}}
\put(520,465){\circle*{12}}
\put(600,449){\circle*{12}}
\put(680,407){\circle*{12}}
\put(759,376){\circle*{12}}
\put(839,338){\circle*{12}}
\put(591,205){\circle*{12}}
\put(829.0,374.0){\rule[-0.200pt]{4.818pt}{0.400pt}}
\put(521,164){\makebox(0,0)[r]{MSE $< 10^{-5}$}}
\put(541.0,164.0){\rule[-0.200pt]{24.090pt}{0.400pt}}
\put(541.0,154.0){\rule[-0.200pt]{0.400pt}{4.818pt}}
\put(641.0,154.0){\rule[-0.200pt]{0.400pt}{4.818pt}}
\put(201,507){\usebox{\plotpoint}}
\multiput(201.00,505.92)(0.801,-0.498){97}{\rule{0.740pt}{0.120pt}}
\multiput(201.00,506.17)(78.464,-50.000){2}{\rule{0.370pt}{0.400pt}}
\multiput(281.00,455.92)(2.720,-0.494){27}{\rule{2.233pt}{0.119pt}}
\multiput(281.00,456.17)(75.365,-15.000){2}{\rule{1.117pt}{0.400pt}}
\multiput(361.00,440.92)(1.735,-0.496){43}{\rule{1.474pt}{0.120pt}}
\multiput(361.00,441.17)(75.941,-23.000){2}{\rule{0.737pt}{0.400pt}}
\multiput(440.00,417.93)(5.248,-0.488){13}{\rule{4.100pt}{0.117pt}}
\multiput(440.00,418.17)(71.490,-8.000){2}{\rule{2.050pt}{0.400pt}}
\multiput(600.00,409.92)(0.872,-0.498){89}{\rule{0.796pt}{0.120pt}}
\multiput(600.00,410.17)(78.349,-46.000){2}{\rule{0.398pt}{0.400pt}}
\multiput(680.00,363.92)(1.134,-0.498){67}{\rule{1.003pt}{0.120pt}}
\multiput(680.00,364.17)(76.919,-35.000){2}{\rule{0.501pt}{0.400pt}}
\multiput(759.00,328.92)(3.750,-0.492){19}{\rule{3.009pt}{0.118pt}}
\multiput(759.00,329.17)(73.754,-11.000){2}{\rule{1.505pt}{0.400pt}}
\put(520.0,411.0){\rule[-0.200pt]{19.272pt}{0.400pt}}
\put(201,507){\usebox{\plotpoint}}
\put(191.0,507.0){\rule[-0.200pt]{4.818pt}{0.400pt}}
\put(191.0,507.0){\rule[-0.200pt]{4.818pt}{0.400pt}}
\put(281.0,443.0){\rule[-0.200pt]{0.400pt}{6.745pt}}
\put(271.0,443.0){\rule[-0.200pt]{4.818pt}{0.400pt}}
\put(271.0,471.0){\rule[-0.200pt]{4.818pt}{0.400pt}}
\put(361.0,425.0){\rule[-0.200pt]{0.400pt}{8.191pt}}
\put(351.0,425.0){\rule[-0.200pt]{4.818pt}{0.400pt}}
\put(351.0,459.0){\rule[-0.200pt]{4.818pt}{0.400pt}}
\put(440.0,399.0){\rule[-0.200pt]{0.400pt}{9.395pt}}
\put(430.0,399.0){\rule[-0.200pt]{4.818pt}{0.400pt}}
\put(430.0,438.0){\rule[-0.200pt]{4.818pt}{0.400pt}}
\put(520.0,390.0){\rule[-0.200pt]{0.400pt}{10.118pt}}
\put(510.0,390.0){\rule[-0.200pt]{4.818pt}{0.400pt}}
\put(510.0,432.0){\rule[-0.200pt]{4.818pt}{0.400pt}}
\put(600.0,389.0){\rule[-0.200pt]{0.400pt}{10.600pt}}
\put(590.0,389.0){\rule[-0.200pt]{4.818pt}{0.400pt}}
\put(590.0,433.0){\rule[-0.200pt]{4.818pt}{0.400pt}}
\put(680.0,327.0){\rule[-0.200pt]{0.400pt}{18.308pt}}
\put(670.0,327.0){\rule[-0.200pt]{4.818pt}{0.400pt}}
\put(670.0,403.0){\rule[-0.200pt]{4.818pt}{0.400pt}}
\put(759.0,293.0){\rule[-0.200pt]{0.400pt}{18.067pt}}
\put(749.0,293.0){\rule[-0.200pt]{4.818pt}{0.400pt}}
\put(749.0,368.0){\rule[-0.200pt]{4.818pt}{0.400pt}}
\put(839.0,281.0){\rule[-0.200pt]{0.400pt}{18.308pt}}
\put(829.0,281.0){\rule[-0.200pt]{4.818pt}{0.400pt}}
\put(201,507){\circle{12}}
\put(281,457){\circle{12}}
\put(361,442){\circle{12}}
\put(440,419){\circle{12}}
\put(520,411){\circle{12}}
\put(600,411){\circle{12}}
\put(680,365){\circle{12}}
\put(759,330){\circle{12}}
\put(839,319){\circle{12}}
\put(591,164){\circle{12}}
\put(829.0,357.0){\rule[-0.200pt]{4.818pt}{0.400pt}}
\put(201.0,123.0){\rule[-0.200pt]{153.694pt}{0.400pt}}
\put(839.0,123.0){\rule[-0.200pt]{0.400pt}{92.506pt}}
\put(201.0,507.0){\rule[-0.200pt]{153.694pt}{0.400pt}}
\put(201.0,123.0){\rule[-0.200pt]{0.400pt}{92.506pt}}
\end{picture}

%% file: plot_classm40.tex
\setlength{\unitlength}{0.240900pt}
\ifx\plotpoint\undefined\newsavebox{\plotpoint}\fi
\sbox{\plotpoint}{\rule[-0.200pt]{0.400pt}{0.400pt}}%
\begin{picture}(900,629)(0,0)
\sbox{\plotpoint}{\rule[-0.200pt]{0.400pt}{0.400pt}}%
\put(201.0,123.0){\rule[-0.200pt]{4.818pt}{0.400pt}}
\put(181,123){\makebox(0,0)[r]{ 0.8}}
\put(819.0,123.0){\rule[-0.200pt]{4.818pt}{0.400pt}}
\put(201.0,161.0){\rule[-0.200pt]{4.818pt}{0.400pt}}
\put(181,161){\makebox(0,0)[r]{ 0.82}}
\put(819.0,161.0){\rule[-0.200pt]{4.818pt}{0.400pt}}
\put(201.0,200.0){\rule[-0.200pt]{4.818pt}{0.400pt}}
\put(181,200){\makebox(0,0)[r]{ 0.84}}
\put(819.0,200.0){\rule[-0.200pt]{4.818pt}{0.400pt}}
\put(201.0,238.0){\rule[-0.200pt]{4.818pt}{0.400pt}}
\put(181,238){\makebox(0,0)[r]{ 0.86}}
\put(819.0,238.0){\rule[-0.200pt]{4.818pt}{0.400pt}}
\put(201.0,277.0){\rule[-0.200pt]{4.818pt}{0.400pt}}
\put(181,277){\makebox(0,0)[r]{ 0.88}}
\put(819.0,277.0){\rule[-0.200pt]{4.818pt}{0.400pt}}
\put(201.0,315.0){\rule[-0.200pt]{4.818pt}{0.400pt}}
\put(181,315){\makebox(0,0)[r]{ 0.9}}
\put(819.0,315.0){\rule[-0.200pt]{4.818pt}{0.400pt}}
\put(201.0,353.0){\rule[-0.200pt]{4.818pt}{0.400pt}}
\put(181,353){\makebox(0,0)[r]{ 0.92}}
\put(819.0,353.0){\rule[-0.200pt]{4.818pt}{0.400pt}}
\put(201.0,392.0){\rule[-0.200pt]{4.818pt}{0.400pt}}
\put(181,392){\makebox(0,0)[r]{ 0.94}}
\put(819.0,392.0){\rule[-0.200pt]{4.818pt}{0.400pt}}
\put(201.0,430.0){\rule[-0.200pt]{4.818pt}{0.400pt}}
\put(181,430){\makebox(0,0)[r]{ 0.96}}
\put(819.0,430.0){\rule[-0.200pt]{4.818pt}{0.400pt}}
\put(201.0,469.0){\rule[-0.200pt]{4.818pt}{0.400pt}}
\put(181,469){\makebox(0,0)[r]{ 0.98}}
\put(819.0,469.0){\rule[-0.200pt]{4.818pt}{0.400pt}}
\put(201.0,507.0){\rule[-0.200pt]{4.818pt}{0.400pt}}
\put(181,507){\makebox(0,0)[r]{ 1}}
\put(819.0,507.0){\rule[-0.200pt]{4.818pt}{0.400pt}}
\put(201.0,123.0){\rule[-0.200pt]{0.400pt}{4.818pt}}
\put(201,82){\makebox(0,0){ 0}}
\put(201.0,487.0){\rule[-0.200pt]{0.400pt}{4.818pt}}
\put(281.0,123.0){\rule[-0.200pt]{0.400pt}{4.818pt}}
\put(281,82){\makebox(0,0){ 0.5}}
\put(281.0,487.0){\rule[-0.200pt]{0.400pt}{4.818pt}}
\put(361.0,123.0){\rule[-0.200pt]{0.400pt}{4.818pt}}
\put(361,82){\makebox(0,0){ 1}}
\put(361.0,487.0){\rule[-0.200pt]{0.400pt}{4.818pt}}
\put(440.0,123.0){\rule[-0.200pt]{0.400pt}{4.818pt}}
\put(440,82){\makebox(0,0){ 1.5}}
\put(440.0,487.0){\rule[-0.200pt]{0.400pt}{4.818pt}}
\put(520.0,123.0){\rule[-0.200pt]{0.400pt}{4.818pt}}
\put(520,82){\makebox(0,0){ 2}}
\put(520.0,487.0){\rule[-0.200pt]{0.400pt}{4.818pt}}
\put(600.0,123.0){\rule[-0.200pt]{0.400pt}{4.818pt}}
\put(600,82){\makebox(0,0){ 2.5}}
\put(600.0,487.0){\rule[-0.200pt]{0.400pt}{4.818pt}}
\put(680.0,123.0){\rule[-0.200pt]{0.400pt}{4.818pt}}
\put(680,82){\makebox(0,0){ 3}}
\put(680.0,487.0){\rule[-0.200pt]{0.400pt}{4.818pt}}
\put(759.0,123.0){\rule[-0.200pt]{0.400pt}{4.818pt}}
\put(759,82){\makebox(0,0){ 3.5}}
\put(759.0,487.0){\rule[-0.200pt]{0.400pt}{4.818pt}}
\put(839.0,123.0){\rule[-0.200pt]{0.400pt}{4.818pt}}
\put(839,82){\makebox(0,0){ 4}}
\put(839.0,487.0){\rule[-0.200pt]{0.400pt}{4.818pt}}
\put(201.0,123.0){\rule[-0.200pt]{153.694pt}{0.400pt}}
\put(839.0,123.0){\rule[-0.200pt]{0.400pt}{92.506pt}}
\put(201.0,507.0){\rule[-0.200pt]{153.694pt}{0.400pt}}
\put(201.0,123.0){\rule[-0.200pt]{0.400pt}{92.506pt}}
\put(40,315){\makebox(0,0){\rotatebox{90}{matching accuracy}}}
\put(520,21){\makebox(0,0){$256\epsilon$}}
\put(520,569){\makebox(0,0){$|\mathcal S|= 40$}}
\put(521,205){\makebox(0,0)[r]{MSE $< 10^{-9}$}}
\put(541.0,205.0){\rule[-0.200pt]{24.090pt}{0.400pt}}
\put(541.0,195.0){\rule[-0.200pt]{0.400pt}{4.818pt}}
\put(641.0,195.0){\rule[-0.200pt]{0.400pt}{4.818pt}}
\put(201,507){\usebox{\plotpoint}}
\multiput(201.00,505.93)(5.248,-0.488){13}{\rule{4.100pt}{0.117pt}}
\multiput(201.00,506.17)(71.490,-8.000){2}{\rule{2.050pt}{0.400pt}}
\multiput(281.00,497.93)(6.061,-0.485){11}{\rule{4.671pt}{0.117pt}}
\multiput(281.00,498.17)(70.304,-7.000){2}{\rule{2.336pt}{0.400pt}}
\multiput(361.00,490.92)(2.514,-0.494){29}{\rule{2.075pt}{0.119pt}}
\multiput(361.00,491.17)(74.693,-16.000){2}{\rule{1.038pt}{0.400pt}}
\multiput(440.00,474.92)(1.493,-0.497){51}{\rule{1.285pt}{0.120pt}}
\multiput(440.00,475.17)(77.333,-27.000){2}{\rule{0.643pt}{0.400pt}}
\multiput(520.00,447.92)(1.552,-0.497){49}{\rule{1.331pt}{0.120pt}}
\multiput(520.00,448.17)(77.238,-26.000){2}{\rule{0.665pt}{0.400pt}}
\multiput(600.00,421.92)(2.026,-0.496){37}{\rule{1.700pt}{0.119pt}}
\multiput(600.00,422.17)(76.472,-20.000){2}{\rule{0.850pt}{0.400pt}}
\multiput(680.00,401.92)(0.746,-0.498){103}{\rule{0.696pt}{0.120pt}}
\multiput(680.00,402.17)(77.555,-53.000){2}{\rule{0.348pt}{0.400pt}}
\multiput(759.00,348.92)(0.571,-0.499){137}{\rule{0.557pt}{0.120pt}}
\multiput(759.00,349.17)(78.844,-70.000){2}{\rule{0.279pt}{0.400pt}}
\put(201,507){\usebox{\plotpoint}}
\put(191.0,507.0){\rule[-0.200pt]{4.818pt}{0.400pt}}
\put(191.0,507.0){\rule[-0.200pt]{4.818pt}{0.400pt}}
\put(281.0,494.0){\rule[-0.200pt]{0.400pt}{2.650pt}}
\put(271.0,494.0){\rule[-0.200pt]{4.818pt}{0.400pt}}
\put(271.0,505.0){\rule[-0.200pt]{4.818pt}{0.400pt}}
\put(361.0,484.0){\rule[-0.200pt]{0.400pt}{3.613pt}}
\put(351.0,484.0){\rule[-0.200pt]{4.818pt}{0.400pt}}
\put(351.0,499.0){\rule[-0.200pt]{4.818pt}{0.400pt}}
\put(440.0,465.0){\rule[-0.200pt]{0.400pt}{5.541pt}}
\put(430.0,465.0){\rule[-0.200pt]{4.818pt}{0.400pt}}
\put(430.0,488.0){\rule[-0.200pt]{4.818pt}{0.400pt}}
\put(520.0,433.0){\rule[-0.200pt]{0.400pt}{7.950pt}}
\put(510.0,433.0){\rule[-0.200pt]{4.818pt}{0.400pt}}
\put(510.0,466.0){\rule[-0.200pt]{4.818pt}{0.400pt}}
\put(600.0,402.0){\rule[-0.200pt]{0.400pt}{9.877pt}}
\put(590.0,402.0){\rule[-0.200pt]{4.818pt}{0.400pt}}
\put(590.0,443.0){\rule[-0.200pt]{4.818pt}{0.400pt}}
\put(680.0,382.0){\rule[-0.200pt]{0.400pt}{10.118pt}}
\put(670.0,382.0){\rule[-0.200pt]{4.818pt}{0.400pt}}
\put(670.0,424.0){\rule[-0.200pt]{4.818pt}{0.400pt}}
\put(759.0,324.0){\rule[-0.200pt]{0.400pt}{12.286pt}}
\put(749.0,324.0){\rule[-0.200pt]{4.818pt}{0.400pt}}
\put(749.0,375.0){\rule[-0.200pt]{4.818pt}{0.400pt}}
\put(839.0,246.0){\rule[-0.200pt]{0.400pt}{16.622pt}}
\put(829.0,246.0){\rule[-0.200pt]{4.818pt}{0.400pt}}
\put(201,507){\circle*{12}}
\put(281,499){\circle*{12}}
\put(361,492){\circle*{12}}
\put(440,476){\circle*{12}}
\put(520,449){\circle*{12}}
\put(600,423){\circle*{12}}
\put(680,403){\circle*{12}}
\put(759,350){\circle*{12}}
\put(839,280){\circle*{12}}
\put(591,205){\circle*{12}}
\put(829.0,315.0){\rule[-0.200pt]{4.818pt}{0.400pt}}
\put(521,164){\makebox(0,0)[r]{MSE $< 10^{-5}$}}
\put(541.0,164.0){\rule[-0.200pt]{24.090pt}{0.400pt}}
\put(541.0,154.0){\rule[-0.200pt]{0.400pt}{4.818pt}}
\put(641.0,154.0){\rule[-0.200pt]{0.400pt}{4.818pt}}
\put(201,507){\usebox{\plotpoint}}
\multiput(201.00,505.93)(5.248,-0.488){13}{\rule{4.100pt}{0.117pt}}
\multiput(201.00,506.17)(71.490,-8.000){2}{\rule{2.050pt}{0.400pt}}
\multiput(281.00,497.92)(3.750,-0.492){19}{\rule{3.009pt}{0.118pt}}
\multiput(281.00,498.17)(73.754,-11.000){2}{\rule{1.505pt}{0.400pt}}
\multiput(361.00,486.92)(1.282,-0.497){59}{\rule{1.119pt}{0.120pt}}
\multiput(361.00,487.17)(76.677,-31.000){2}{\rule{0.560pt}{0.400pt}}
\multiput(440.00,455.92)(2.135,-0.495){35}{\rule{1.784pt}{0.119pt}}
\multiput(440.00,456.17)(76.297,-19.000){2}{\rule{0.892pt}{0.400pt}}
\multiput(520.00,436.92)(2.135,-0.495){35}{\rule{1.784pt}{0.119pt}}
\multiput(520.00,437.17)(76.297,-19.000){2}{\rule{0.892pt}{0.400pt}}
\multiput(600.00,417.92)(1.493,-0.497){51}{\rule{1.285pt}{0.120pt}}
\multiput(600.00,418.17)(77.333,-27.000){2}{\rule{0.643pt}{0.400pt}}
\multiput(680.00,390.92)(0.637,-0.499){121}{\rule{0.610pt}{0.120pt}}
\multiput(680.00,391.17)(77.735,-62.000){2}{\rule{0.305pt}{0.400pt}}
\multiput(759.00,328.92)(0.656,-0.499){119}{\rule{0.625pt}{0.120pt}}
\multiput(759.00,329.17)(78.704,-61.000){2}{\rule{0.312pt}{0.400pt}}
\put(201,507){\usebox{\plotpoint}}
\put(191.0,507.0){\rule[-0.200pt]{4.818pt}{0.400pt}}
\put(191.0,507.0){\rule[-0.200pt]{4.818pt}{0.400pt}}
\put(281.0,494.0){\rule[-0.200pt]{0.400pt}{2.650pt}}
\put(271.0,494.0){\rule[-0.200pt]{4.818pt}{0.400pt}}
\put(271.0,505.0){\rule[-0.200pt]{4.818pt}{0.400pt}}
\put(361.0,478.0){\rule[-0.200pt]{0.400pt}{4.818pt}}
\put(351.0,478.0){\rule[-0.200pt]{4.818pt}{0.400pt}}
\put(351.0,498.0){\rule[-0.200pt]{4.818pt}{0.400pt}}
\put(440.0,441.0){\rule[-0.200pt]{0.400pt}{7.709pt}}
\put(430.0,441.0){\rule[-0.200pt]{4.818pt}{0.400pt}}
\put(430.0,473.0){\rule[-0.200pt]{4.818pt}{0.400pt}}
\put(520.0,419.0){\rule[-0.200pt]{0.400pt}{9.154pt}}
\put(510.0,419.0){\rule[-0.200pt]{4.818pt}{0.400pt}}
\put(510.0,457.0){\rule[-0.200pt]{4.818pt}{0.400pt}}
\put(600.0,397.0){\rule[-0.200pt]{0.400pt}{10.600pt}}
\put(590.0,397.0){\rule[-0.200pt]{4.818pt}{0.400pt}}
\put(590.0,441.0){\rule[-0.200pt]{4.818pt}{0.400pt}}
\put(680.0,368.0){\rule[-0.200pt]{0.400pt}{11.563pt}}
\put(670.0,368.0){\rule[-0.200pt]{4.818pt}{0.400pt}}
\put(670.0,416.0){\rule[-0.200pt]{4.818pt}{0.400pt}}
\put(759.0,302.0){\rule[-0.200pt]{0.400pt}{13.731pt}}
\put(749.0,302.0){\rule[-0.200pt]{4.818pt}{0.400pt}}
\put(749.0,359.0){\rule[-0.200pt]{4.818pt}{0.400pt}}
\put(839.0,233.0){\rule[-0.200pt]{0.400pt}{17.345pt}}
\put(829.0,233.0){\rule[-0.200pt]{4.818pt}{0.400pt}}
\put(201,507){\circle{12}}
\put(281,499){\circle{12}}
\put(361,488){\circle{12}}
\put(440,457){\circle{12}}
\put(520,438){\circle{12}}
\put(600,419){\circle{12}}
\put(680,392){\circle{12}}
\put(759,330){\circle{12}}
\put(839,269){\circle{12}}
\put(591,164){\circle{12}}
\put(829.0,305.0){\rule[-0.200pt]{4.818pt}{0.400pt}}
\put(201.0,123.0){\rule[-0.200pt]{153.694pt}{0.400pt}}
\put(839.0,123.0){\rule[-0.200pt]{0.400pt}{92.506pt}}
\put(201.0,507.0){\rule[-0.200pt]{153.694pt}{0.400pt}}
\put(201.0,123.0){\rule[-0.200pt]{0.400pt}{92.506pt}}
\end{picture}

%% file: plot_house15.tex
\setlength{\unitlength}{0.240900pt}
\ifx\plotpoint\undefined\newsavebox{\plotpoint}\fi
\sbox{\plotpoint}{\rule[-0.200pt]{0.400pt}{0.400pt}}%
\begin{picture}(900,584)(0,0)
\sbox{\plotpoint}{\rule[-0.200pt]{0.400pt}{0.400pt}}%
\put(181.0,123.0){\rule[-0.200pt]{4.818pt}{0.400pt}}
\put(161,123){\makebox(0,0)[r]{ 0}}
\put(819.0,123.0){\rule[-0.200pt]{4.818pt}{0.400pt}}
\put(181.0,191.0){\rule[-0.200pt]{4.818pt}{0.400pt}}
\put(161,191){\makebox(0,0)[r]{ 0.2}}
\put(819.0,191.0){\rule[-0.200pt]{4.818pt}{0.400pt}}
\put(181.0,259.0){\rule[-0.200pt]{4.818pt}{0.400pt}}
\put(161,259){\makebox(0,0)[r]{ 0.4}}
\put(819.0,259.0){\rule[-0.200pt]{4.818pt}{0.400pt}}
\put(181.0,326.0){\rule[-0.200pt]{4.818pt}{0.400pt}}
\put(161,326){\makebox(0,0)[r]{ 0.6}}
\put(819.0,326.0){\rule[-0.200pt]{4.818pt}{0.400pt}}
\put(181.0,394.0){\rule[-0.200pt]{4.818pt}{0.400pt}}
\put(161,394){\makebox(0,0)[r]{ 0.8}}
\put(819.0,394.0){\rule[-0.200pt]{4.818pt}{0.400pt}}
\put(181.0,462.0){\rule[-0.200pt]{4.818pt}{0.400pt}}
\put(161,462){\makebox(0,0)[r]{ 1}}
\put(819.0,462.0){\rule[-0.200pt]{4.818pt}{0.400pt}}
\put(181.0,123.0){\rule[-0.200pt]{0.400pt}{4.818pt}}
\put(181,82){\makebox(0,0){ 0}}
\put(181.0,442.0){\rule[-0.200pt]{0.400pt}{4.818pt}}
\put(313.0,123.0){\rule[-0.200pt]{0.400pt}{4.818pt}}
\put(313,82){\makebox(0,0){ 20}}
\put(313.0,442.0){\rule[-0.200pt]{0.400pt}{4.818pt}}
\put(444.0,123.0){\rule[-0.200pt]{0.400pt}{4.818pt}}
\put(444,82){\makebox(0,0){ 40}}
\put(444.0,442.0){\rule[-0.200pt]{0.400pt}{4.818pt}}
\put(576.0,123.0){\rule[-0.200pt]{0.400pt}{4.818pt}}
\put(576,82){\makebox(0,0){ 60}}
\put(576.0,442.0){\rule[-0.200pt]{0.400pt}{4.818pt}}
\put(707.0,123.0){\rule[-0.200pt]{0.400pt}{4.818pt}}
\put(707,82){\makebox(0,0){ 80}}
\put(707.0,442.0){\rule[-0.200pt]{0.400pt}{4.818pt}}
\put(839.0,123.0){\rule[-0.200pt]{0.400pt}{4.818pt}}
\put(839,82){\makebox(0,0){ 100}}
\put(839.0,442.0){\rule[-0.200pt]{0.400pt}{4.818pt}}
\put(181.0,123.0){\rule[-0.200pt]{0.400pt}{81.665pt}}
\put(181.0,123.0){\rule[-0.200pt]{158.512pt}{0.400pt}}
\put(839.0,123.0){\rule[-0.200pt]{0.400pt}{81.665pt}}
\put(181.0,462.0){\rule[-0.200pt]{158.512pt}{0.400pt}}
\put(40,292){\makebox(0,0){\rotatebox{90}{matching accuracy}}}
\put(510,21){\makebox(0,0){baseline}}
\put(510,524){\makebox(0,0){$|\mathcal T|= 15$, MSE $< 10^{-9}$}}
\put(261,205){\makebox(0,0)[r]{JT}}
\multiput(281,205)(20.756,0.000){5}{\usebox{\plotpoint}}
\put(381,205){\usebox{\plotpoint}}
\put(281.00,215.00){\usebox{\plotpoint}}
\put(281,195){\usebox{\plotpoint}}
\put(381.00,215.00){\usebox{\plotpoint}}
\put(381,195){\usebox{\plotpoint}}
\put(181,462){\usebox{\plotpoint}}
\multiput(181,462)(18.786,-8.824){4}{\usebox{\plotpoint}}
\multiput(247,431)(17.390,-11.330){4}{\usebox{\plotpoint}}
\multiput(313,388)(15.725,-13.547){4}{\usebox{\plotpoint}}
\multiput(378,332)(18.451,-9.505){4}{\usebox{\plotpoint}}
\multiput(444,298)(19.107,-8.106){3}{\usebox{\plotpoint}}
\multiput(510,270)(18.895,-8.589){3}{\usebox{\plotpoint}}
\multiput(576,240)(17.390,-11.330){4}{\usebox{\plotpoint}}
\multiput(642,197)(17.677,-10.878){4}{\usebox{\plotpoint}}
\multiput(707,157)(20.717,1.256){3}{\usebox{\plotpoint}}
\multiput(773,161)(18.895,8.589){4}{\usebox{\plotpoint}}
\put(839,191){\usebox{\plotpoint}}
\put(181,462){\usebox{\plotpoint}}
\put(181,462){\usebox{\plotpoint}}
\put(171.00,462.00){\usebox{\plotpoint}}
\put(191,462){\usebox{\plotpoint}}
\put(171.00,462.00){\usebox{\plotpoint}}
\put(191,462){\usebox{\plotpoint}}
\put(247.00,428.00){\usebox{\plotpoint}}
\put(247,435){\usebox{\plotpoint}}
\put(237.00,428.00){\usebox{\plotpoint}}
\put(257,428){\usebox{\plotpoint}}
\put(237.00,435.00){\usebox{\plotpoint}}
\put(257,435){\usebox{\plotpoint}}
\put(313.00,383.00){\usebox{\plotpoint}}
\put(313,394){\usebox{\plotpoint}}
\put(303.00,383.00){\usebox{\plotpoint}}
\put(323,383){\usebox{\plotpoint}}
\put(303.00,394.00){\usebox{\plotpoint}}
\put(323,394){\usebox{\plotpoint}}
\put(378.00,325.00){\usebox{\plotpoint}}
\put(378,338){\usebox{\plotpoint}}
\put(368.00,325.00){\usebox{\plotpoint}}
\put(388,325){\usebox{\plotpoint}}
\put(368.00,338.00){\usebox{\plotpoint}}
\put(388,338){\usebox{\plotpoint}}
\put(444.00,292.00){\usebox{\plotpoint}}
\put(444,305){\usebox{\plotpoint}}
\put(434.00,292.00){\usebox{\plotpoint}}
\put(454,292){\usebox{\plotpoint}}
\put(434.00,305.00){\usebox{\plotpoint}}
\put(454,305){\usebox{\plotpoint}}
\put(510.00,263.00){\usebox{\plotpoint}}
\put(510,277){\usebox{\plotpoint}}
\put(500.00,263.00){\usebox{\plotpoint}}
\put(520,263){\usebox{\plotpoint}}
\put(500.00,277.00){\usebox{\plotpoint}}
\put(520,277){\usebox{\plotpoint}}
\put(576.00,231.00){\usebox{\plotpoint}}
\put(576,248){\usebox{\plotpoint}}
\put(566.00,231.00){\usebox{\plotpoint}}
\put(586,231){\usebox{\plotpoint}}
\put(566.00,248.00){\usebox{\plotpoint}}
\put(586,248){\usebox{\plotpoint}}
\put(642.00,188.00){\usebox{\plotpoint}}
\put(642,207){\usebox{\plotpoint}}
\put(632.00,188.00){\usebox{\plotpoint}}
\put(652,188){\usebox{\plotpoint}}
\put(632.00,207.00){\usebox{\plotpoint}}
\put(652,207){\usebox{\plotpoint}}
\put(707.00,150.00){\usebox{\plotpoint}}
\put(707,164){\usebox{\plotpoint}}
\put(697.00,150.00){\usebox{\plotpoint}}
\put(717,150){\usebox{\plotpoint}}
\put(697.00,164.00){\usebox{\plotpoint}}
\put(717,164){\usebox{\plotpoint}}
\put(773.00,153.00){\usebox{\plotpoint}}
\put(773,169){\usebox{\plotpoint}}
\put(763.00,153.00){\usebox{\plotpoint}}
\put(783,153){\usebox{\plotpoint}}
\put(763.00,169.00){\usebox{\plotpoint}}
\put(783,169){\usebox{\plotpoint}}
\put(839.00,183.00){\usebox{\plotpoint}}
\put(839,198){\usebox{\plotpoint}}
\put(829.00,183.00){\usebox{\plotpoint}}
\put(849,183){\usebox{\plotpoint}}
\put(829.00,198.00){\usebox{\plotpoint}}
\put(849,198){\usebox{\plotpoint}}
\put(181,462){\circle{12}}
\put(247,431){\circle{12}}
\put(313,388){\circle{12}}
\put(378,332){\circle{12}}
\put(444,298){\circle{12}}
\put(510,270){\circle{12}}
\put(576,240){\circle{12}}
\put(642,197){\circle{12}}
\put(707,157){\circle{12}}
\put(773,161){\circle{12}}
\put(839,191){\circle{12}}
\put(331,205){\circle{12}}
\sbox{\plotpoint}{\rule[-0.400pt]{0.800pt}{0.800pt}}%
\sbox{\plotpoint}{\rule[-0.200pt]{0.400pt}{0.400pt}}%
\put(261,164){\makebox(0,0)[r]{LBP}}
\sbox{\plotpoint}{\rule[-0.400pt]{0.800pt}{0.800pt}}%
\put(281.0,164.0){\rule[-0.400pt]{24.090pt}{0.800pt}}
\put(281.0,154.0){\rule[-0.400pt]{0.800pt}{4.818pt}}
\put(381.0,154.0){\rule[-0.400pt]{0.800pt}{4.818pt}}
\put(181,462){\usebox{\plotpoint}}
\put(181,459.84){\rule{15.899pt}{0.800pt}}
\multiput(181.00,460.34)(33.000,-1.000){2}{\rule{7.950pt}{0.800pt}}
\multiput(247.00,459.08)(3.604,-0.514){13}{\rule{5.480pt}{0.124pt}}
\multiput(247.00,459.34)(54.626,-10.000){2}{\rule{2.740pt}{0.800pt}}
\multiput(313.00,449.09)(1.025,-0.503){57}{\rule{1.825pt}{0.121pt}}
\multiput(313.00,449.34)(61.212,-32.000){2}{\rule{0.913pt}{0.800pt}}
\multiput(378.00,417.09)(1.694,-0.505){33}{\rule{2.840pt}{0.122pt}}
\multiput(378.00,417.34)(60.105,-20.000){2}{\rule{1.420pt}{0.800pt}}
\multiput(444.00,397.09)(0.549,-0.502){113}{\rule{1.080pt}{0.121pt}}
\multiput(444.00,397.34)(63.758,-60.000){2}{\rule{0.540pt}{0.800pt}}
\multiput(511.41,334.55)(0.501,-0.545){125}{\rule{0.121pt}{1.073pt}}
\multiput(508.34,336.77)(66.000,-69.774){2}{\rule{0.800pt}{0.536pt}}
\multiput(576.00,265.09)(1.401,-0.504){41}{\rule{2.400pt}{0.122pt}}
\multiput(576.00,265.34)(61.019,-24.000){2}{\rule{1.200pt}{0.800pt}}
\multiput(642.00,241.09)(2.112,-0.507){25}{\rule{3.450pt}{0.122pt}}
\multiput(642.00,241.34)(57.839,-16.000){2}{\rule{1.725pt}{0.800pt}}
\multiput(707.00,225.09)(1.893,-0.506){29}{\rule{3.133pt}{0.122pt}}
\multiput(707.00,225.34)(59.497,-18.000){2}{\rule{1.567pt}{0.800pt}}
\put(773,206.34){\rule{15.899pt}{0.800pt}}
\multiput(773.00,207.34)(33.000,-2.000){2}{\rule{7.950pt}{0.800pt}}
\put(181,462){\usebox{\plotpoint}}
\put(171.0,462.0){\rule[-0.400pt]{4.818pt}{0.800pt}}
\put(171.0,462.0){\rule[-0.400pt]{4.818pt}{0.800pt}}
\put(247.0,461.0){\usebox{\plotpoint}}
\put(237.0,461.0){\rule[-0.400pt]{4.818pt}{0.800pt}}
\put(237.0,462.0){\rule[-0.400pt]{4.818pt}{0.800pt}}
\put(313.0,450.0){\usebox{\plotpoint}}
\put(303.0,450.0){\rule[-0.400pt]{4.818pt}{0.800pt}}
\put(303.0,452.0){\rule[-0.400pt]{4.818pt}{0.800pt}}
\put(378.0,413.0){\rule[-0.400pt]{0.800pt}{2.891pt}}
\put(368.0,413.0){\rule[-0.400pt]{4.818pt}{0.800pt}}
\put(368.0,425.0){\rule[-0.400pt]{4.818pt}{0.800pt}}
\put(444.0,392.0){\rule[-0.400pt]{0.800pt}{3.132pt}}
\put(434.0,392.0){\rule[-0.400pt]{4.818pt}{0.800pt}}
\put(434.0,405.0){\rule[-0.400pt]{4.818pt}{0.800pt}}
\put(510.0,329.0){\rule[-0.400pt]{0.800pt}{4.818pt}}
\put(500.0,329.0){\rule[-0.400pt]{4.818pt}{0.800pt}}
\put(500.0,349.0){\rule[-0.400pt]{4.818pt}{0.800pt}}
\put(576.0,253.0){\rule[-0.400pt]{0.800pt}{6.745pt}}
\put(566.0,253.0){\rule[-0.400pt]{4.818pt}{0.800pt}}
\put(566.0,281.0){\rule[-0.400pt]{4.818pt}{0.800pt}}
\put(642.0,226.0){\rule[-0.400pt]{0.800pt}{7.950pt}}
\put(632.0,226.0){\rule[-0.400pt]{4.818pt}{0.800pt}}
\put(632.0,259.0){\rule[-0.400pt]{4.818pt}{0.800pt}}
\put(707.0,208.0){\rule[-0.400pt]{0.800pt}{8.913pt}}
\put(697.0,208.0){\rule[-0.400pt]{4.818pt}{0.800pt}}
\put(697.0,245.0){\rule[-0.400pt]{4.818pt}{0.800pt}}
\put(773.0,191.0){\rule[-0.400pt]{0.800pt}{8.672pt}}
\put(763.0,191.0){\rule[-0.400pt]{4.818pt}{0.800pt}}
\put(763.0,227.0){\rule[-0.400pt]{4.818pt}{0.800pt}}
\put(839.0,185.0){\rule[-0.400pt]{0.800pt}{10.840pt}}
\put(829.0,185.0){\rule[-0.400pt]{4.818pt}{0.800pt}}
\put(181,462){\circle{18}}
\put(247,461){\circle{18}}
\put(313,451){\circle{18}}
\put(378,419){\circle{18}}
\put(444,399){\circle{18}}
\put(510,339){\circle{18}}
\put(576,267){\circle{18}}
\put(642,243){\circle{18}}
\put(707,227){\circle{18}}
\put(773,209){\circle{18}}
\put(839,207){\circle{18}}
\put(331,164){\circle{18}}
\put(829.0,230.0){\rule[-0.400pt]{4.818pt}{0.800pt}}
\sbox{\plotpoint}{\rule[-0.200pt]{0.400pt}{0.400pt}}%
\put(181.0,123.0){\rule[-0.200pt]{0.400pt}{81.665pt}}
\put(181.0,123.0){\rule[-0.200pt]{158.512pt}{0.400pt}}
\put(839.0,123.0){\rule[-0.200pt]{0.400pt}{81.665pt}}
\put(181.0,462.0){\rule[-0.200pt]{158.512pt}{0.400pt}}
\end{picture}

%% file: plot_house20.tex
\setlength{\unitlength}{0.240900pt}
\ifx\plotpoint\undefined\newsavebox{\plotpoint}\fi
\sbox{\plotpoint}{\rule[-0.200pt]{0.400pt}{0.400pt}}%
\begin{picture}(900,584)(0,0)
\sbox{\plotpoint}{\rule[-0.200pt]{0.400pt}{0.400pt}}%
\put(181.0,123.0){\rule[-0.200pt]{4.818pt}{0.400pt}}
\put(161,123){\makebox(0,0)[r]{ 0}}
\put(819.0,123.0){\rule[-0.200pt]{4.818pt}{0.400pt}}
\put(181.0,191.0){\rule[-0.200pt]{4.818pt}{0.400pt}}
\put(161,191){\makebox(0,0)[r]{ 0.2}}
\put(819.0,191.0){\rule[-0.200pt]{4.818pt}{0.400pt}}
\put(181.0,259.0){\rule[-0.200pt]{4.818pt}{0.400pt}}
\put(161,259){\makebox(0,0)[r]{ 0.4}}
\put(819.0,259.0){\rule[-0.200pt]{4.818pt}{0.400pt}}
\put(181.0,326.0){\rule[-0.200pt]{4.818pt}{0.400pt}}
\put(161,326){\makebox(0,0)[r]{ 0.6}}
\put(819.0,326.0){\rule[-0.200pt]{4.818pt}{0.400pt}}
\put(181.0,394.0){\rule[-0.200pt]{4.818pt}{0.400pt}}
\put(161,394){\makebox(0,0)[r]{ 0.8}}
\put(819.0,394.0){\rule[-0.200pt]{4.818pt}{0.400pt}}
\put(181.0,462.0){\rule[-0.200pt]{4.818pt}{0.400pt}}
\put(161,462){\makebox(0,0)[r]{ 1}}
\put(819.0,462.0){\rule[-0.200pt]{4.818pt}{0.400pt}}
\put(181.0,123.0){\rule[-0.200pt]{0.400pt}{4.818pt}}
\put(181,82){\makebox(0,0){ 0}}
\put(181.0,442.0){\rule[-0.200pt]{0.400pt}{4.818pt}}
\put(313.0,123.0){\rule[-0.200pt]{0.400pt}{4.818pt}}
\put(313,82){\makebox(0,0){ 20}}
\put(313.0,442.0){\rule[-0.200pt]{0.400pt}{4.818pt}}
\put(444.0,123.0){\rule[-0.200pt]{0.400pt}{4.818pt}}
\put(444,82){\makebox(0,0){ 40}}
\put(444.0,442.0){\rule[-0.200pt]{0.400pt}{4.818pt}}
\put(576.0,123.0){\rule[-0.200pt]{0.400pt}{4.818pt}}
\put(576,82){\makebox(0,0){ 60}}
\put(576.0,442.0){\rule[-0.200pt]{0.400pt}{4.818pt}}
\put(707.0,123.0){\rule[-0.200pt]{0.400pt}{4.818pt}}
\put(707,82){\makebox(0,0){ 80}}
\put(707.0,442.0){\rule[-0.200pt]{0.400pt}{4.818pt}}
\put(839.0,123.0){\rule[-0.200pt]{0.400pt}{4.818pt}}
\put(839,82){\makebox(0,0){ 100}}
\put(839.0,442.0){\rule[-0.200pt]{0.400pt}{4.818pt}}
\put(181.0,123.0){\rule[-0.200pt]{0.400pt}{81.665pt}}
\put(181.0,123.0){\rule[-0.200pt]{158.512pt}{0.400pt}}
\put(839.0,123.0){\rule[-0.200pt]{0.400pt}{81.665pt}}
\put(181.0,462.0){\rule[-0.200pt]{158.512pt}{0.400pt}}
\put(40,292){\makebox(0,0){\rotatebox{90}{matching accuracy}}}
\put(510,21){\makebox(0,0){baseline}}
\put(510,524){\makebox(0,0){$|\mathcal T|= 20$, MSE $< 10^{-9}$}}
\put(261,205){\makebox(0,0)[r]{JT}}
\multiput(281,205)(20.756,0.000){5}{\usebox{\plotpoint}}
\put(381,205){\usebox{\plotpoint}}
\put(281.00,215.00){\usebox{\plotpoint}}
\put(281,195){\usebox{\plotpoint}}
\put(381.00,215.00){\usebox{\plotpoint}}
\put(381,195){\usebox{\plotpoint}}
\put(181,462){\usebox{\plotpoint}}
\multiput(181,462)(19.599,-6.830){4}{\usebox{\plotpoint}}
\multiput(247,439)(17.149,-11.692){4}{\usebox{\plotpoint}}
\multiput(313,394)(13.593,-15.685){4}{\usebox{\plotpoint}}
\multiput(378,319)(15.474,-13.833){5}{\usebox{\plotpoint}}
\multiput(444,260)(18.786,-8.824){3}{\usebox{\plotpoint}}
\multiput(510,229)(19.945,-5.742){4}{\usebox{\plotpoint}}
\multiput(576,210)(18.786,-8.824){3}{\usebox{\plotpoint}}
\multiput(642,179)(19.062,-8.211){4}{\usebox{\plotpoint}}
\multiput(707,151)(20.753,-0.314){3}{\usebox{\plotpoint}}
\multiput(773,150)(20.640,2.189){3}{\usebox{\plotpoint}}
\put(839,157){\usebox{\plotpoint}}
\put(181,462){\usebox{\plotpoint}}
\put(181,462){\usebox{\plotpoint}}
\put(171.00,462.00){\usebox{\plotpoint}}
\put(191,462){\usebox{\plotpoint}}
\put(171.00,462.00){\usebox{\plotpoint}}
\put(191,462){\usebox{\plotpoint}}
\put(247.00,436.00){\usebox{\plotpoint}}
\put(247,441){\usebox{\plotpoint}}
\put(237.00,436.00){\usebox{\plotpoint}}
\put(257,436){\usebox{\plotpoint}}
\put(237.00,441.00){\usebox{\plotpoint}}
\put(257,441){\usebox{\plotpoint}}
\put(313.00,390.00){\usebox{\plotpoint}}
\put(313,399){\usebox{\plotpoint}}
\put(303.00,390.00){\usebox{\plotpoint}}
\put(323,390){\usebox{\plotpoint}}
\put(303.00,399.00){\usebox{\plotpoint}}
\put(323,399){\usebox{\plotpoint}}
\put(378.00,312.00){\usebox{\plotpoint}}
\put(378,326){\usebox{\plotpoint}}
\put(368.00,312.00){\usebox{\plotpoint}}
\put(388,312){\usebox{\plotpoint}}
\put(368.00,326.00){\usebox{\plotpoint}}
\put(388,326){\usebox{\plotpoint}}
\put(444.00,251.00){\usebox{\plotpoint}}
\put(444,268){\usebox{\plotpoint}}
\put(434.00,251.00){\usebox{\plotpoint}}
\put(454,251){\usebox{\plotpoint}}
\put(434.00,268.00){\usebox{\plotpoint}}
\put(454,268){\usebox{\plotpoint}}
\put(510.00,222.00){\usebox{\plotpoint}}
\put(510,236){\usebox{\plotpoint}}
\put(500.00,222.00){\usebox{\plotpoint}}
\put(520,222){\usebox{\plotpoint}}
\put(500.00,236.00){\usebox{\plotpoint}}
\put(520,236){\usebox{\plotpoint}}
\put(576.00,204.00){\usebox{\plotpoint}}
\put(576,217){\usebox{\plotpoint}}
\put(566.00,204.00){\usebox{\plotpoint}}
\put(586,204){\usebox{\plotpoint}}
\put(566.00,217.00){\usebox{\plotpoint}}
\put(586,217){\usebox{\plotpoint}}
\put(642.00,172.00){\usebox{\plotpoint}}
\put(642,186){\usebox{\plotpoint}}
\put(632.00,172.00){\usebox{\plotpoint}}
\put(652,172){\usebox{\plotpoint}}
\put(632.00,186.00){\usebox{\plotpoint}}
\put(652,186){\usebox{\plotpoint}}
\put(707.00,145.00){\usebox{\plotpoint}}
\put(707,156){\usebox{\plotpoint}}
\put(697.00,145.00){\usebox{\plotpoint}}
\put(717,145){\usebox{\plotpoint}}
\put(697.00,156.00){\usebox{\plotpoint}}
\put(717,156){\usebox{\plotpoint}}
\put(773.00,144.00){\usebox{\plotpoint}}
\put(773,155){\usebox{\plotpoint}}
\put(763.00,144.00){\usebox{\plotpoint}}
\put(783,144){\usebox{\plotpoint}}
\put(763.00,155.00){\usebox{\plotpoint}}
\put(783,155){\usebox{\plotpoint}}
\put(839.00,153.00){\usebox{\plotpoint}}
\put(839,161){\usebox{\plotpoint}}
\put(829.00,153.00){\usebox{\plotpoint}}
\put(849,153){\usebox{\plotpoint}}
\put(829.00,161.00){\usebox{\plotpoint}}
\put(849,161){\usebox{\plotpoint}}
\put(181,462){\circle{12}}
\put(247,439){\circle{12}}
\put(313,394){\circle{12}}
\put(378,319){\circle{12}}
\put(444,260){\circle{12}}
\put(510,229){\circle{12}}
\put(576,210){\circle{12}}
\put(642,179){\circle{12}}
\put(707,151){\circle{12}}
\put(773,150){\circle{12}}
\put(839,157){\circle{12}}
\put(331,205){\circle{12}}
\sbox{\plotpoint}{\rule[-0.400pt]{0.800pt}{0.800pt}}%
\sbox{\plotpoint}{\rule[-0.200pt]{0.400pt}{0.400pt}}%
\put(261,164){\makebox(0,0)[r]{LBP}}
\sbox{\plotpoint}{\rule[-0.400pt]{0.800pt}{0.800pt}}%
\put(281.0,164.0){\rule[-0.400pt]{24.090pt}{0.800pt}}
\put(281.0,154.0){\rule[-0.400pt]{0.800pt}{4.818pt}}
\put(381.0,154.0){\rule[-0.400pt]{0.800pt}{4.818pt}}
\put(181,462){\usebox{\plotpoint}}
\put(181,459.84){\rule{15.899pt}{0.800pt}}
\multiput(181.00,460.34)(33.000,-1.000){2}{\rule{7.950pt}{0.800pt}}
\put(247,458.84){\rule{15.899pt}{0.800pt}}
\multiput(247.00,459.34)(33.000,-1.000){2}{\rule{7.950pt}{0.800pt}}
\multiput(313.00,458.09)(1.510,-0.505){37}{\rule{2.564pt}{0.122pt}}
\multiput(313.00,458.34)(59.679,-22.000){2}{\rule{1.282pt}{0.800pt}}
\multiput(378.00,436.09)(0.874,-0.503){69}{\rule{1.589pt}{0.121pt}}
\multiput(378.00,436.34)(62.701,-38.000){2}{\rule{0.795pt}{0.800pt}}
\multiput(444.00,398.09)(0.735,-0.502){83}{\rule{1.373pt}{0.121pt}}
\multiput(444.00,398.34)(63.150,-45.000){2}{\rule{0.687pt}{0.800pt}}
\multiput(510.00,353.09)(0.789,-0.502){77}{\rule{1.457pt}{0.121pt}}
\multiput(510.00,353.34)(62.976,-42.000){2}{\rule{0.729pt}{0.800pt}}
\multiput(576.00,311.09)(0.507,-0.501){123}{\rule{1.012pt}{0.121pt}}
\multiput(576.00,311.34)(63.899,-65.000){2}{\rule{0.506pt}{0.800pt}}
\multiput(642.00,246.09)(0.560,-0.502){109}{\rule{1.097pt}{0.121pt}}
\multiput(642.00,246.34)(62.724,-58.000){2}{\rule{0.548pt}{0.800pt}}
\multiput(707.00,188.09)(0.704,-0.502){87}{\rule{1.323pt}{0.121pt}}
\multiput(707.00,188.34)(63.253,-47.000){2}{\rule{0.662pt}{0.800pt}}
\multiput(773.00,144.41)(1.041,0.503){57}{\rule{1.850pt}{0.121pt}}
\multiput(773.00,141.34)(62.160,32.000){2}{\rule{0.925pt}{0.800pt}}
\put(181,462){\usebox{\plotpoint}}
\put(171.0,462.0){\rule[-0.400pt]{4.818pt}{0.800pt}}
\put(171.0,462.0){\rule[-0.400pt]{4.818pt}{0.800pt}}
\put(247.0,461.0){\usebox{\plotpoint}}
\put(237.0,461.0){\rule[-0.400pt]{4.818pt}{0.800pt}}
\put(237.0,462.0){\rule[-0.400pt]{4.818pt}{0.800pt}}
\put(313.0,460.0){\usebox{\plotpoint}}
\put(303.0,460.0){\rule[-0.400pt]{4.818pt}{0.800pt}}
\put(303.0,461.0){\rule[-0.400pt]{4.818pt}{0.800pt}}
\put(378.0,436.0){\rule[-0.400pt]{0.800pt}{1.204pt}}
\put(368.0,436.0){\rule[-0.400pt]{4.818pt}{0.800pt}}
\put(368.0,441.0){\rule[-0.400pt]{4.818pt}{0.800pt}}
\put(444.0,397.0){\rule[-0.400pt]{0.800pt}{1.445pt}}
\put(434.0,397.0){\rule[-0.400pt]{4.818pt}{0.800pt}}
\put(434.0,403.0){\rule[-0.400pt]{4.818pt}{0.800pt}}
\put(510.0,350.0){\rule[-0.400pt]{0.800pt}{2.409pt}}
\put(500.0,350.0){\rule[-0.400pt]{4.818pt}{0.800pt}}
\put(500.0,360.0){\rule[-0.400pt]{4.818pt}{0.800pt}}
\put(576.0,306.0){\rule[-0.400pt]{0.800pt}{3.613pt}}
\put(566.0,306.0){\rule[-0.400pt]{4.818pt}{0.800pt}}
\put(566.0,321.0){\rule[-0.400pt]{4.818pt}{0.800pt}}
\put(642.0,238.0){\rule[-0.400pt]{0.800pt}{5.059pt}}
\put(632.0,238.0){\rule[-0.400pt]{4.818pt}{0.800pt}}
\put(632.0,259.0){\rule[-0.400pt]{4.818pt}{0.800pt}}
\put(707.0,180.0){\rule[-0.400pt]{0.800pt}{5.059pt}}
\put(697.0,180.0){\rule[-0.400pt]{4.818pt}{0.800pt}}
\put(697.0,201.0){\rule[-0.400pt]{4.818pt}{0.800pt}}
\put(773.0,135.0){\rule[-0.400pt]{0.800pt}{3.854pt}}
\put(763.0,135.0){\rule[-0.400pt]{4.818pt}{0.800pt}}
\put(763.0,151.0){\rule[-0.400pt]{4.818pt}{0.800pt}}
\put(839.0,165.0){\rule[-0.400pt]{0.800pt}{5.059pt}}
\put(829.0,165.0){\rule[-0.400pt]{4.818pt}{0.800pt}}
\put(181,462){\circle{18}}
\put(247,461){\circle{18}}
\put(313,460){\circle{18}}
\put(378,438){\circle{18}}
\put(444,400){\circle{18}}
\put(510,355){\circle{18}}
\put(576,313){\circle{18}}
\put(642,248){\circle{18}}
\put(707,190){\circle{18}}
\put(773,143){\circle{18}}
\put(839,175){\circle{18}}
\put(331,164){\circle{18}}
\put(829.0,186.0){\rule[-0.400pt]{4.818pt}{0.800pt}}
\sbox{\plotpoint}{\rule[-0.200pt]{0.400pt}{0.400pt}}%
\put(181.0,123.0){\rule[-0.200pt]{0.400pt}{81.665pt}}
\put(181.0,123.0){\rule[-0.200pt]{158.512pt}{0.400pt}}
\put(839.0,123.0){\rule[-0.200pt]{0.400pt}{81.665pt}}
\put(181.0,462.0){\rule[-0.200pt]{158.512pt}{0.400pt}}
\end{picture}

%% file: plot_house25.tex
\setlength{\unitlength}{0.240900pt}
\ifx\plotpoint\undefined\newsavebox{\plotpoint}\fi
\sbox{\plotpoint}{\rule[-0.200pt]{0.400pt}{0.400pt}}%
\begin{picture}(900,584)(0,0)
\sbox{\plotpoint}{\rule[-0.200pt]{0.400pt}{0.400pt}}%
\put(181.0,123.0){\rule[-0.200pt]{4.818pt}{0.400pt}}
\put(161,123){\makebox(0,0)[r]{ 0}}
\put(819.0,123.0){\rule[-0.200pt]{4.818pt}{0.400pt}}
\put(181.0,191.0){\rule[-0.200pt]{4.818pt}{0.400pt}}
\put(161,191){\makebox(0,0)[r]{ 0.2}}
\put(819.0,191.0){\rule[-0.200pt]{4.818pt}{0.400pt}}
\put(181.0,259.0){\rule[-0.200pt]{4.818pt}{0.400pt}}
\put(161,259){\makebox(0,0)[r]{ 0.4}}
\put(819.0,259.0){\rule[-0.200pt]{4.818pt}{0.400pt}}
\put(181.0,326.0){\rule[-0.200pt]{4.818pt}{0.400pt}}
\put(161,326){\makebox(0,0)[r]{ 0.6}}
\put(819.0,326.0){\rule[-0.200pt]{4.818pt}{0.400pt}}
\put(181.0,394.0){\rule[-0.200pt]{4.818pt}{0.400pt}}
\put(161,394){\makebox(0,0)[r]{ 0.8}}
\put(819.0,394.0){\rule[-0.200pt]{4.818pt}{0.400pt}}
\put(181.0,462.0){\rule[-0.200pt]{4.818pt}{0.400pt}}
\put(161,462){\makebox(0,0)[r]{ 1}}
\put(819.0,462.0){\rule[-0.200pt]{4.818pt}{0.400pt}}
\put(181.0,123.0){\rule[-0.200pt]{0.400pt}{4.818pt}}
\put(181,82){\makebox(0,0){ 0}}
\put(181.0,442.0){\rule[-0.200pt]{0.400pt}{4.818pt}}
\put(313.0,123.0){\rule[-0.200pt]{0.400pt}{4.818pt}}
\put(313,82){\makebox(0,0){ 20}}
\put(313.0,442.0){\rule[-0.200pt]{0.400pt}{4.818pt}}
\put(444.0,123.0){\rule[-0.200pt]{0.400pt}{4.818pt}}
\put(444,82){\makebox(0,0){ 40}}
\put(444.0,442.0){\rule[-0.200pt]{0.400pt}{4.818pt}}
\put(576.0,123.0){\rule[-0.200pt]{0.400pt}{4.818pt}}
\put(576,82){\makebox(0,0){ 60}}
\put(576.0,442.0){\rule[-0.200pt]{0.400pt}{4.818pt}}
\put(707.0,123.0){\rule[-0.200pt]{0.400pt}{4.818pt}}
\put(707,82){\makebox(0,0){ 80}}
\put(707.0,442.0){\rule[-0.200pt]{0.400pt}{4.818pt}}
\put(839.0,123.0){\rule[-0.200pt]{0.400pt}{4.818pt}}
\put(839,82){\makebox(0,0){ 100}}
\put(839.0,442.0){\rule[-0.200pt]{0.400pt}{4.818pt}}
\put(181.0,123.0){\rule[-0.200pt]{0.400pt}{81.665pt}}
\put(181.0,123.0){\rule[-0.200pt]{158.512pt}{0.400pt}}
\put(839.0,123.0){\rule[-0.200pt]{0.400pt}{81.665pt}}
\put(181.0,462.0){\rule[-0.200pt]{158.512pt}{0.400pt}}
\put(40,292){\makebox(0,0){\rotatebox{90}{matching accuracy}}}
\put(510,21){\makebox(0,0){baseline}}
\put(510,524){\makebox(0,0){$|\mathcal T|= 25$, MSE $< 10^{-9}$}}
\put(261,205){\makebox(0,0)[r]{JT}}
\multiput(281,205)(20.756,0.000){5}{\usebox{\plotpoint}}
\put(381,205){\usebox{\plotpoint}}
\put(281.00,215.00){\usebox{\plotpoint}}
\put(281,195){\usebox{\plotpoint}}
\put(381.00,215.00){\usebox{\plotpoint}}
\put(381,195){\usebox{\plotpoint}}
\put(181,462){\usebox{\plotpoint}}
\multiput(181,462)(19.945,-5.742){4}{\usebox{\plotpoint}}
\multiput(247,443)(17.987,-10.356){3}{\usebox{\plotpoint}}
\multiput(313,405)(14.342,-15.004){5}{\usebox{\plotpoint}}
\multiput(378,337)(18.895,-8.589){4}{\usebox{\plotpoint}}
\multiput(444,307)(18.221,-9.939){3}{\usebox{\plotpoint}}
\multiput(510,271)(19.690,-6.563){3}{\usebox{\plotpoint}}
\multiput(576,249)(20.304,-4.307){4}{\usebox{\plotpoint}}
\multiput(642,235)(19.750,-6.381){3}{\usebox{\plotpoint}}
\multiput(707,214)(18.676,-9.055){4}{\usebox{\plotpoint}}
\multiput(773,182)(20.099,-5.177){3}{\usebox{\plotpoint}}
\put(839,165){\usebox{\plotpoint}}
\put(181,462){\usebox{\plotpoint}}
\put(181,462){\usebox{\plotpoint}}
\put(171.00,462.00){\usebox{\plotpoint}}
\put(191,462){\usebox{\plotpoint}}
\put(171.00,462.00){\usebox{\plotpoint}}
\put(191,462){\usebox{\plotpoint}}
\put(247.00,441.00){\usebox{\plotpoint}}
\put(247,446){\usebox{\plotpoint}}
\put(237.00,441.00){\usebox{\plotpoint}}
\put(257,441){\usebox{\plotpoint}}
\put(237.00,446.00){\usebox{\plotpoint}}
\put(257,446){\usebox{\plotpoint}}
\put(313.00,401.00){\usebox{\plotpoint}}
\put(313,409){\usebox{\plotpoint}}
\put(303.00,401.00){\usebox{\plotpoint}}
\put(323,401){\usebox{\plotpoint}}
\put(303.00,409.00){\usebox{\plotpoint}}
\put(323,409){\usebox{\plotpoint}}
\put(378.00,332.00){\usebox{\plotpoint}}
\put(378,342){\usebox{\plotpoint}}
\put(368.00,332.00){\usebox{\plotpoint}}
\put(388,332){\usebox{\plotpoint}}
\put(368.00,342.00){\usebox{\plotpoint}}
\put(388,342){\usebox{\plotpoint}}
\put(444.00,303.00){\usebox{\plotpoint}}
\put(444,311){\usebox{\plotpoint}}
\put(434.00,303.00){\usebox{\plotpoint}}
\put(454,303){\usebox{\plotpoint}}
\put(434.00,311.00){\usebox{\plotpoint}}
\put(454,311){\usebox{\plotpoint}}
\put(510.00,266.00){\usebox{\plotpoint}}
\put(510,276){\usebox{\plotpoint}}
\put(500.00,266.00){\usebox{\plotpoint}}
\put(520,266){\usebox{\plotpoint}}
\put(500.00,276.00){\usebox{\plotpoint}}
\put(520,276){\usebox{\plotpoint}}
\put(576.00,244.00){\usebox{\plotpoint}}
\put(576,255){\usebox{\plotpoint}}
\put(566.00,244.00){\usebox{\plotpoint}}
\put(586,244){\usebox{\plotpoint}}
\put(566.00,255.00){\usebox{\plotpoint}}
\put(586,255){\usebox{\plotpoint}}
\put(642.00,229.00){\usebox{\plotpoint}}
\put(642,241){\usebox{\plotpoint}}
\put(632.00,229.00){\usebox{\plotpoint}}
\put(652,229){\usebox{\plotpoint}}
\put(632.00,241.00){\usebox{\plotpoint}}
\put(652,241){\usebox{\plotpoint}}
\put(707.00,206.00){\usebox{\plotpoint}}
\put(707,221){\usebox{\plotpoint}}
\put(697.00,206.00){\usebox{\plotpoint}}
\put(717,206){\usebox{\plotpoint}}
\put(697.00,221.00){\usebox{\plotpoint}}
\put(717,221){\usebox{\plotpoint}}
\put(773.00,176.00){\usebox{\plotpoint}}
\put(773,188){\usebox{\plotpoint}}
\put(763.00,176.00){\usebox{\plotpoint}}
\put(783,176){\usebox{\plotpoint}}
\put(763.00,188.00){\usebox{\plotpoint}}
\put(783,188){\usebox{\plotpoint}}
\put(839.00,162.00){\usebox{\plotpoint}}
\put(839,168){\usebox{\plotpoint}}
\put(829.00,162.00){\usebox{\plotpoint}}
\put(849,162){\usebox{\plotpoint}}
\put(829.00,168.00){\usebox{\plotpoint}}
\put(849,168){\usebox{\plotpoint}}
\put(181,462){\circle{12}}
\put(247,443){\circle{12}}
\put(313,405){\circle{12}}
\put(378,337){\circle{12}}
\put(444,307){\circle{12}}
\put(510,271){\circle{12}}
\put(576,249){\circle{12}}
\put(642,235){\circle{12}}
\put(707,214){\circle{12}}
\put(773,182){\circle{12}}
\put(839,165){\circle{12}}
\put(331,205){\circle{12}}
\sbox{\plotpoint}{\rule[-0.400pt]{0.800pt}{0.800pt}}%
\sbox{\plotpoint}{\rule[-0.200pt]{0.400pt}{0.400pt}}%
\put(261,164){\makebox(0,0)[r]{LBP}}
\sbox{\plotpoint}{\rule[-0.400pt]{0.800pt}{0.800pt}}%
\put(281.0,164.0){\rule[-0.400pt]{24.090pt}{0.800pt}}
\put(281.0,154.0){\rule[-0.400pt]{0.800pt}{4.818pt}}
\put(381.0,154.0){\rule[-0.400pt]{0.800pt}{4.818pt}}
\put(181,462){\usebox{\plotpoint}}
\put(181,459.84){\rule{15.899pt}{0.800pt}}
\multiput(181.00,460.34)(33.000,-1.000){2}{\rule{7.950pt}{0.800pt}}
\multiput(313.00,459.08)(2.886,-0.511){17}{\rule{4.533pt}{0.123pt}}
\multiput(313.00,459.34)(55.591,-12.000){2}{\rule{2.267pt}{0.800pt}}
\multiput(379.41,442.48)(0.501,-0.859){125}{\rule{0.121pt}{1.570pt}}
\multiput(376.34,445.74)(66.000,-109.742){2}{\rule{0.800pt}{0.785pt}}
\multiput(445.41,330.99)(0.501,-0.629){125}{\rule{0.121pt}{1.206pt}}
\multiput(442.34,333.50)(66.000,-80.497){2}{\rule{0.800pt}{0.603pt}}
\multiput(510.00,251.09)(1.113,-0.503){53}{\rule{1.960pt}{0.121pt}}
\multiput(510.00,251.34)(61.932,-30.000){2}{\rule{0.980pt}{0.800pt}}
\put(576,222.84){\rule{15.899pt}{0.800pt}}
\multiput(576.00,221.34)(33.000,3.000){2}{\rule{7.950pt}{0.800pt}}
\multiput(642.00,224.09)(1.760,-0.506){31}{\rule{2.937pt}{0.122pt}}
\multiput(642.00,224.34)(58.904,-19.000){2}{\rule{1.468pt}{0.800pt}}
\multiput(707.00,205.09)(2.011,-0.507){27}{\rule{3.306pt}{0.122pt}}
\multiput(707.00,205.34)(59.138,-17.000){2}{\rule{1.653pt}{0.800pt}}
\multiput(773.00,191.41)(1.401,0.504){41}{\rule{2.400pt}{0.122pt}}
\multiput(773.00,188.34)(61.019,24.000){2}{\rule{1.200pt}{0.800pt}}
\put(247.0,461.0){\rule[-0.400pt]{15.899pt}{0.800pt}}
\put(181,462){\usebox{\plotpoint}}
\put(171.0,462.0){\rule[-0.400pt]{4.818pt}{0.800pt}}
\put(171.0,462.0){\rule[-0.400pt]{4.818pt}{0.800pt}}
\put(247.0,461.0){\usebox{\plotpoint}}
\put(237.0,461.0){\rule[-0.400pt]{4.818pt}{0.800pt}}
\put(237.0,462.0){\rule[-0.400pt]{4.818pt}{0.800pt}}
\put(313.0,460.0){\usebox{\plotpoint}}
\put(303.0,460.0){\rule[-0.400pt]{4.818pt}{0.800pt}}
\put(303.0,461.0){\rule[-0.400pt]{4.818pt}{0.800pt}}
\put(378.0,446.0){\rule[-0.400pt]{0.800pt}{1.204pt}}
\put(368.0,446.0){\rule[-0.400pt]{4.818pt}{0.800pt}}
\put(368.0,451.0){\rule[-0.400pt]{4.818pt}{0.800pt}}
\put(444.0,329.0){\rule[-0.400pt]{0.800pt}{3.373pt}}
\put(434.0,329.0){\rule[-0.400pt]{4.818pt}{0.800pt}}
\put(434.0,343.0){\rule[-0.400pt]{4.818pt}{0.800pt}}
\put(510.0,246.0){\rule[-0.400pt]{0.800pt}{3.373pt}}
\put(500.0,246.0){\rule[-0.400pt]{4.818pt}{0.800pt}}
\put(500.0,260.0){\rule[-0.400pt]{4.818pt}{0.800pt}}
\put(576.0,218.0){\rule[-0.400pt]{0.800pt}{2.409pt}}
\put(566.0,218.0){\rule[-0.400pt]{4.818pt}{0.800pt}}
\put(566.0,228.0){\rule[-0.400pt]{4.818pt}{0.800pt}}
\put(642.0,220.0){\rule[-0.400pt]{0.800pt}{2.891pt}}
\put(632.0,220.0){\rule[-0.400pt]{4.818pt}{0.800pt}}
\put(632.0,232.0){\rule[-0.400pt]{4.818pt}{0.800pt}}
\put(707.0,199.0){\rule[-0.400pt]{0.800pt}{3.613pt}}
\put(697.0,199.0){\rule[-0.400pt]{4.818pt}{0.800pt}}
\put(697.0,214.0){\rule[-0.400pt]{4.818pt}{0.800pt}}
\put(773.0,184.0){\rule[-0.400pt]{0.800pt}{2.891pt}}
\put(763.0,184.0){\rule[-0.400pt]{4.818pt}{0.800pt}}
\put(763.0,196.0){\rule[-0.400pt]{4.818pt}{0.800pt}}
\put(839.0,205.0){\rule[-0.400pt]{0.800pt}{4.336pt}}
\put(829.0,205.0){\rule[-0.400pt]{4.818pt}{0.800pt}}
\put(181,462){\circle{18}}
\put(247,461){\circle{18}}
\put(313,461){\circle{18}}
\put(378,449){\circle{18}}
\put(444,336){\circle{18}}
\put(510,253){\circle{18}}
\put(576,223){\circle{18}}
\put(642,226){\circle{18}}
\put(707,207){\circle{18}}
\put(773,190){\circle{18}}
\put(839,214){\circle{18}}
\put(331,164){\circle{18}}
\put(829.0,223.0){\rule[-0.400pt]{4.818pt}{0.800pt}}
\sbox{\plotpoint}{\rule[-0.200pt]{0.400pt}{0.400pt}}%
\put(181.0,123.0){\rule[-0.200pt]{0.400pt}{81.665pt}}
\put(181.0,123.0){\rule[-0.200pt]{158.512pt}{0.400pt}}
\put(839.0,123.0){\rule[-0.200pt]{0.400pt}{81.665pt}}
\put(181.0,462.0){\rule[-0.200pt]{158.512pt}{0.400pt}}
\end{picture}

%% file: plot_house30.tex
\setlength{\unitlength}{0.240900pt}
\ifx\plotpoint\undefined\newsavebox{\plotpoint}\fi
\sbox{\plotpoint}{\rule[-0.200pt]{0.400pt}{0.400pt}}%
\begin{picture}(900,584)(0,0)
\sbox{\plotpoint}{\rule[-0.200pt]{0.400pt}{0.400pt}}%
\put(181.0,123.0){\rule[-0.200pt]{4.818pt}{0.400pt}}
\put(161,123){\makebox(0,0)[r]{ 0}}
\put(819.0,123.0){\rule[-0.200pt]{4.818pt}{0.400pt}}
\put(181.0,191.0){\rule[-0.200pt]{4.818pt}{0.400pt}}
\put(161,191){\makebox(0,0)[r]{ 0.2}}
\put(819.0,191.0){\rule[-0.200pt]{4.818pt}{0.400pt}}
\put(181.0,259.0){\rule[-0.200pt]{4.818pt}{0.400pt}}
\put(161,259){\makebox(0,0)[r]{ 0.4}}
\put(819.0,259.0){\rule[-0.200pt]{4.818pt}{0.400pt}}
\put(181.0,326.0){\rule[-0.200pt]{4.818pt}{0.400pt}}
\put(161,326){\makebox(0,0)[r]{ 0.6}}
\put(819.0,326.0){\rule[-0.200pt]{4.818pt}{0.400pt}}
\put(181.0,394.0){\rule[-0.200pt]{4.818pt}{0.400pt}}
\put(161,394){\makebox(0,0)[r]{ 0.8}}
\put(819.0,394.0){\rule[-0.200pt]{4.818pt}{0.400pt}}
\put(181.0,462.0){\rule[-0.200pt]{4.818pt}{0.400pt}}
\put(161,462){\makebox(0,0)[r]{ 1}}
\put(819.0,462.0){\rule[-0.200pt]{4.818pt}{0.400pt}}
\put(181.0,123.0){\rule[-0.200pt]{0.400pt}{4.818pt}}
\put(181,82){\makebox(0,0){ 0}}
\put(181.0,442.0){\rule[-0.200pt]{0.400pt}{4.818pt}}
\put(313.0,123.0){\rule[-0.200pt]{0.400pt}{4.818pt}}
\put(313,82){\makebox(0,0){ 20}}
\put(313.0,442.0){\rule[-0.200pt]{0.400pt}{4.818pt}}
\put(444.0,123.0){\rule[-0.200pt]{0.400pt}{4.818pt}}
\put(444,82){\makebox(0,0){ 40}}
\put(444.0,442.0){\rule[-0.200pt]{0.400pt}{4.818pt}}
\put(576.0,123.0){\rule[-0.200pt]{0.400pt}{4.818pt}}
\put(576,82){\makebox(0,0){ 60}}
\put(576.0,442.0){\rule[-0.200pt]{0.400pt}{4.818pt}}
\put(707.0,123.0){\rule[-0.200pt]{0.400pt}{4.818pt}}
\put(707,82){\makebox(0,0){ 80}}
\put(707.0,442.0){\rule[-0.200pt]{0.400pt}{4.818pt}}
\put(839.0,123.0){\rule[-0.200pt]{0.400pt}{4.818pt}}
\put(839,82){\makebox(0,0){ 100}}
\put(839.0,442.0){\rule[-0.200pt]{0.400pt}{4.818pt}}
\put(181.0,123.0){\rule[-0.200pt]{0.400pt}{81.665pt}}
\put(181.0,123.0){\rule[-0.200pt]{158.512pt}{0.400pt}}
\put(839.0,123.0){\rule[-0.200pt]{0.400pt}{81.665pt}}
\put(181.0,462.0){\rule[-0.200pt]{158.512pt}{0.400pt}}
\put(40,292){\makebox(0,0){\rotatebox{90}{matching accuracy}}}
\put(510,21){\makebox(0,0){baseline}}
\put(510,524){\makebox(0,0){$|\mathcal T|= 30$, MSE $< 10^{-9}$}}
\put(261,205){\makebox(0,0)[r]{JT}}
\multiput(281,205)(20.756,0.000){5}{\usebox{\plotpoint}}
\put(381,205){\usebox{\plotpoint}}
\put(281.00,215.00){\usebox{\plotpoint}}
\put(281,195){\usebox{\plotpoint}}
\put(381.00,215.00){\usebox{\plotpoint}}
\put(381,195){\usebox{\plotpoint}}
\put(181,462){\usebox{\plotpoint}}
\multiput(181,462)(20.239,-4.600){4}{\usebox{\plotpoint}}
\multiput(247,447)(18.564,-9.282){3}{\usebox{\plotpoint}}
\multiput(313,414)(15.725,-13.547){4}{\usebox{\plotpoint}}
\multiput(378,358)(19.410,-7.352){4}{\usebox{\plotpoint}}
\multiput(444,333)(20.239,-4.600){3}{\usebox{\plotpoint}}
\multiput(510,318)(20.473,-3.412){3}{\usebox{\plotpoint}}
\multiput(576,307)(20.304,-4.307){4}{\usebox{\plotpoint}}
\multiput(642,293)(20.290,-4.370){3}{\usebox{\plotpoint}}
\multiput(707,279)(19.778,-6.293){3}{\usebox{\plotpoint}}
\multiput(773,258)(20.473,-3.412){3}{\usebox{\plotpoint}}
\put(839,247){\usebox{\plotpoint}}
\put(181,462){\usebox{\plotpoint}}
\put(181,462){\usebox{\plotpoint}}
\put(171.00,462.00){\usebox{\plotpoint}}
\put(191,462){\usebox{\plotpoint}}
\put(171.00,462.00){\usebox{\plotpoint}}
\put(191,462){\usebox{\plotpoint}}
\put(247.00,445.00){\usebox{\plotpoint}}
\put(247,448){\usebox{\plotpoint}}
\put(237.00,445.00){\usebox{\plotpoint}}
\put(257,445){\usebox{\plotpoint}}
\put(237.00,448.00){\usebox{\plotpoint}}
\put(257,448){\usebox{\plotpoint}}
\put(313.00,411.00){\usebox{\plotpoint}}
\put(313,418){\usebox{\plotpoint}}
\put(303.00,411.00){\usebox{\plotpoint}}
\put(323,411){\usebox{\plotpoint}}
\put(303.00,418.00){\usebox{\plotpoint}}
\put(323,418){\usebox{\plotpoint}}
\put(378.00,353.00){\usebox{\plotpoint}}
\put(378,362){\usebox{\plotpoint}}
\put(368.00,353.00){\usebox{\plotpoint}}
\put(388,353){\usebox{\plotpoint}}
\put(368.00,362.00){\usebox{\plotpoint}}
\put(388,362){\usebox{\plotpoint}}
\put(444.00,330.00){\usebox{\plotpoint}}
\put(444,336){\usebox{\plotpoint}}
\put(434.00,330.00){\usebox{\plotpoint}}
\put(454,330){\usebox{\plotpoint}}
\put(434.00,336.00){\usebox{\plotpoint}}
\put(454,336){\usebox{\plotpoint}}
\put(510.00,316.00){\usebox{\plotpoint}}
\put(510,320){\usebox{\plotpoint}}
\put(500.00,316.00){\usebox{\plotpoint}}
\put(520,316){\usebox{\plotpoint}}
\put(500.00,320.00){\usebox{\plotpoint}}
\put(520,320){\usebox{\plotpoint}}
\put(576.00,305.00){\usebox{\plotpoint}}
\put(576,309){\usebox{\plotpoint}}
\put(566.00,305.00){\usebox{\plotpoint}}
\put(586,305){\usebox{\plotpoint}}
\put(566.00,309.00){\usebox{\plotpoint}}
\put(586,309){\usebox{\plotpoint}}
\put(642.00,291.00){\usebox{\plotpoint}}
\put(642,295){\usebox{\plotpoint}}
\put(632.00,291.00){\usebox{\plotpoint}}
\put(652,291){\usebox{\plotpoint}}
\put(632.00,295.00){\usebox{\plotpoint}}
\put(652,295){\usebox{\plotpoint}}
\put(707.00,275.00){\usebox{\plotpoint}}
\put(707,283){\usebox{\plotpoint}}
\put(697.00,275.00){\usebox{\plotpoint}}
\put(717,275){\usebox{\plotpoint}}
\put(697.00,283.00){\usebox{\plotpoint}}
\put(717,283){\usebox{\plotpoint}}
\put(773.00,254.00){\usebox{\plotpoint}}
\put(773,262){\usebox{\plotpoint}}
\put(763.00,254.00){\usebox{\plotpoint}}
\put(783,254){\usebox{\plotpoint}}
\put(763.00,262.00){\usebox{\plotpoint}}
\put(783,262){\usebox{\plotpoint}}
\put(839.00,245.00){\usebox{\plotpoint}}
\put(839,249){\usebox{\plotpoint}}
\put(829.00,245.00){\usebox{\plotpoint}}
\put(849,245){\usebox{\plotpoint}}
\put(829.00,249.00){\usebox{\plotpoint}}
\put(849,249){\usebox{\plotpoint}}
\put(181,462){\circle{12}}
\put(247,447){\circle{12}}
\put(313,414){\circle{12}}
\put(378,358){\circle{12}}
\put(444,333){\circle{12}}
\put(510,318){\circle{12}}
\put(576,307){\circle{12}}
\put(642,293){\circle{12}}
\put(707,279){\circle{12}}
\put(773,258){\circle{12}}
\put(839,247){\circle{12}}
\put(331,205){\circle{12}}
\sbox{\plotpoint}{\rule[-0.400pt]{0.800pt}{0.800pt}}%
\sbox{\plotpoint}{\rule[-0.200pt]{0.400pt}{0.400pt}}%
\put(261,164){\makebox(0,0)[r]{LBP}}
\sbox{\plotpoint}{\rule[-0.400pt]{0.800pt}{0.800pt}}%
\put(281.0,164.0){\rule[-0.400pt]{24.090pt}{0.800pt}}
\put(281.0,154.0){\rule[-0.400pt]{0.800pt}{4.818pt}}
\put(381.0,154.0){\rule[-0.400pt]{0.800pt}{4.818pt}}
\put(181,462){\usebox{\plotpoint}}
\put(247,459.84){\rule{15.899pt}{0.800pt}}
\multiput(247.00,460.34)(33.000,-1.000){2}{\rule{7.950pt}{0.800pt}}
\multiput(313.00,459.08)(2.643,-0.509){19}{\rule{4.200pt}{0.123pt}}
\multiput(313.00,459.34)(56.283,-13.000){2}{\rule{2.100pt}{0.800pt}}
\multiput(378.00,446.09)(0.559,-0.502){111}{\rule{1.095pt}{0.121pt}}
\multiput(378.00,446.34)(63.727,-59.000){2}{\rule{0.547pt}{0.800pt}}
\multiput(444.00,387.09)(0.507,-0.501){123}{\rule{1.012pt}{0.121pt}}
\multiput(444.00,387.34)(63.899,-65.000){2}{\rule{0.506pt}{0.800pt}}
\multiput(510.00,322.09)(0.600,-0.502){103}{\rule{1.160pt}{0.121pt}}
\multiput(510.00,322.34)(63.592,-55.000){2}{\rule{0.580pt}{0.800pt}}
\multiput(576.00,267.09)(1.290,-0.504){45}{\rule{2.231pt}{0.121pt}}
\multiput(576.00,267.34)(61.370,-26.000){2}{\rule{1.115pt}{0.800pt}}
\multiput(642.00,241.09)(0.964,-0.503){61}{\rule{1.729pt}{0.121pt}}
\multiput(642.00,241.34)(61.411,-34.000){2}{\rule{0.865pt}{0.800pt}}
\multiput(707.00,207.09)(0.829,-0.502){73}{\rule{1.520pt}{0.121pt}}
\multiput(707.00,207.34)(62.845,-40.000){2}{\rule{0.760pt}{0.800pt}}
\put(773,166.84){\rule{15.899pt}{0.800pt}}
\multiput(773.00,167.34)(33.000,-1.000){2}{\rule{7.950pt}{0.800pt}}
\put(181.0,462.0){\rule[-0.400pt]{15.899pt}{0.800pt}}
\put(181,462){\usebox{\plotpoint}}
\put(171.0,462.0){\rule[-0.400pt]{4.818pt}{0.800pt}}
\put(171.0,462.0){\rule[-0.400pt]{4.818pt}{0.800pt}}
\put(247.0,461.0){\usebox{\plotpoint}}
\put(237.0,461.0){\rule[-0.400pt]{4.818pt}{0.800pt}}
\put(237.0,462.0){\rule[-0.400pt]{4.818pt}{0.800pt}}
\put(313.0,461.0){\usebox{\plotpoint}}
\put(303.0,461.0){\rule[-0.400pt]{4.818pt}{0.800pt}}
\put(303.0,462.0){\rule[-0.400pt]{4.818pt}{0.800pt}}
\put(378.0,446.0){\rule[-0.400pt]{0.800pt}{0.964pt}}
\put(368.0,446.0){\rule[-0.400pt]{4.818pt}{0.800pt}}
\put(368.0,450.0){\rule[-0.400pt]{4.818pt}{0.800pt}}
\put(444.0,386.0){\rule[-0.400pt]{0.800pt}{1.445pt}}
\put(434.0,386.0){\rule[-0.400pt]{4.818pt}{0.800pt}}
\put(434.0,392.0){\rule[-0.400pt]{4.818pt}{0.800pt}}
\put(510.0,318.0){\rule[-0.400pt]{0.800pt}{2.650pt}}
\put(500.0,318.0){\rule[-0.400pt]{4.818pt}{0.800pt}}
\put(500.0,329.0){\rule[-0.400pt]{4.818pt}{0.800pt}}
\put(576.0,265.0){\rule[-0.400pt]{0.800pt}{1.927pt}}
\put(566.0,265.0){\rule[-0.400pt]{4.818pt}{0.800pt}}
\put(566.0,273.0){\rule[-0.400pt]{4.818pt}{0.800pt}}
\put(642.0,238.0){\rule[-0.400pt]{0.800pt}{2.168pt}}
\put(632.0,238.0){\rule[-0.400pt]{4.818pt}{0.800pt}}
\put(632.0,247.0){\rule[-0.400pt]{4.818pt}{0.800pt}}
\put(707.0,203.0){\rule[-0.400pt]{0.800pt}{2.891pt}}
\put(697.0,203.0){\rule[-0.400pt]{4.818pt}{0.800pt}}
\put(697.0,215.0){\rule[-0.400pt]{4.818pt}{0.800pt}}
\put(773.0,163.0){\rule[-0.400pt]{0.800pt}{2.891pt}}
\put(763.0,163.0){\rule[-0.400pt]{4.818pt}{0.800pt}}
\put(763.0,175.0){\rule[-0.400pt]{4.818pt}{0.800pt}}
\put(839.0,161.0){\rule[-0.400pt]{0.800pt}{3.373pt}}
\put(829.0,161.0){\rule[-0.400pt]{4.818pt}{0.800pt}}
\put(181,462){\circle{18}}
\put(247,462){\circle{18}}
\put(313,461){\circle{18}}
\put(378,448){\circle{18}}
\put(444,389){\circle{18}}
\put(510,324){\circle{18}}
\put(576,269){\circle{18}}
\put(642,243){\circle{18}}
\put(707,209){\circle{18}}
\put(773,169){\circle{18}}
\put(839,168){\circle{18}}
\put(331,164){\circle{18}}
\put(829.0,175.0){\rule[-0.400pt]{4.818pt}{0.800pt}}
\sbox{\plotpoint}{\rule[-0.200pt]{0.400pt}{0.400pt}}%
\put(181.0,123.0){\rule[-0.200pt]{0.400pt}{81.665pt}}
\put(181.0,123.0){\rule[-0.200pt]{158.512pt}{0.400pt}}
\put(839.0,123.0){\rule[-0.200pt]{0.400pt}{81.665pt}}
\put(181.0,462.0){\rule[-0.200pt]{158.512pt}{0.400pt}}
\end{picture}